\documentclass[sigconf]{acmart}

\AtBeginDocument{%
  \providecommand\BibTeX{{%
    \normalfont B\kern-0.5em{\scshape i\kern-0.25em b}\kern-0.8em\TeX}}}

\setcopyright{none}
\renewcommand\footnotetextcopyrightpermission[1]{}

\acmConference[arXiv preprint]{arXiv preprint}
{2020}
\usepackage{booktabs}       % professional-quality tables

\usepackage{graphicx}
\usepackage{subfigure}
\usepackage{multirow}
\usepackage{amssymb}
\usepackage{amsfonts, amsmath, amsthm}
\usepackage{algorithm,algorithmic}

\newtheorem{assumption}{Assumption}
\newtheorem{theorem}{Theorem}
\newtheorem{lemma}{Lemma}
\newtheorem{corollary}{Corollary}
\theoremstyle{definition}
\newtheorem{example}{Example}

\newcommand{\D}{\mathcal{D}}
\newcommand{\R}{\mathbb{R}}

\newcommand{\E}{\mathbb{E}}
\newcommand{\B}{\mathcal{B}}
\newcommand{\Bh}{\hat{\mathcal{B}}}
\newcommand{\CC}{\mathcal{C}}
\newcommand{\Ch}{\hat{\mathcal{C}}}
\newcommand{\x}{\mathbf{x}}
\newcommand{\e}{\mathbf{e}}
\newcommand{\uu}{\mathbf{u}}
\newcommand{\Vv}{\mathbf{v}}
\newcommand{\N}{\mathbb{N}}

\newcommand{\g}{\mathbf{g}}
\newcommand{\f}{\hat{f}}
\newcommand{\ze}{\mathbf{0}}

\newcommand{\Tht}{\hat{T}^{t}}
\newcommand{\Tt}{T^{t}}
\newcommand{\Tts}{T^{t - 1}}
\newcommand{\Tit}{\Tt_i}
\newcommand{\Tits}{\Tts_i}
\newcommand{\Tpt}{\Tt_p}
\newcommand{\Tqt}{\Tt_q}
\newcommand{\Tonet}{\Tt_1}
\newcommand{\Ttwot}{\Tt_2}
\newcommand{\Tjtau}{T_j^{\tau}}

\newcommand{\xrC}{\x^{rC}}

\newcommand{\xTonet}{\x^{\Tonet}}
\newcommand{\xTtwot}{\x^{\Ttwot}}

\newcommand{\xt}{\x^{t}}
\newcommand{\xit}{\xt_i}

\newcommand{\xts}{\x^{t-1}}
\newcommand{\xits}{\xts_i}

\newcommand{\uir}{\uu_i^{r}}
\newcommand{\uirt}{\uu_i^{r^t}}

\newcommand{\ur}{\uu^{r}}
\newcommand{\Rir}{R_i^{r}}
\newcommand{\Rirt}{R_i^{r^t}}
\newcommand{\Rirts}{R_i^{r^t - 1}}
\newcommand{\Rirs}{R_i^{r - 1}}

\newcommand{\uiRirs}{\uu_i^{\Rirs}}

\newcommand{\git}{\g^{t  }_i}
\newcommand{\gt}{\g^{ t  }}
\newcommand{\giTit}{\g^{  \Tit  }_i}
\newcommand{\giTits}{\g^{  \Tits  }_i}
\newcommand{\gqTqt}{\g^{  \Tqt  }_q}
\newcommand{\gpTpt}{\g^{  \Tpt  }_p}
\newcommand{\gqTpt}{\g^{  \Tpt  }_q}
\newcommand{\gjTjtau}{\g^{  \Tjtau  }_j}
\newcommand{\intermediateForUpdate}{\g_j^{R_j^{r_{\tau}} C - r_{\tau} C + \tau }}

\newcommand{\giTittau}{\g^{  \Rirt C  - r^t C + t  }_i}

\newcommand{\gittau}{\g^{ t  }_i}

\newcommand{\fxrCs}{f ( \x^{ rC - 1} )}
\newcommand{\fxrtCs}{f ( \x^{  r^tC - 1 } )}
\newcommand{\fixTitS}{f_i ( \x^{ \Tit-1  } ) }

\newcommand{\fpxTptS}{f_p ( \x^{  \Tpt-1  } ) }
\newcommand{\fqxTqtS}{f_q ( \x^{  \Tqt-1  } ) }
\newcommand{\fqxTptS}{f_q ( \x^{  \Tpt-1  } ) }
\newcommand{\fxts}{f ( \x^{  t-1  } ) }
\newcommand{\Fxts}{\f ( \x^{  t - 1  } ) }
\newcommand{\fxt}{f ( \x^{  t  } ) }

\newcommand{\fxTtS}{f ( \x^{  \Tt - 1  } ) }
\newcommand{\fxThtS}{f ( \x^{  \Tht - 1  } ) }
\newcommand{\fxThtauS}{f ( \x^{  \hat{T}^{\tau } - 1  } ) }
\newcommand{\fixTtS}{f_i ( \x^{  \Tt - 1  } ) }
\newcommand{\fixts}{f_i ( \x^{  t-1  } ) }

\newcommand{\fixttaus}{f_i ( \x_i^{  t - 1 } ) }
\newcommand{\fixTittaus}{f_i ( \x_i^{  \Rirt C - r^t C + t - 1  } )}

\newcommand{\XiTitS}{\xi^{\left[\Tit-1 \right]}}
\newcommand{\XiTtS}{\xi^{\left[\Tt - 1\right]}}

\newcommand{\equuvv}{$\langle\uu,\vv\rangle=\frac{1}{2}(\parallel \uu\parallel ^2+\parallel \vv\parallel ^2-\parallel \uu-\vv\parallel ^2 )$}

\newcommand{\neqCSAMGM}{Cauchy–Schwarz inequality and AM-GM inequality}
\newcommand{\convex}{the convexity of $\parallel \cdot\parallel ^2$}
\newcommand{\LTE}{Law of Total Expectation $\E[\E[\mathbf{X}|\mathbf{Y}]]=\E[\mathbf{X}]$}

\newcommand{\NK}{\left\lceil \frac{N}{K} \right\rceil}
\newcommand{\newl}{\nonumber\\&&}
\newcommand{\new}{\nonumber\\}

\newcommand{\myappendix}[1]{Appendix~\ref{#1}}

\begin{document}

%%
%% The "title" command has an optional parameter,
%% allowing the author to define a "short title" to be used in page headers.
\title{Distributed Non-Convex Optimization with Sublinear Speedup under Intermittent Client Availability }

%%
%% The "author" command and its associated commands are used to define
%% the authors and their affiliations.
%% Of note is the shared affiliation of the first two authors, and the
%% "authornote" and "authornotemark" commands
%% used to denote shared contribution to the research.
\author{Yikai Yan$^1$, Chaoyue Niu$^1$, Yucheng Ding$^1$, Zhenzhe Zheng$^1$, Fan Wu$^1$, Guihai Chen$^1$, \\
Shaojie Tang$^2$, and Zhihua Wu$^3$}
\affiliation{
  \institution{
    $^1$Shanghai Jiao Tong University, China\quad
    $^2$University of Texas at Dallas, USA\quad
    $^3$Alibaba Group, China}
}

%%
%% By default, the full list of authors will be used in the page
%% headers. Often, this list is too long, and will overlap
%% other information printed in the page headers. This command allows
%% the author to define a more concise list
%% of authors' names for this purpose.
\renewcommand{\shortauthors}{Yan, et al.}

%%
%% The abstract is a short summary of the work to be presented in the
%% article.
\begin{abstract}
  Federated learning is a new distributed machine learning framework, 
  where a bunch of heterogeneous clients collaboratively train a model 
  without sharing training data. 
  In this work, we consider a practical and ubiquitous issue 
  when deploying federated learning in mobile environments: 
  \emph{intermittent client availability}, 
  where the set of eligible clients may change during the training process. 
  Such intermittent client availability would seriously deteriorate 
  the performance of the classical Federated Averaging algorithm 
  (FedAvg for short). 
  Thus, we propose a simple distributed non-convex optimization algorithm, 
  called Federated Latest Averaging (FedLaAvg for short), 
  which leverages the latest gradients of all clients, 
  even when the clients are not available, 
  to jointly update the global model in each iteration. 
  Our theoretical analysis shows that FedLaAvg attains the 
  convergence rate of $O(E^{1/2}/(N^{1/4} T^{1/2}))$, 
  achieving a sublinear speedup with respect to 
  the total number of clients. 
  We implement FedLaAvg along with several baselines and evaluate them over 
  the benchmarking MNIST and Sentiment140 datasets. 
  The evaluation results demonstrate that FedLaAvg achieves more stable 
  training than FedAvg in both convex and non-convex settings
  and indeed reaches a sublinear speedup.
\end{abstract}

%%
%% The code below is generated by the tool at http://dl.acm.org/ccs.cfm.
%% Please copy and paste the code instead of the example below.
%%
\begin{comment}
\begin{CCSXML}
<ccs2012>
	<concept>
		<concept_id>10010147.10010178.10010219</concept_id>
		<concept_desc>Computing methodologies~Distributed artificial intelligence</concept_desc>
		<concept_significance>500</concept_significance>
		</concept>
	<concept>
		<concept_id>10003752.10003809.10003716.10011138.10011140</concept_id>
		<concept_desc>Theory of computation~Nonconvex optimization</concept_desc>
		<concept_significance>500</concept_significance>
		</concept>
	</ccs2012>
\end{CCSXML}
\end{comment}
%\ccsdesc[500]{Computing methodologies~Distributed artificial intelligence}
%\ccsdesc[500]{Theory of computation~Nonconvex optimization}

%%
%% Keywords. The author(s) should pick words that accurately describe
%% the work being presented. Separate the keywords with commas.
%\keywords{federated learning, client availability, distributed optimization}

%%
%% This command processes the author and affiliation and title
%% information and builds the first part of the formatted document.
\maketitle

\section{Introduction}
\label{sec:introduction}
% introduce FL
Federated Learning (FL) is a new paradigm of distributed 
machine learning~\cite{
 % FLcoalition,
 % incentive,
 % incentive2,
 % VEC,
 survey1,survey2,McMahan2016CommunicationEfficientLO}. 
It allows multiple clients to collaboratively train a global model without
needing to upload local data to a centralized cloud server. In the FL setting,
data are massively distributed over clients, with non-IID distribution~\cite{KevinCorr2019Noniid, FedProx} 
and unbalance in quantity~\cite{MehryarIcml2019Agnos};
in these ways, FL is distinguished from traditional 
distributed optimization~\cite{
  % LianIcml2018Adpsgd, 
  % HanlinNips2018Compress, 
  % HanlinIcml2018D2, 
  % HanlinIcml2019DS, 
  % HaoIcml2019Dynamic,
  MuOsdi2014Distributed}. 
Furthermore, the agents participating in FL are typically 
unreliable heterogeneous clients, e.g., mobile devices,
with limited computation resources and 
unstable communication links~\cite{
  % fog-hete,
  net-fog,
  AAAIChuanHete,
  var-reduct-robust,
  % robust2,
  robust3,
  % SCEDA, 
  % ByzantineDL, 
  % secureDL,
  FedNova}, 
resulting in a varying set of eligible clients 
during the training process.  
These new features pose challenges in designing 
and analyzing learning algorithms for FL.

% Why availability is a issue
One of the leading challenges in deploying FL systems is
client availability, where the clients may not be available throughout
the entire training process. 
Consider the typical FL scenario where Google's mobile keyboard Gboard 
polishes its language models among numerous mobile-device 
users~\cite{Bonawitz2019TowardsFL, Yang2018APPLIEDFL, CoRR:Gboard}.
To minimize the negative impact on user experience, only devices 
that meet certain requirements (e.g., charging, idle, and free Wi-Fi) 
are eligible for model training. 
These requirements are usually met at night local time 
but are not satisfied in the daytime when the devices are busy.
Such intermittent client availability would introduce bias into 
training data.
In particular, the clients with longer time available 
are more likely to be selected to participate, 
and thus their training data would be over-represented. 
In contrast, the training data of the clients, who have 
shorter time available and lower chance to be chosen, 
may be under-represented. 
Further, if the free resources on local devices 
(e.g., CPU and RAM) are also incorporated, 
the availability patterns of different clients would be more diverse, 
implying that the data representations
are more differentiated in the collaborative training process.
Nevertheless, the test data distribution, 
which is irrelevant with client availability in the training phase, 
would be inconsistent with the training data distribution.
This inconsistency is also known as 
dataset shift~\cite{DSinML, DatasetshiftClassification}, 
a notorious obstacle to the convergence of 
machine learning algorithms~\cite{AdarshAistats2019Ds, JasperNips2019Ds}, 
which also exists in FL,
and can degrade the generalization ability of FL algorithms.

% \begin{table*}[t]
% 	\centering
% 	\caption{Convergence results of FedAvg under different client availability assumptions.}
% 	\label{tb:related work}
% 	\begin{center}
% 		\resizebox{\textwidth}{!}{
% 			\begin{tabular}{lcc}
% 				\toprule
% 				Studies & Assumptions on Client Availability & Convergence Rate \\
% 				\midrule
% 				\cite{Wang2018CooperativeSA,
% 					Yu2018ParallelRS,
% 					Khaled2019FirstAO,
% 					Stich2019LocalSC,
% 					SebastianEFSGD}
% 				&
% 				All clients are available and participate in training.
% 				&
% 				$O(1/\sqrt{NT})$
% 				\\
% 				\cite{Li2019OnTC, SCAFFOLD, power-of-choice} 
% 				& All clients are available, and a subset of clients participate in training.  
% 				& $O(1/T)$ 
% 				\\
% 				\multirow{2}{*}{Current study} 
% 				&Each client is available at least once during any period with length $E$.
% 				&\multirow{2}{*}{$O(E^{1/2}/(N^{1/4} T^{1/2}))$}
% 				\\
% 				&A subset of the available clients participate in training.
% 				&
% 				\\
% 				\bottomrule
% 			\end{tabular}
% 		}
% 	\end{center}
% \end{table*}

% Related Work
% summary
Existing work in the literature has not touched the issue of 
intermittent client availability\footnote{
	A concurrent work~\cite{flexible}, released roughly three months after our preprint~\cite{FedLaAvg},
	considered a different availability setting,
	where some clients submit partially completed work or drop out occasionally.
	They proposed to kick out frequently dropped clients, 
	which, however, cannot eliminate the training data bias under intermittent client availability.
	We reserve the divergence analysis of their method in 
	\myappendix{sec:comp flexible}.
}, and the convergence analysis 
of FL algorithms always requires all the clients to be 
available throughout the training process. 
In this case, there is no bias in the training data, which is an 
essential condition to obtain the 
positive convergence results.
Much effort~\cite{Wang2018CooperativeSA,Yu2018ParallelRS,Khaled2019FirstAO,
Stich2019LocalSC,SebastianEFSGD,Li2019OnTC,SCAFFOLD,
power-of-choice}
has been expended in proving the convergence of the classical 
FedAvg algorithm~\cite{McMahan2016CommunicationEfficientLO}.
% full participation
One line of work~\cite{Wang2018CooperativeSA,
Yu2018ParallelRS,Khaled2019FirstAO,Stich2019LocalSC,SebastianEFSGD} 
assumed that \textbf{all the clients are available and participate
in each iteration of the training}, 
to establish the $O(1/\sqrt{NT})$\footnote{Notation $N$ is the total number of clients, 
and $T$ is the total number of iterations in the training.} 
convergence of FedAvg.
However, the requirement of full client participation would 
significantly increase the synchronization latency of 
the collaborative training process, and is hard to be satisfied 
in practical cross-device FL scenarios.
% partial participation
Another line of work~\cite{Li2019OnTC, SCAFFOLD,power-of-choice}
\textbf{allowed partial client participation but required 
all the clients to be available},
to proved an $O(1/T)$ convergence of FedAvg.
In their analysis, clients are selected either uniformly at random~\cite{Li2019OnTC, SCAFFOLD} or 
according to a certain strategy~\cite{power-of-choice},
which are, however, possible only if all clients are available.
% insight

% Our model
In this work, we integrate the consideration of 
intermittent client availability into the design and
analysis of the FL algorithm. 
We first formulate a practical model for 
intermittent client availability in FL; 
this model allows the set of available clients to 
follow any time-varying distribution, 
with the assumption that each client needs to be available 
at least once during any period with length $E$.
Under such a client availability model, 
FedAvg would diverge even in a simple learning scenario with a quadratic objective
(shown in Section~\ref{sec:divergence}), 
because the training data are biased towards 
those highly available clients.
For general distributed non-convex optimization, 
we propose a simple Federated Latest Averaging algorithm, 
namely FedLaAvg, to approximately balance the influence 
of each client's data on the global model training.
Specifically, instead of averaging only the gradients 
collected from participating clients, FedLaAvg averages 
the latest gradients\footnote{The latest gradient of a given client is the 
gradient calculated in its latest participating iteration. 
Please refer to Section~\ref{sec:FedLaAvg} for detailed definition.} 
of all clients. 
By setting appropriate parameters, 
we can prove an $O(E^{1/2}/(N^{1/4}T^{1/2}))$ convergence for FedLaAvg,
implying that FedLaAvg can achieve a 
sublinear speedup with respect to the total number of clients.
We summarize our contributions as follows.
\begin{itemize}
	\item 
	To the best of our knowledge, 
	we are the first to study the problem of 
	intermittent client availability in FL, 
	an ubiquitous phenomenon in practical mobile environments,
	and present a formal formulation thereof.
	\item
	We demonstrate that even with exact (not stochastic) gradient descent,
	two clients in the system,
	one local iteration on either client,
	and a simple quadratic (convex) optimization objective,
	FedAvg can diverge due to intermittent client availability.
	\item
	We identify the reasons behind the divergence of FedAvg and further 
	propose a convergent algorithm FedLaAvg, which aggregates the latest gradients of all clients in each training iteration. 
	Our theoretical analysis shows the  $O(E^{1/2}/(N^{1/4}T^{1/2}))$ convergence of FedLaAvg for general distributed non-convex optimization.
	\item 
	Using the public MNIST and Sentiment140 datasets, 
	we evaluate FedLaAvg and compare 
	its performance with FedSGD, FedAvg, and FedProx~\cite{FedProx}. 
	The evaluation results validate the 
	superiority of our FedLaAvg in terms of  
	more smooth training process, 
	sublinear speedup,
	and lower training loss.
\end{itemize}

\section{Problem Formulation}
\label{sec:preliminary}
We consider a general distributed non-convex optimization scenario in which $N$ clients collaboratively solve the following consensus
optimization problem:
%\begin{equation}
\[
\min_{\x\in \R^m} f(\x)
\triangleq 
\sum_{i=1}^{N}w_i \E_{\xi_i\sim \D_i}[F(\x;\xi_i)]
=\sum_{i=1}^{N}w_i \tilde{f_i}(\x).
\]
%\end{equation}
Each client $i$ holds training data $\xi_i \sim \D_i$, and $w_i$ is the weight of this client (typically the proportion of client $i$'s local data volume in the total data volume of the FL system~\cite{McMahan2016CommunicationEfficientLO}). 
Function $F(\x;\xi_i)$ is the training error of model parameters $\x$ over local data $\xi_i$, and $\tilde{f_i}(\x)$ is the local generalization error, taking expectation over the randomness of local data. 
In iteration\footnote{Since our major focus is client availability, 
for the sake of conciseness, 
we first consider the training scenario where participating clients perform only one local iteration, 
and extend our results to the multiple-local-iteration scenario in Section~\ref{sec:multiround}.}
$t$, participating client $i$ observes the local stochastic gradient:
$\git=\nabla F(\x^{t-1};\xi_i^{t})$,
where $\x^{t-1}$ is the model parameters from the previous iteration and $\xi_i^{t}$ is the local training data in this iteration. We note  $\E\left[\git\mid\xi^{[t-1]}\right]=\nabla \tilde{f_i}(\x^{t-1})$,
where $\xi^{[t-1]}$ is the historical training data from all clients before iteration $t$:
$
\xi^{[t-1]}
\triangleq \{\xi_i^\tau|i\in\{1,2,\cdots,N\},\tau\in\{1,2,\cdots,t-1\}\}.
$
To simplify the analysis of unbalanced data volume among clients, we use a scaling technique to obtain a revised local objective function:
$
f_i(\x) = w_i N \tilde{f_i}(\x).
$
Then, we can rewrite the global objective function as
$
f(\x) = \frac{1}{N} \sum_{i=1}^{N} f_i(\x).
$

In this study, we make three assumptions regarding the objective funtions as follows.
\begin{assumption}\label{as:1}
	Local objective functions $f_i$ are all $L$-smooth: 
	$
	\lVert \nabla f_i(\uu)-\nabla f_i(\Vv)\rVert \leq L\lVert \uu-\Vv\rVert,\,\forall i,\,\uu,\,\Vv
	$.
	The corollary is 
	$f_i(\mathbf{v})\leq f_i(\mathbf{u})+\left\langle\mathbf{v}-\mathbf{u},\nabla f_i(\mathbf{u})\right\rangle+\frac{L}{2}\parallel \mathbf{v}-\mathbf{u}\parallel^2,\,\forall i,\,\uu,\,\Vv$.
\end{assumption}
\begin{assumption}\label{as:2}
	Bounded variance: with variance $\sigma>0$, $ \forall i,\,\x $,
	$
	\E_{\xi_i\sim \D_i}\left[\left\lVert \nabla F\left(\x;\xi_i\right)-\nabla f_i\left(\x\right)\right\rVert ^2\right]\leq \sigma^2.
	$
\end{assumption}
\begin{assumption}\label{as:3}
	Bounded gradient: with gradient norm $G>0$, $\forall i,\,\x$,
	$\E_{\xi_i\sim \D_i}\left[\parallel \nabla F\left(\x;\xi_i\right)\parallel ^2\right]\leq G^2$.
	\end{assumption}

% availability modeling containing all previous work
To model intermittent client availability, we use $\CC^t$ to denote the set of available clients in iteration $t$. 
We formally introduce Assumption~\ref{as:4} regarding the intermittent client availability model in FL.
\begin{assumption}\label{as:4}
	Minimal availability: each client $i$ is available at least once in any period with $E$ successive iterations:
	$
	\forall i,\,\forall t,\,\exists \tau \in \{t, t+1,\cdots,t+E-1\},\text{ such that } i\in \CC^\tau.
	$
\end{assumption}

Assumption \ref{as:1} is standard, and Assumptions \ref{as:2} and \ref{as:3} have also been widely made in the 
literature~\cite{Yu2018ParallelRS,Stich2019LocalSC,Li2019OnTC,Zhang2012CommunicationefficientAF,Stich2018SparsifiedSW,Yu2019OnTL}. 
Specifically, \citet{Yu2018ParallelRS} worked with non-convex functions under Assumptions~\ref{as:1}--\ref{as:3}, and required all clients to be available and to participate in each iteration.
Meanwhile, \citet{Li2019OnTC} focused on convex functions while imposing the same full client availability requirement.
The full client availability model in existing work is equivalent to the special case of our intermittent client availability model by setting $E=1$ in Assumption~\ref{as:4}.
Furthermore, Assumption~\ref{as:4} regarding the intermittent client availability model is reasonable in practical FL.
For example, as discussed earlier, clients are typically available at night, and thus Assumption~\ref{as:4} with $E$ equal to the number of iterations in one day can describe such a client availability scenario. 

\section{Algorithm Design}
\label{sec:design}
If the ideal full client availability is guaranteed, 
FedAvg with only one local iteration is equivalent to the classical 
mini-batch SGD and can converge.
However, under the practical intermittent client availability model, 
the equivalence does not hold, and we show that FedAvg would 
produce arbitrarily poor results, 
even if only one local iteration is performed.
We further investigate the underlying reasons for the divergence of FedAvg, and then propose a new convergent algorithm called FedLaAvg.

\subsection{Divergence of FedAvg}\label{sec:divergence}
\begin{example}\label{exp:mean estimation}
	We consider a distributed optimization problem with only two clients (denoted as $1$ and $2$) and a convex objective function.
	The goal is to learn the mean of one-dimensional data from these two clients.
	Following the problem formulation in Section~\ref{sec:preliminary}, the local data distribution 
	is $\xi_i \sim \D_i$ with mean $\e_i=\E\left[ \xi_i \right]$. For simplicity, 
	we assume the amounts of data from the two clients are balanced. 
	We can formulate this  learning problem as minimizing the mean square error (MSE):
	\begin{align}\label{eq:2worker}
	f(\x) 
	&= 
	\frac{1}{2}\sum_{i=1}^{2}f_i(\x)
	=
	\frac{1}{2}\sum_{i=1}^{2}\E_{\xi_i\sim \D_i} \left[(\x-\xi_i)^2\right] \new
	&=
	\frac{1}{2}\sum_{i=1}^{2}(\x-\e_i)^2 
	+
	\frac{1}{2}\sum_{i=1}^{2}\E_{\xi_i\sim \D_i}\left[(\xi_i-\e_i)^2\right].
	\nonumber
	\end{align}
	For this example, we consider a specific intermittent client availability model: 
	two clients are available periodically and alternately, i.e., 
	in each period, client 1 is available in the first $t_1$ iterations, 
	and client 2 is available in the following $t_2$ iterations. Let $k$ index the period; we then have
	\[1\in \CC^{k(t_1 + t_2) + i}, k\in \N, i\in\{1,2,\cdots,t_1\};\]
	\[2\in \CC^{k(t_1 + t_2) + i}, k\in \N, i\in\{t_1 + 1,t_1+2 \cdots, t_1+t_2\}.\]
	This model describes the client availability with a regular diurnal pattern, which has been widely observed in previous studies~\cite{Bonawitz2019TowardsFL,Yang2018APPLIEDFL,Eichner2019SemiCyclicSG}, and is a typical subcase of our intermittent client availability model. For example, clients around the world participate in FL at night. Clients 1 and 2 may correspond to clients from two different geographic regions, respectively.
\end{example}
%	$\frac{t_1\e_1+t_2\e_2}{t_1+t_2}$

\begin{theorem}\label{the:divergence}
	Suppose each client computes the exact (not stochastic) gradient. 
	In Example~\ref{exp:mean estimation}, even with a sufficiently low learning rate, 
	the model parameters returned by FedAvg at the end of each period, 
	i.e., $\x^{k(t_1+t_2)}$, would converge to \mbox{$(t_1\e_1+t_2\e_2)/(t_1+t_2)$}, 
	which can be arbitrarily far away from the optimal solution $\x^* = (\e_1 + \e_2)/2$.
\end{theorem}
\begin{proof}[Proof of Theorem~\ref{the:divergence}]
	In Example~\ref{exp:mean estimation}, the training process of FedAvg is that 
	the two clients train the global model using their own local data alternatively.
	Hence, after a certain number of training iterations, the global model parameters would
	be ``pulled'' in opposite directions when different clients are available, 
	and would finally oscillate periodically around \mbox{$(t_1\e_1+t_2\e_2)/(t_1+t_2)$}.
	The detailed proof is given as follows.
% 	Further, since there is only one client available in each round, 
% 	client selection strategies do not work.
	
	We first show that if $\gamma<1/2$, $\x^{k(t_1+t_2)}$ would converge to
% 	\begin{equation}\label{eq:defX}
    \[
	X = \frac{\left(1-2\gamma\right)^{t_2}\left(\e_1-\e_2\right) + \e_2 - \e_1 \left(1-2\gamma\right)^{t_1+t_2}}
	{1-\left(1-2\gamma\right)^{t_1+t_2}}.
	\]
% 	\end{equation}
	Note that for iterations where client $1$ is available, we have
	\[
	\forall t\in\{k(t_1+t_2)+i\mid k\in\N, i\in\{1,\cdots, t_1\}\},\,
	\x^{t+1}\! = \x^{t} - 2\gamma\left(\x^{t}-\e_1\right),
	\]
	where $\gamma$ is the learning rate. Rearrange the equation, we have
	\[
	\x^{t+1} - \e_1 = (1-2\gamma)\left(\x^{t} - \e_1\right),
	\]
	which implies that $\left(\x^{t} - \e_1\right)$ is a geometric progression. Hence, we have
	\begin{equation}\label{eq:client1}
% 	\[
	\x^{k\left(t_1+t_2\right)+t_1} = \left(1-2\gamma\right)^{t_1}\left(\x^{k\left(t_1+t_2\right)}-\e_1\right) + \e_1.
% 	\]
	\end{equation}
	Applying the same analysis on iterations where client $2$ is available,  we have
	\begin{equation}\label{eq:client2}
	\x^{\left(k+1\right)\left(t_1+t_2\right)} = \left(1-2\gamma\right)^{t_2}\left(\x^{k\left(t_1+t_2\right)+t_1}-\e_2\right) + \e_2.
	\end{equation}
	Substituting (\ref{eq:client1}) into (\ref{eq:client2}), we have
	\begin{align}%\label{eq:phase}
	&\x^{\left(k+1\right)\left(t_1+t_2\right)}
	\new
	=&
	\left(1-2\gamma\right)^{t_1+t_2}\left(\x^{k\left(t_1+t_2\right)}-\e_1\right) 
	+ \left(1-2\gamma\right)^{t_2}\left(\e_1-\e_2\right) + \e_2.\nonumber
	\end{align}
	Based on this recursion formula, we have
	\[
	\x^{k\left(t_1+t_2\right)}
	=
	\left(1-2\gamma\right)^{\left(t_1+t_2\right)k}
	\x^{0} 
	+ \left(1-\left(1-2\gamma\right)^{\left(t_1+t_2\right)k}\right)X.
	\]
	Since $\gamma < \frac{1}{2}$, we have
	$
	\lim_{k \to +\infty}{\x^{k\left(t_1+t_2\right)}} = X
	$.
	Based on L'Hopital's rule, we then have
	$
	\lim_{\gamma\to0^+}{X}=(t_1\e_1+t_2\e_2)/(t_1+t_2)
	$.

	The global minimization objective is
	\[
	f(\x) = \frac{1}{2}\sum_{i=1}^{2}(\x-\e_i)^2 + \frac{1}{2}\sum_{i=1}^{2}\E_{\xi_i\sim \D_i}\left[(\xi_i-\e_i)^2\right],
	\]
	and the minimum is reached when $\x=\x^*=(\e_1+\e_2)/2$.
	Note that $(\e_1+\e_2)/2= (t_1\e_1+t_2\e_2)/(t_1+t_2)$ only when $\e_1=\e_2$ (data distributions are IID) or $t_1=t_2$. Hence, FedAvg  will produce arbitrarily poor-quality results without these inpractical assumptions.
\end{proof}

\subsection{Federated Latest Averaging}
\label{sec:FedLaAvg} 

As shown in Section~\ref{sec:divergence}, intermittent client availability seriously degrades the performance 
of FedAvg. In FL, the overall data distribution is an unbiased mixture of all clients' local 
data distributions. FedAvg can be proven to converge in the full client participation 
scenario~\cite{Yu2018ParallelRS}, because it uses the current gradients of all clients to update 
the global model. This makes the training data distribution in each iteration consistent with 
the overall data distribution. However, due to the intermittent client availability, some clients 
are selected to participate in the training process more frequently, introducing the bias into 
 training data. To mitigate the bias problem, we imitate the full client participation scenario, 
 and attempt to leverage the gradient information of all clients for model training in each iteration. 
 The difficulty in employing this idea is that as some clients are absent from the training due to 
 being either unavailable or unselected, we cannot obtain the current gradients of these clients. 
 To resolve the lack of gradient information, we propose a natural and simple idea: 
 \textbf{using the latest gradient of the client when its current gradient is not available.} By doing so, we can eliminate the bias in training data, and establish the convergence result.
 
 \begin{figure*}[t]
	\centering
	
	\includegraphics[width=0.9\textwidth]{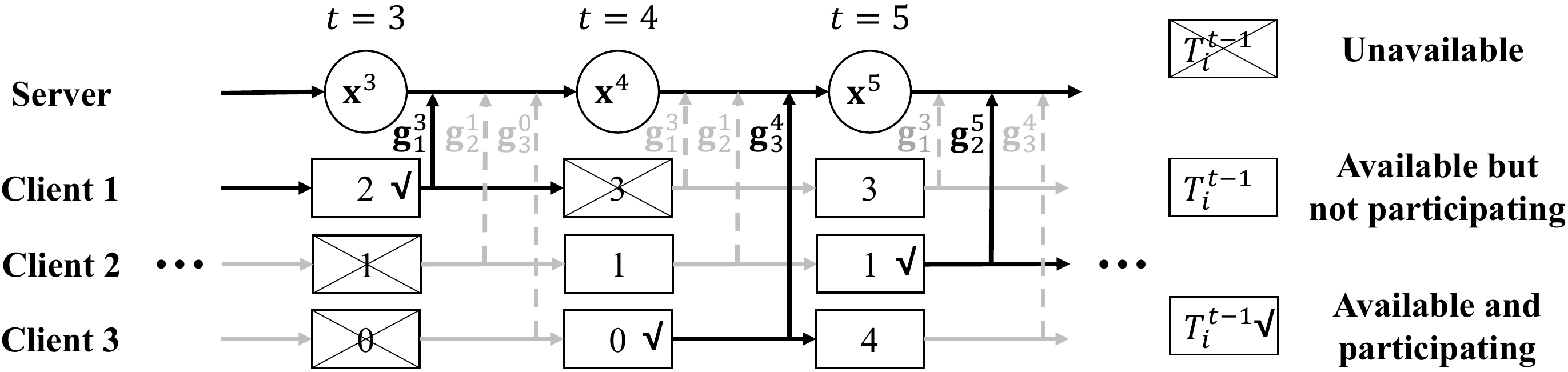}
	
	\caption{
		A simple illustration of FedLaAvg. 
		The cloud server chooses $K = 1$ client with the most outdated gradient from the available clients to participate in each iteration.
		The black lines between iterations mean the chosen client participates in this iteration and uploads its current gradient, while the grey dotted lines mean the cloud server uses the absent client's latest gradient in this iteration. 
		Number $\Tits$ in the rectangles denotes the latest iteration in which the client $i$ participates. 
	}
	\label{fig:overview}
\end{figure*}

We present in Algorithm~\ref{alg:LAS} the detailed procedures of FedLaAvg, 
and give Figure~\ref{fig:overview} for easy illustration.
In each iteration $t$, each selected client $i$ locally calculates the gradient $\git$, and the cloud server maintains the average latest gradient $\gt$ of all clients.
The client selection principle in FedLaAvg is to \textbf{choose the $K$ clients that are absent from the training process for the longest time from the available clients} (Lines~5--7).
Together with Assumption~\ref{as:4}, we can guarantee that each client is selected at least once during any period with $I$ successive iterations, where $I$ is a function of parameters $K$, $N$, and $E$ (please refer to Lemma~\ref{le:I} in Section~\ref{sec:performance} for the details). 
Based on this condition, we can establish an upper bound for the difference between each client's latest gradient and its current gradient, which would be critical for the convergence analysis of FedLaAvg in Section~\ref{sec:performance}. 
To implement this principle, we use $\Tit$ to record the latest iteration before or at $t$ in which client $i$ participates in the training process.
During the aggregation procedure (Lines 8--9), to reduce the aggregation overhead, each selected client uploads the gradient difference: the difference between the gradients computed in the current participating iteration and the previous participating iteration, i.e., $\git - \giTits$, rather than the current gradient $\git$ as in the traditional FedAvg algorithm.
%A naive aggregation requires $\Theta\left(N\right)$ space complexity to store the latest gradients of all clients on the cloud server and $\Theta\left(N\right)$ time complexity to average the latest gradients in each iteration.
Once the gradient difference from each client $i$ is received, the cloud server would update the global gradient (Line~9). 
Following this aggregation method, the cloud server only needs to store the average latest gradient $\gt$ and run $K$ update operations.
It can be proved by induction that at the end of each iteration $t$, the resulting gradient is indeed the average latest gradient:
\begin{equation}\label{eq:gt}
\g^{t}=\frac{1}{N}\sum_{i=1}^{N}\giTit.
\end{equation}
Once the average latest gradient $\gt$ is obtained, the cloud server uses it to update the global model parameters in (\ref{eq:update}).

\begin{algorithm}[tb]
	\caption{Federated Latest Averaging Algorithm}
	\label{alg:LAS}
	\begin{algorithmic}[1]
		
		\STATE {\bfseries Input:} initial model parameters $\x^{0}$; number of clients $N$; number of total iterations $T$; learning rate $\gamma$; proportion of selected clients $\beta$ (i.e., the number of participating clients in each iteration is $K = \beta N$.)
		
		\STATE Initialization:
		$\g^{0}\leftarrow \ze;\,\forall i\in\{1,2,\cdots,N\},\,\g_i^{0}\leftarrow \ze,\,T_i^{ 0 } \leftarrow 0$.
		
		\FOR{$t=1$ {\bfseries to} $T$}
		\STATE $\gt \leftarrow \g^{ t - 1 }$
		
		\STATE $\CC^t\leftarrow$ the set of available clients
		
		\STATE $\B^t\leftarrow$ $K$ clients from $\CC^t$ with the lowest $\Tits$ values
		
		\STATE Update $\Tit$ values: 
		$
		\Tit \leftarrow t,\, \forall i\in \B^t;\; \Tit \leftarrow \Tits,\, \forall i\notin \B^t.
		$
		
		\STATE Each client $i\in \B^t$ calculates local gradient $\git$ and uploads gradient difference $\git - \giTits$.
		
		\STATE Once receiving the gradient information from client $i$, the cloud server calculates the global gradient:
		\[
		\gt\leftarrow \gt + \frac{1}{N}\left(\git - \giTits\right).
		\]
		
		\STATE The cloud server updates the global model parameters:
		\begin{equation}\label{eq:update}
		\x^{t}\leftarrow \x^{t - 1} - \gamma\gt.
		\end{equation}
		\ENDFOR
	\end{algorithmic}
\end{algorithm}

\section{Convergence Analysis}
\label{sec:performance}
In this section, under intermittent client availability, 
we show that FedLaAvg achieves $O(E^{1/2}/(N^{1/4}T^{1/2}))$ 
convergence rate on general non-convex functions with 
a sublinear speedup in terms of the total number of clients.

\subsection{Convergence on Example~\ref{exp:mean estimation}}\label{sec:case}
We first demonstrate that FedLaAvg converges in Example~\ref{exp:mean estimation}, where FedAvg produces an arbitrarily poor-quality result.
The convergence analysis of FedLaAvg for this simple example 
sheds light on the analysis for the case of general 
non-convex optimization in the next subsection. 
\begin{theorem}\label{the:convergence}
	Suppose each client computes the exact (not stochastic) gradient.
	In Example~\ref{exp:mean estimation}, 
	after $T$ iterations, FedLaAvg with the learning rate $\gamma=1/(2\sqrt{T})$ produces a solution $\hat{\x}$ that is within $O(1/\sqrt{T})$ range of the optimal solution $\x^*$:
	$
	\left(\hat{\x} - \x^*\right)^2 =  O(1/\sqrt{T}),
	$
	where we choose $\hat{\x} = \arg \min_{\xt}f(\xt)$ as the output.
\end{theorem}
\begin{proof}[Proof of Theorem~\ref{the:convergence}]	
	We recall that
	\[
	f(\x) = 
	\left(\x - \frac{\e_1+\e_2}{2}\right)^2 + \frac{\left(\e_1-\e_2\right)^2}{4} 
	+ \frac{1}{2}\sum_{i=1}^{2}\E_{\xi_i\sim \D_i}\left[(\xi_i-\e_i)^2\right],
	\]
	where the latter two terms are not associated with the variable $\x$. Hence, we only need to focus on the following part of the loss function:
	$
	\f(\x) = \left(\x - \x^* \right)^2,
	$
	where $\x^* = \frac{\e_1+\e_2}{2} $ is the optimal solution.
	Note that
	\begin{equation}\label{eq:case-main}
	\f ( \x^{ t  } ) - \Fxts 
	=
	\left(\x^{t} - \xts\right)^2
	+2\left(\xts - \x^*\right) \left(\x^{t} - \xts\right).
	\end{equation}
	We calculate the difference of $\x$ between two successive iterations:
	\begin{align}\label{eq:diff}
	\x^{t} - \xts
	&=
	-\frac{\gamma}{2} \left(\g_1^{\Tonet} +\g_2^{\Ttwot}\right)
	=
	-\gamma \left(\xTonet - \e_1 + \xTtwot - \e_2 \right)\new
	&=
	-\gamma \left(\xTonet + \xTtwot - 2\x^* \right),
	\end{align}
	where $\Tit$ is defined in Section~\ref{sec:FedLaAvg}.
	Hence, we have
	\begin{align}\label{eq:prod}
	&
	2\left(\xts - \x^*\right) \left(\x^{t} - \xts \right)
	\new
	=&
	-\gamma\left(2\xts-2\x^*\right) \left(\xTonet+\xTtwot-2 \x^* \right)
	\new
	=&
	-\frac{\gamma}{2}  \left(2\xts  - 2\x^*\right)^2 
	-\frac{\gamma}{2}  \left(\xTonet + \xTtwot - 2 \x^* \right)^2 
	\new
	&+\frac{\gamma}{2}  \left(2\xts - \xTonet - \xTtwot \right)^2.
	\end{align}
	Substituting (\ref{eq:diff}) and (\ref{eq:prod}) into (\ref{eq:case-main}), we have
	\begin{equation}\label{eq:case-main1}
	\f ( \x^{  t  } ) - \Fxts 
	%		=&
	%		\left(\gamma^2 - \frac{\gamma}{2}\right) \left(\xTonet + \xTtwot - 2\x^* \right)^2
	%		-2\gamma \left(\xts - \x^*\right)^2
	%		+\frac{\gamma}{2} \left(2\xts - \xTonet - \xTtwot \right)^2
	%		\new
	%		\overset{(a)}{\leq}&
	\leq
	-2\gamma \left(\xts - \x^*\right)^2
	+\frac{\gamma}{2} \left(2\xts - \xTonet - \xTtwot \right)^2,
	\end{equation}
	which follows from $0 < \gamma \leq 1/2$.
	
	The algorithm starts from model parameters $\x^{0}$. When client $1$ is available, $\x$ moves towards $\e_1$, and when client $2$ is available, $\x$ moves towards $\e_2$. Hence, $\x$ is always within $G/2$ range of $\x^*$:
	\begin{equation}\label{eq:t*}
	-\frac{G}{2} \leq \x^{t} - \x^* \leq \frac{G}{2},\,\forall t \geq 0,
	\end{equation}
	where $G=\max\left\{2\left( \x^{0} - \x^* \right), \left\lvert\e_1-\e_2\right\rvert \right\}$ is the the largest gradient norm during the training process.
	Substituting (\ref{eq:t*}) into (\ref{eq:diff}), we have
	\begin{align}\label{eq:bounded_diff}
	-\gamma G \leq \x^{t} - \xts \leq \gamma G.
	\end{align}
	Referring to the specific client availability model in this example, we have
	$t - \Tit \leq I = \max\left\{t_1, t_2\right\}$, 
	$i = 1,\,2$.
	Therefore, when $t \geq \Tit + 2$, summing (\ref{eq:bounded_diff}) over iterations from $\Tit + 1$ to $t-1$, we have 
	\begin{align}\label{eq:total_diff}
	-\gamma I G \leq \x^{t-1} - \x^{\Tit} \leq \gamma I G, \, i=1,\,2.
	\end{align}
	Note that when $t = \Tit$ or $t=\Tit+1$, the above formula also holds.
	
	Substituting (\ref{eq:total_diff}) into (\ref{eq:case-main1}), we have
	\[
	\f ( \x^{  t  } ) - \Fxts 
	\leq
	-2\gamma \left(\xts - \x^*\right)^2
	+2 \gamma^3 I^2 G^2.
	\]
	Rearranging the formula, we have
	\[
	\left(\xts - \x^*\right)^2
	\leq
	\frac{1}{2\gamma}\left(\Fxts - \f ( \x^{  t  } )\right)
	+\gamma^2 I^2 G^2.
	\]
	Summing this inequility over iterations from 1 to $T$, we have
	\begin{align}\label{eq:case-main4}
	\frac{1}{T}\sum_{t=1}^{T}\left(\xts - \x^*\right)^2
	&\leq
	\frac{1}{2\gamma T}\left(\f(\x^{0}) - \f ( \x^{  T  } )\right)
	+\gamma^2 I^2 G^2
	\new
	&
	\leq
	\frac{1}{2\gamma T}\left(\f(\x^{0}) - \f ( \x^* )\right)
	+\gamma^2 I^2 G^2.
	\end{align}
	Substituting $\gamma=1/(2\sqrt{T})$ into (\ref{eq:case-main4}), we have
	\begin{align}%\label{eq:case-main5}
	\frac{1}{T}\sum_{t=1}^{T}\left(\xts - \x^*\right)^2
	&\leq
	\frac{1}{\sqrt{T}}\left(\f(\x^{0}) - \f ( \x^* )\right)
	+\frac{I^2 G^2}{4T}
	\new
	&=
	\frac{1}{\sqrt{T}}\left(f(\x^{0}) - f ( \x^* )\right)
	+\frac{I^2 G^2}{4T}.
	\nonumber
	\end{align}
	
	Finally, we have
	\begin{align}%\label{eq:case-final}
	&\left(\hat{\x} - \x^*\right)^2 
	\leq
	\frac{1}{T}\sum_{t=1}^{T}\left(\xts - \x^*\right)^2
	\new
	\leq&
	\frac{1}{\sqrt{T}}\left(f(\x^{0}) - f ( \x^* )\right)
	+\frac{I^2 G^2}{4T}
	=O\left(\frac{1}{\sqrt{T}}\right).
	\nonumber
	\end{align}
\end{proof}

\subsection{Convergence on General Non-Convex Functions}
\label{sec:general}
In this subsection, we show the $O(E^{1/2}/(N^{1/4} T^{1/2}))$
convergence of FedLaAvg on general non-convex functions.
First, we introduce Lemma~\ref{le:I} about client participation.
\begin{lemma}\label{le:I}
	Under Assumption~\ref{as:4}, the client selection policy in FedLaAvg guarantees that for each client, the latest participating iteration is at most $I$ iterations earlier than the current iteration:
	$
	t-\Tit \leq I,\,\forall t,\,\forall i,\, 
	\text{where}\,I  = \NK E - 1
	$
\end{lemma}
\begin{proof}[Proof of Lemma~\ref{le:I}]
	Please refer to \myappendix{sec:I}.
\end{proof}

With such a client participation condition, we can derive a key result for analyzing the convergence of FedLaAvg.
\begin{theorem}\label{the:main}
	By setting $\gamma \leq 1/(2L)$ in FedLaAvg, we can derive the following bound on the average expected squared gradient norm under Assumptions~\ref{as:1}--\ref{as:4}:
	\begin{align}
		&
		\frac{1}{T}\sum_{t=1}^{T}\E \left[ \left\lVert \nabla \fxts \right\rVert^2 \right]
		\new
		\leq&
		\underbrace{\frac{2 \gamma I L \left(G^2 + \sigma^2\right)}{\sqrt{N}} 
		+
		\frac{2 \gamma^2 I^2 L^2 G^2}{1 - 2\gamma L} 
		+ 
		4 \gamma^2  I^2 L^2 G^2}_{\text{gradient staleness}}
		\new
		&+
		\underbrace{\frac{4 \gamma \sigma^2 L}{N}}_{\text{gradient variance}}
		+  
		\underbrace{\frac{4}{\gamma T} \left(
		\E \left[ f(\x^{0}) \right]
		- 
		\E \left[ f(\x^*)  \right]
		\right)}_{\text{objective distance}},
		\nonumber
	\end{align}
	where $\x^*$ is the the optimal solution for the general non-convex optimization problem.
\end{theorem}
\begin{proof}[Proof Sketch of Theorem~\ref{the:main}]
	Based on Assumption~\ref{as:1} about smoothness, 
	we decompose the difference of the loss function values in two successive iterations, 
	i.e., $\E[f(\x^t)]-\E[f(\x^{t-1})]$, into several terms.
	With Lemma~\ref{le:I}, we show that the gradient staleness, 
	i.e., difference between the latest gradient and the corresponding current gradient, 
	is bounded. 
	With Assumption~\ref{as:3}, we show that 
	the error related to stochastic gradient variance is also bounded.
	With the gradient staleness bound and the gradient variance bound,
	we further prove an upper bound for each decomposed term mentioned above.
	Finally, we prove the theorem by summing up the bounds over iterations $t$ from $1$ to $T$ and rearranging the resulted inequation.
	For the detailed proof, please refer to \myappendix{sec:sup main}.
\end{proof}
%============================================================

Before presenting our main result, we consider the full client participation setting discussed in \citet{Yu2018ParallelRS}, in which our FedLaAvg reduces to FedAvg.
Since $K=N$, $E=1$, and $I=\lceil N / K\rceil E - 1=0$ in this setting, the gradient staleness term vanishes
and the result in Theorem~\ref{the:main} becomes
% \begin{equation}
\[
	\frac{1}{T} \sum_{t=1}^{T} \E  \left[  \left\lVert  \nabla \fxts \right\rVert ^2  \right] 
    \leq
	\frac{4 \gamma \sigma^2 L}{N} 
	+
	\frac{4}{\gamma T} 
	\left(
	\E \left[ f(\x^{0}) \right]
    -
	\E \left[ f(\x^*)  \right]
	\right).
\]
% \end{equation}
Choosing $\gamma = \sqrt{N} / (2L\sqrt{T})$, when $T\geq N$, we can obtain the $O(1/\sqrt{NT})$ convergence, which is consistent with the linear speedup in terms of  $N$ as proven in \citet{Yu2018ParallelRS}.

For the intermittent client availability setting considered in this work, FedLaAvg achieves a sublinear speedup by choosing appropriate hyperparameters. For easy illustration, we define the loss difference between the initial solution $\x^0$ and the optimal solution $\x^*$ as $B=f(\x^{0})-f(\x^*)$. 
In addition, we recall that $\beta = K/N$ is the proportion of the selected clients in each iteration.
\begin{corollary}\label{co:main}
	By choosing the learning rate 
	$\gamma= \frac{\beta^{\frac{1}{2}} N^{\frac{1}{4}}}
	{2L E^{\frac{1}{2}} T^{\frac{1}{2}}} $ 
	and requiring $\gamma \leq 1/(4L)$ in FedLaAvg, we have the following convergence result:
% 	\begin{equation}
    \[
		\frac{1}{T} \sum_{t=1}^{T} \E  \left[  \left\lVert  \nabla \fxts \right\rVert ^2  \right] 
		=
		O\left(
		\frac{ E^{\frac{1}{2}} \left(G^2 + \sigma^2 + LB\right) }{ N^{\frac{1}{4}}  T^{\frac{1}{2}} \beta^{\frac{1}{2}} } 
		+
		\frac
		{ E G^2 N^{\frac{1}{2}} }
		{T \beta } 
		\right).
	\]
% 	\end{equation}
	When $T \geq  E N^{3/2}/\beta $,  we further obtain the sublinear speedup with respect to the total number of clients:
	\[
		\frac{1}{T} \sum_{t=1}^{T} \E  \left[  \left\lVert  \nabla \fxts \right\rVert ^2  \right] 
		=
		O\left(
		\frac{ E^{\frac{1}{2}} \left(G^2 + \sigma^2 + LB\right) }{ N^{\frac{1}{4}}  T^{\frac{1}{2}} \beta^{\frac{1}{2}} } 
		\right)
		=
		O\left(
		\frac{E^{\frac{1}{2}}}{ N^{\frac{1}{4}} T^{\frac{1}{2}} } 
		\right).
	\]
\end{corollary}
\begin{proof}[Proof of Corollary~\ref{co:main}]
	Please refer to \myappendix{sec:supco}.
\end{proof}

\section{Extension to Communication Round-Based Setting}
\label{sec:multiround}
In practical FL deployment, 
for communication efficiency,
%~\cite{McMahan2016CommunicationEfficientLO, LOCO, private_communication},
each participating client is allowed to perform 
multiple local training iterations 
before uploading the accumulated local model update 
in each communication round~\cite{
  McMahan2016CommunicationEfficientLO}. 
Although we focused in earlier sections on the 
case where participating clients communicate every iteration, 
FedLaAvg can be trivially extended to support 
multiple local iterations per communication round.

\begin{algorithm}[t]
	\caption{The Communication Round-Based Federated Latest Averaging Algorithm}
	\label{alg:multiround}
	\begin{algorithmic}[1]
		\STATE {\bfseries Input:} 
		initial model parameters $\x^{0}$; 
		number of clients $N$; 
		number of total rounds $R$; 
		learning rate $\gamma$; 
		number of local iterations $C$;
		proportion of selected clients $\beta$ 
		(i.e., the number of participating clients in each round 
		is $K = \beta N$.)
		
		\STATE Do initialization:
		\[\uu^{0}\leftarrow \ze,\,\forall i\in\{1,2,\cdots,N\},\,\uu_i^{0}\leftarrow \ze,\,R_i^{  0  } \leftarrow 0.\]
		
		\FOR{$r=1$ {\bfseries to} $R$}
		\STATE $\uu^{r}\leftarrow \uu^{r-1}$
		
		\STATE $\Ch^r\leftarrow$ the set of available clients
		
		\STATE $\Bh^r\leftarrow$ $K$ clients from $\Ch^r$ with the lowest $\Rirs$ values 
		
		\STATE Update $\Rir$ values:
% 		\begin{equation}\label{Rir}
		\[
		\Rir \leftarrow r,\,\forall i\in \Bh^r;\, \Rir \leftarrow \Rirs,\,\forall i\notin \Bh^r.
		\]
% 		\end{equation}
		
		\STATE Each client $i\in\Bh^r$ calculates the accumulated local model update $\uir$ and uploads the update difference $\uir - \uiRirs$ in parallel.
		
		\STATE Once receiving the update information from client $i$, the cloud server calculates the global update:
% 		\begin{equation}\label{eq:mul-avg}
		\[
		\uu^{r}\leftarrow \uu^{r} + \frac{1}{N}\left(\uir - \uiRirs\right).
		\]
% 		\end{equation}
		
		\STATE The cloud server updates the global model parameters:
% 		\begin{equation}\label{eq:mul-update}
		\[
		\x^{rC}\leftarrow \x^{rC - C} + \ur.
		\]
% 		\end{equation}
		\ENDFOR
	\end{algorithmic}
\end{algorithm}

We introduce some notations to represent the training 
process of the communication round-based FedLaAvg. 
$C$ denotes the local iteration number. 
Let $\Ch^r\triangleq \bigcap_{t=(r-1)C+1}^{rC}\CC^t$
be the set of available clients in round $r$. 
Each client $i$ observes the stochastic gradient 
$\git$ on the local model parameters in each 
local iteration, and accumulates these 
stochastic gradients to obtain the local model update 
$\uir \triangleq -\gamma \sum_{t=(r-1)C+1}^{rC} \git$ 
in round $r$.
After collecting the local model updates 
at the end of round $r$, 
(i.e., in iteration $t=rC$), the cloud server calculates the 
global model parameters $\xrC$.
Similar to the notation $\Tit$, we use 
$\Rir$ to denote the latest round where client $i$ 
is available before or in round $r$.
We define $r^t\triangleq \lfloor (t-1)/C \rfloor + 1$
as the round that iteration $t$ belongs to. 
The one-iteration-per-round scenario 
corresponds to the special case using the specific 
notations: 
$C=1$, $r^t = t$, $\Ch^{r^t} = \CC^t$, 
$\uirt = -\gamma \git$, and $\Rirt = \Tit$.

To apply FedLaAvg to the multiple-iteration-per-round scenario, 
we need only replace the cached gradient $\g^{t}$ in
Algorithm~\ref{alg:LAS} with cached update $\ur$. 
The communicated gradient difference $\git - \giTits$ is 
replaced with the update difference $\uir - \uiRirs$. 
The algorithm is formally presented in Algorithm~\ref{alg:multiround}.

As we consider multiple local iterations per communication round, 
client availability should be measured in rounds. 
Hence we replace Assumption~\ref{as:4} with Assumption~\ref{as:5}:
\begin{assumption}\label{as:5}
	Minimal availability: each client $i$ is available at least once in any period with $E$ successive rounds: 
	\[
	\forall i,\,\forall r,\,\exists \hat{r} \in \{r, r+1,\cdots,r+E-1\},\text{ such that } i\in \Ch^{\hat{r}}.
	\]
\end{assumption}
Under Assumption~\ref{as:5}, we can establish the 
following convergence rate of the extended algorithm:
\begin{theorem}\label{co:multiround}
	We recall that $B=f(\x^{0})-f(\x^*)$ and $\beta = K/N$. Let the communication round-based FedLaAvg execute $R$ rounds.
	By choosing the learning rate $\gamma=( \beta^{1/2} N^{1/4} )/(2L C E^{1/2} R^{1/2} )  $ and requiring $\gamma \leq 1/(4L)$, we have the following convergence result:
% 	\begin{equation}
	\[
	\frac{1}{R} \sum_{r=1}^{R} \E  \left[  \left\lVert  \nabla \fxrCs \right\rVert ^2  \right] 
	=
	O\left(
	\frac{ E^{\frac{1}{2}}  \left(G^2+\sigma^2 + LB\right)}
	{ N^{\frac{1}{4}} R^{\frac{1}{2}} \beta^{\frac{1}{2}} }   
	+  
	\frac{ E  G^2  N^{\frac{1}{2}} }{  R \beta  }   
	\right).
	\]
% 	\end{equation}
	When $R\geq E N^{3/2}/\beta $,  we further have
% 	\begin{equation}
	\[
	\frac{1}{R}\sum_{r=1}^{R}\E \left[ \left\lVert \nabla \fxrCs \right\rVert^2 \right]
	=
	O\left(
	\frac{ E^{\frac{1}{2}}  \left(  G^2  + \sigma^2  \right)  +  BL}
	{ \beta^{\frac{1}{2}} N^{\frac{1}{4}} R^{\frac{1}{2}} }   
	\right) 
	=
	O\left(
	\frac{ E^{\frac{1}{2}}}
	{ N^{\frac{1}{4}} R^{\frac{1}{2}} }   
	\right).
	\]
% 	\end{equation}
\end{theorem}
\begin{proof}[Proof of Theorem~\ref{co:multiround}]
	Please refer to \myappendix{sec:multiround proof}.
\end{proof}

\section{Experiments}
\label{sec:experiment}

In this section, we evaluate the performance of FedLaAvg in different
tasks, datasets, models, and availability settings.

\subsection{Experimental Setups}

\subsubsection{Federated learning tasks.}
We choose the following two tasks for evaluation.

\textbf{Image classification.}
We first take an image classification task
over the MNIST dataset~\cite{mnist},
to validate the convergence of FedLaAvg, 
and the divergence of FedAvg even in a simple convex setting.
The dataset consists of 70000 28$\times$28 grey images of 
hand-written digits,
where 60000 images are for training and the other 10000 
ones are for testing.
In this task, we adopt the multinomial logistic regression model, 
which has a convex optimization objective.
To simulate non-IID data distribution, 
we set each client to hold only images of one certain digit,
and the number of clients holding the same digit to be $N/10$.
To simulate data unbalance, 
we let the number of samples on each client roughly 
follow a normal distribution with mean $6\times 10^4/N$ 
and variance $(1\times 10^4/N)^2$.

\textbf{Sentiment analysis.}
In the general non-convex case,
we take a binary text sentiment classification task
on the Sentiment140 dataset~\cite{sentiment140}.
The dataset consists of 1600000 tweets 
collected from 659775 twitter users.
In this task,
we use a two-layer LSTM binary classifier 
with 16 hidden units and pretrained 25D GloVe embeddings~\cite{glove}
as the classification model.
We naturally partition this dataset by 
letting each Twitter account correspond to a client.
We keep only the clients who hold more than 40 samples 
and get 1473 clients,
and randomly select one-tenth of the data as the test dataset.

\begin{figure}[t]
	\includegraphics[width=0.48\textwidth]{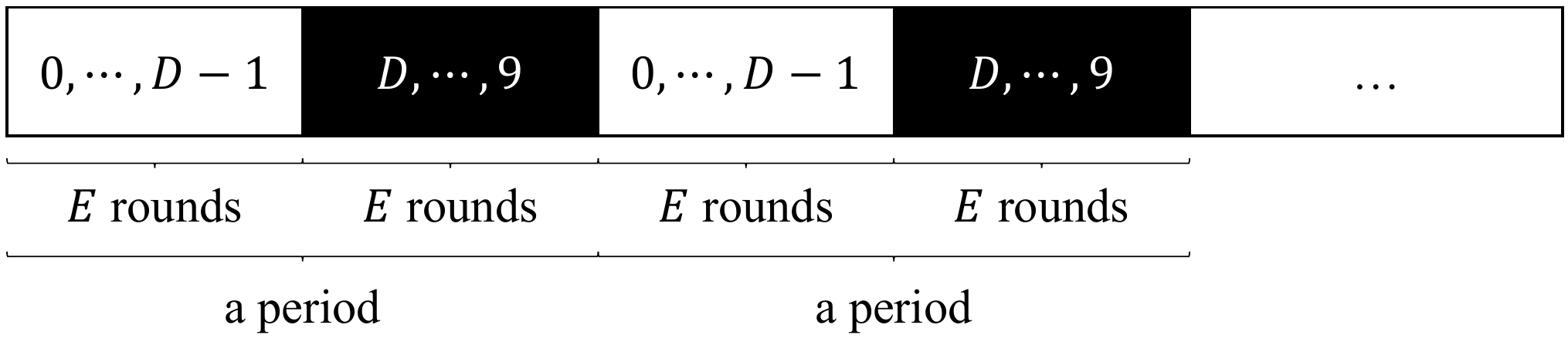}
	\caption{The availability setting with diurnal pattern adopted in 
	the MNIST experiments. 
	Clients holding digits $0,\cdots,D-1$ are available 
	in white grids, while the remaining clients are available 
	in black grids.}
	\label{fig:gantt}
\end{figure}
\begin{figure}[t]
	\includegraphics[width=0.48\textwidth]{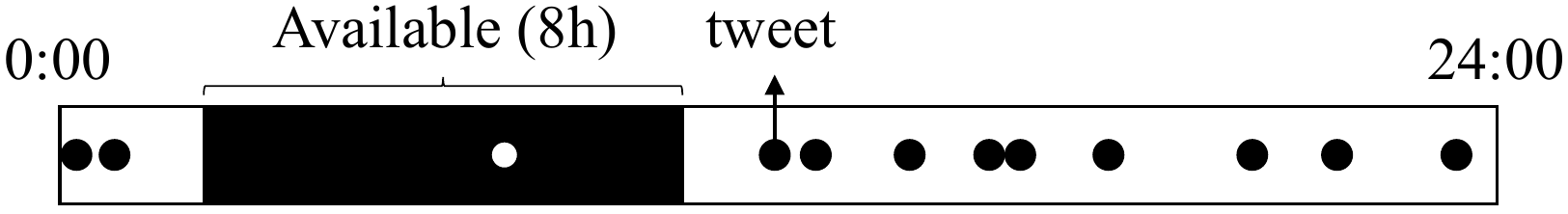}
	\caption{The availability setting adopted in the Sentiment140 experiments.
	The black or white dots denote that the client
	sent a tweet at this moment one day.}
	\label{fig:sentiment140_availability}
\end{figure}

\begin{figure*}[t]
	\centering	
	\subfigure[Legends]{\label{mnist_legends_tl}\includegraphics[width=0.25\textwidth]{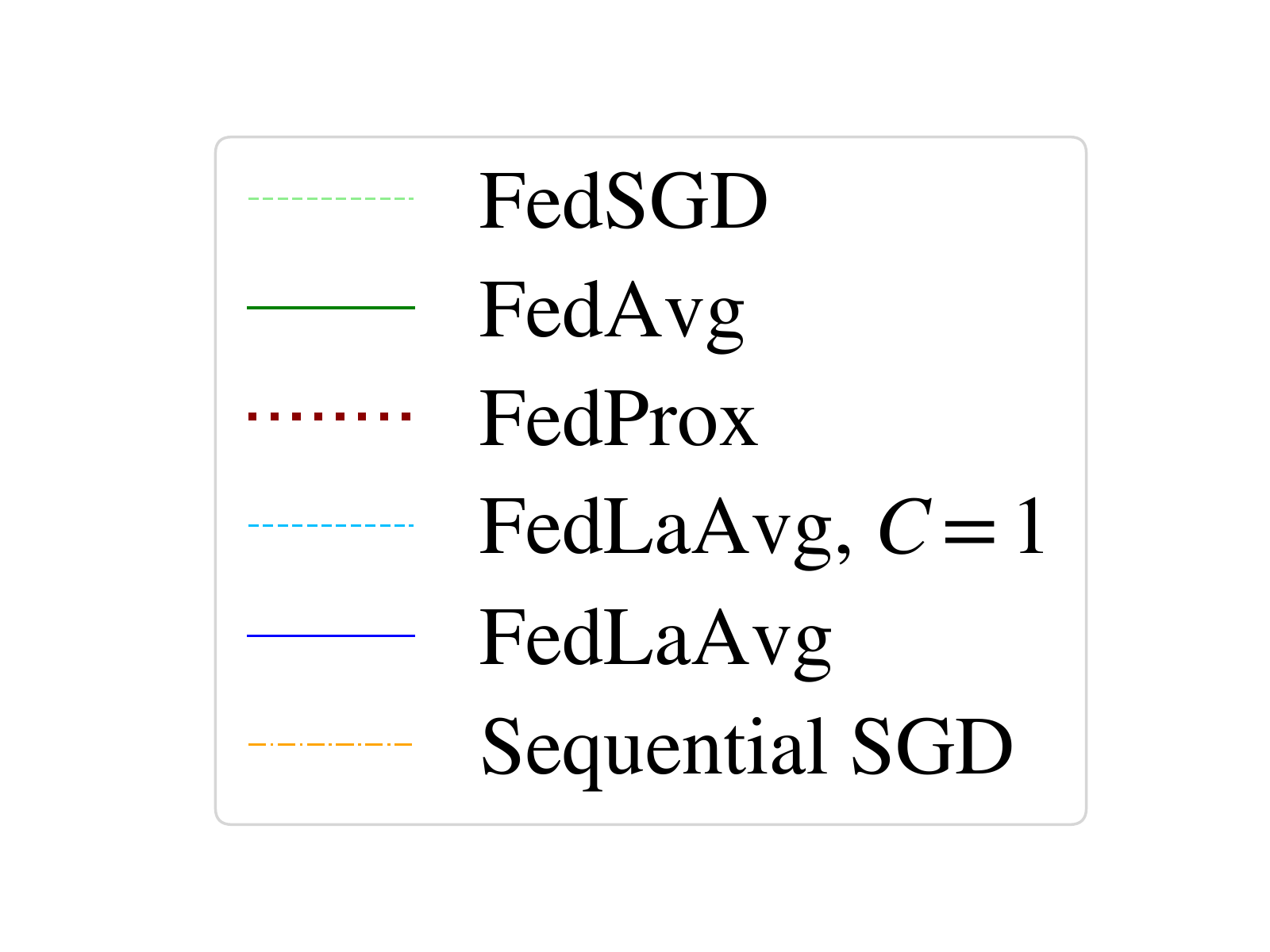}}
	\subfigure[$E=100$, $D=3$]{\label{fig:mnist_E100_D3_tl}\includegraphics[width=0.3\textwidth]{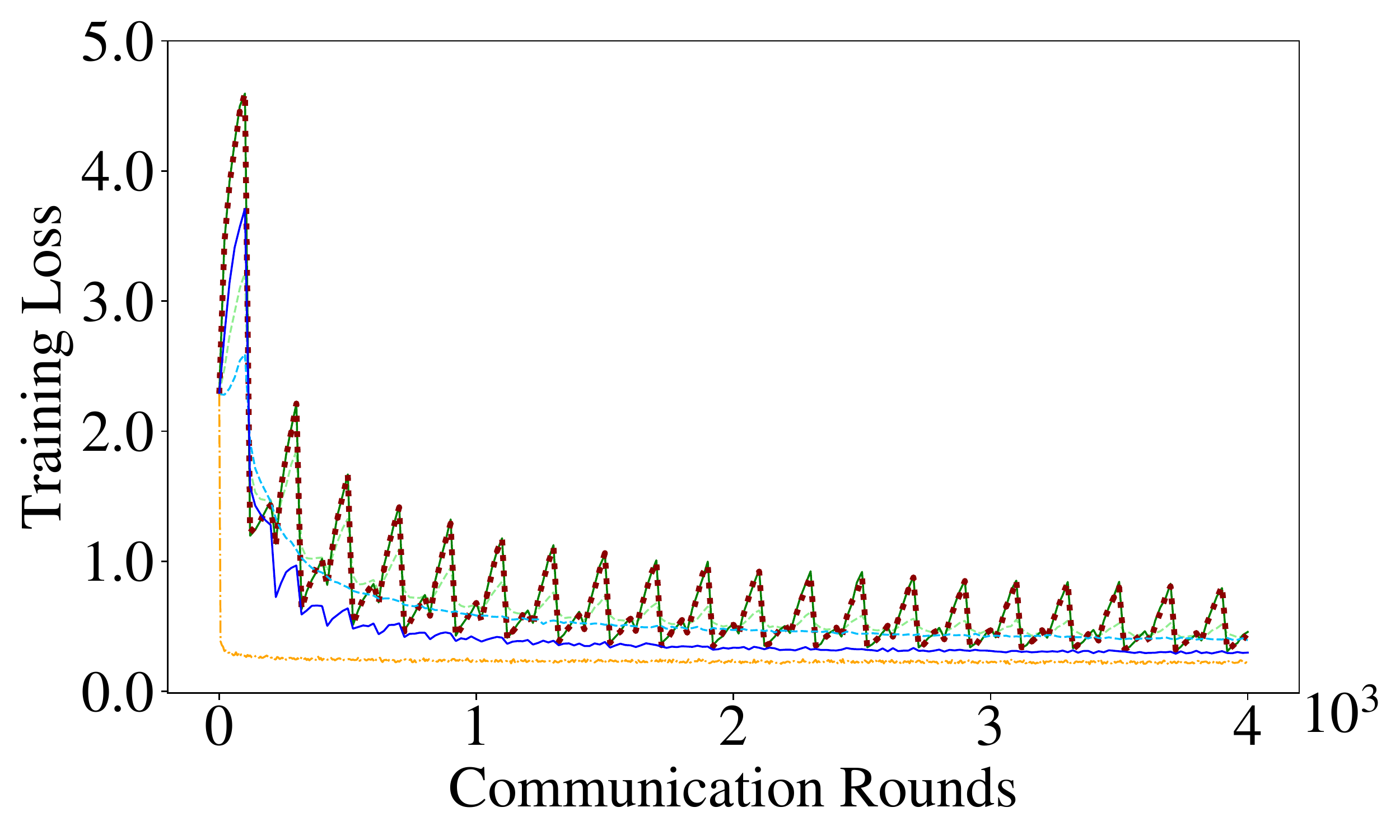}}
	\subfigure[$E=100$, $D=5$]{\label{fig:mnist_E100_D5_tl}\includegraphics[width=0.3\textwidth]{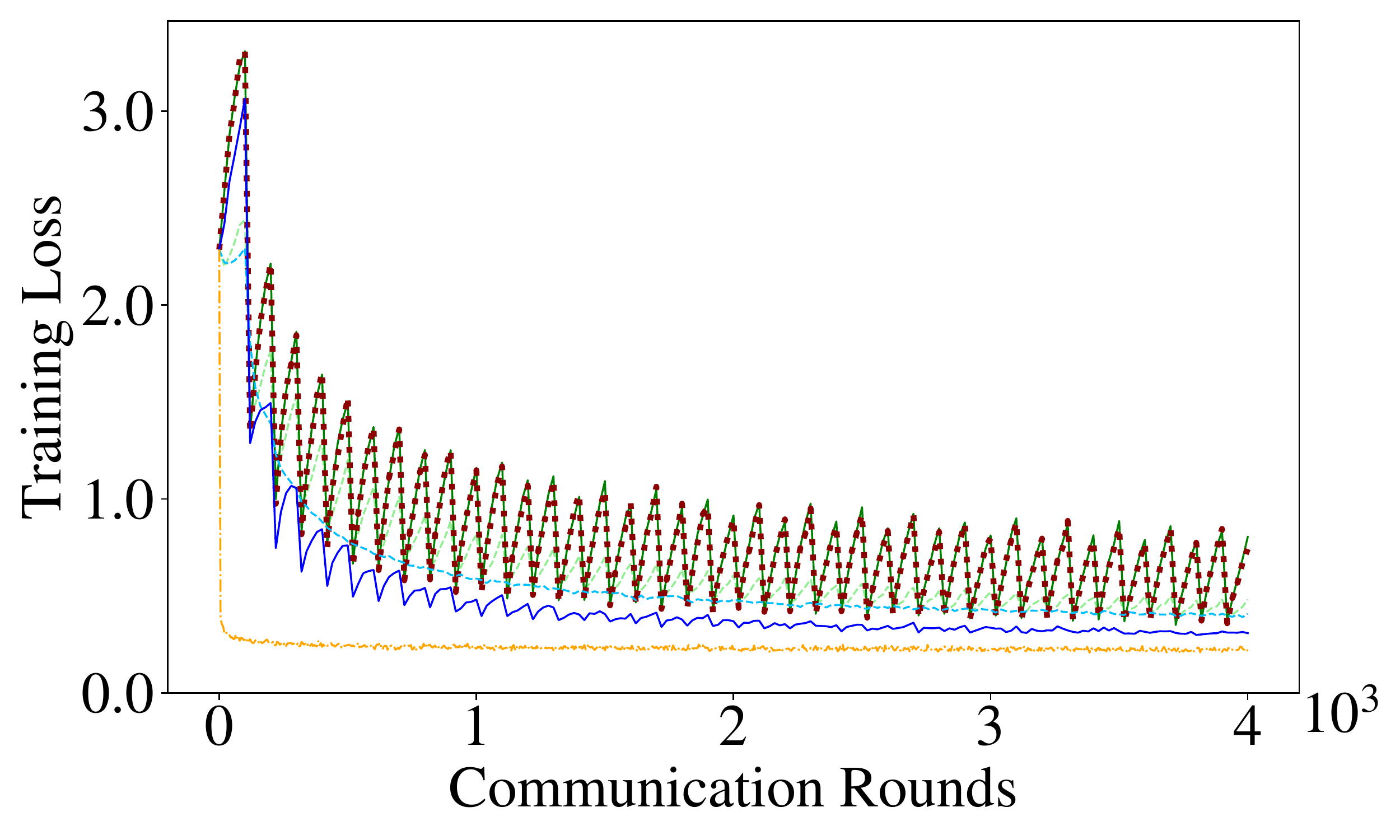}}
	\subfigure[$E=50$, $D=1$]{\label{fig:mnist_E50_D1_tl}\includegraphics[width=0.3\textwidth]{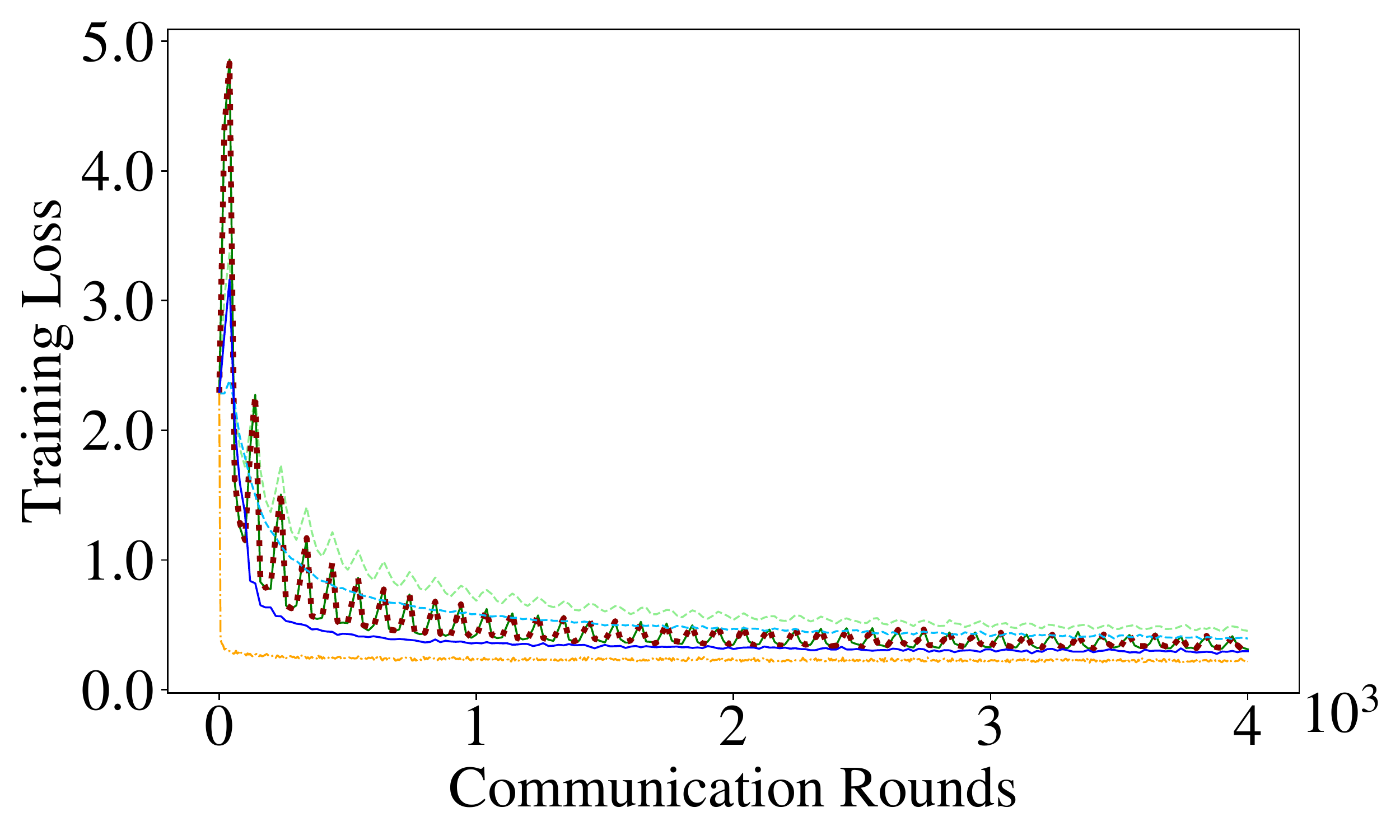}}
	\subfigure[$E=100$, $D=1$]{\label{fig:mnist_E100_D1_tl}\includegraphics[width=0.3\textwidth]{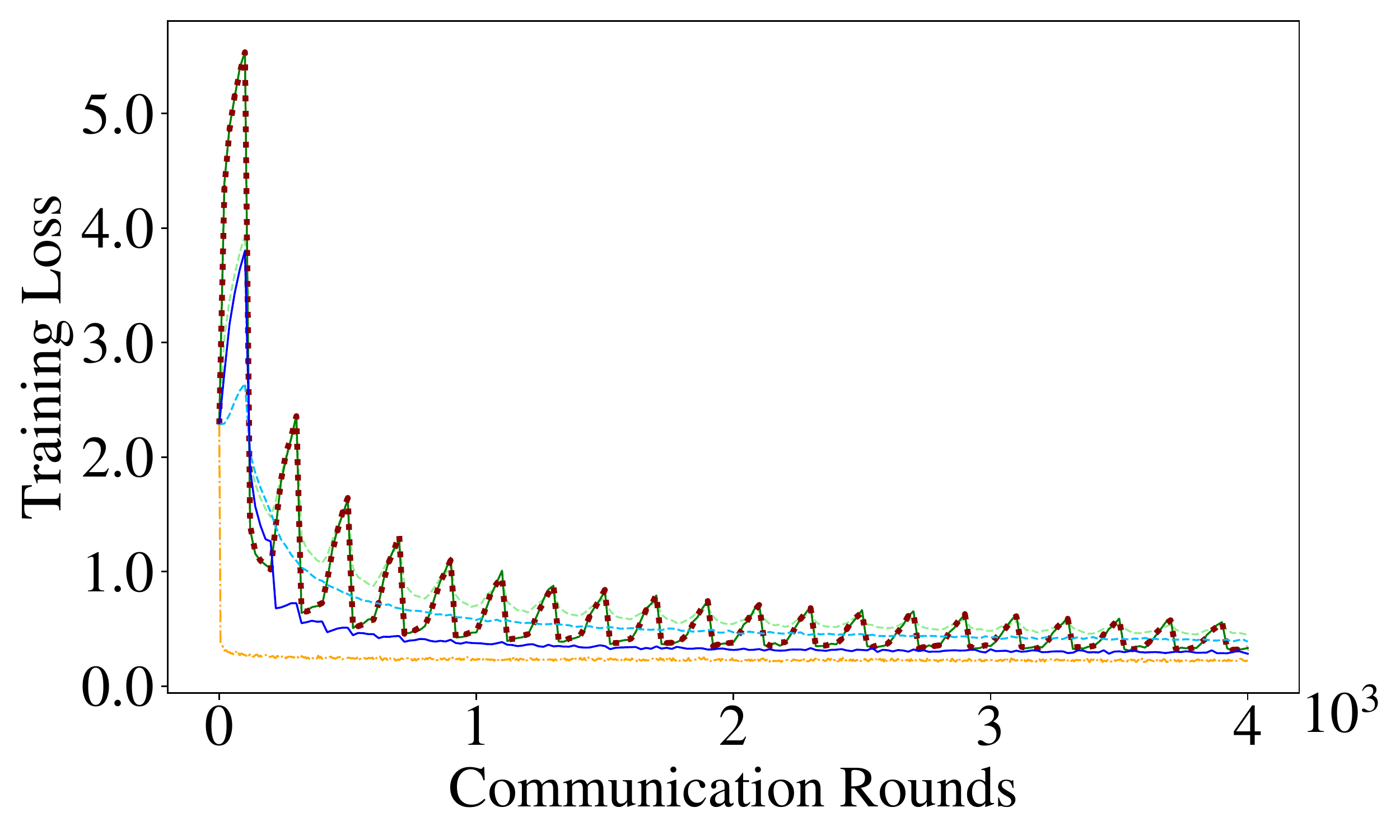}}
	\subfigure[$E=200$, $D=1$]{\label{fig:mnist_E200_D1_tl}\includegraphics[width=0.3\textwidth]{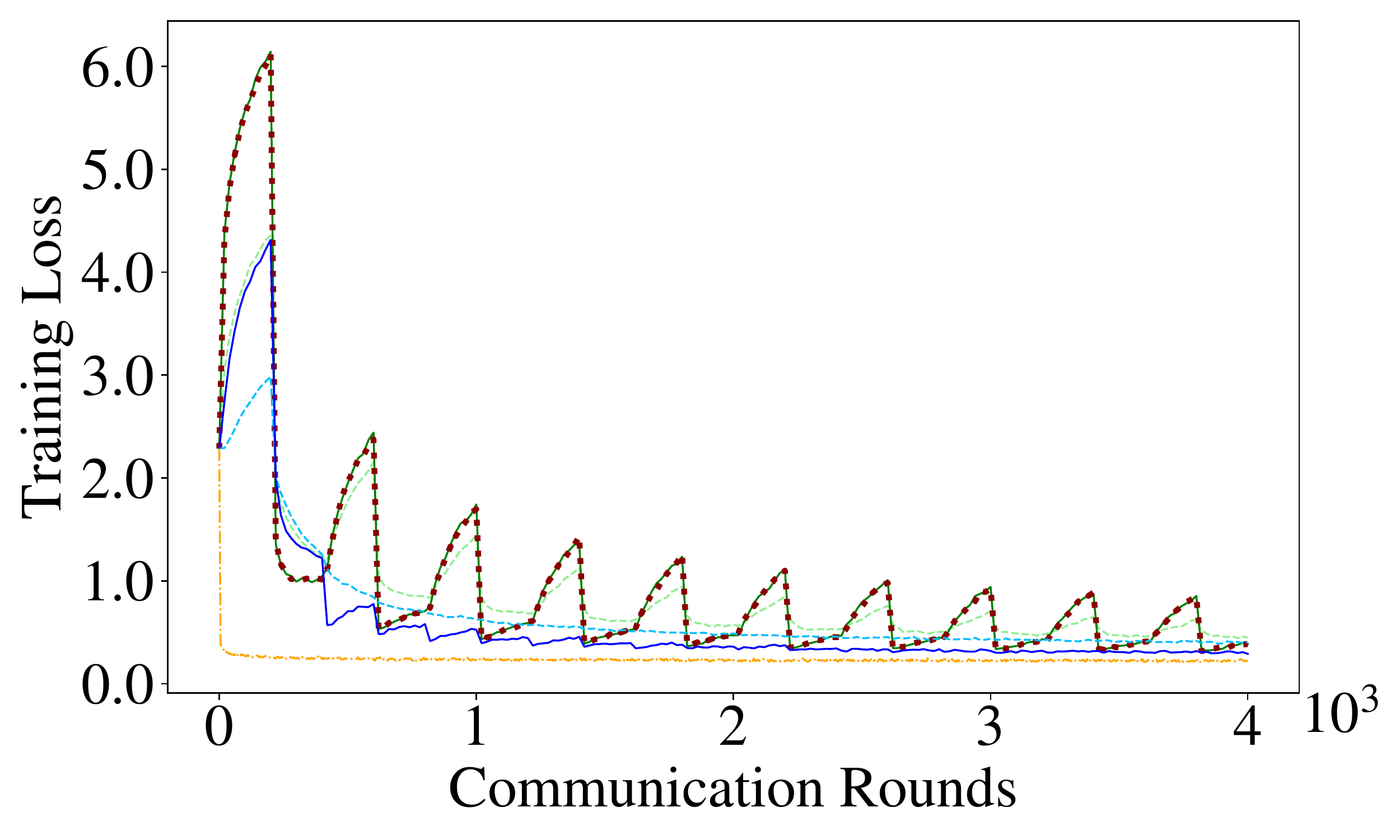}}
	\caption{Training losses of
	FedSGD, FedAvg, FedProx, FedLaAvg, and sequential SGD
	in the MNIST image classification task 
	with different client availability settings.}
	\label{fig:mnist_tl}
\end{figure*}

\subsubsection{Baselines}
We choose the following baselines for comparison.

\textbf{FedAvg and FedSGD.}
FedAvg was proposed by the seminal work of FL~\cite{McMahan2016CommunicationEfficientLO},
which naively uses the average of the collected local model updates 
to update the global model.
FedSGD is a special case of FedAvg with only one local iteration.

\textbf{FedProx.}
FedProx~\cite{FedProx} is the first variant of FedAvg with guaranteed convergence
while allowing partial client participation.
The main modification to FedAvg is that
FedProx adds a quadratic proximal term to 
explicitly limit the local model updates.
We follow the practice of the original work~\cite{FedProx}
to set the proxy term coefficient as 1.0 for MNIST and 0.01 for Sentiment140.

\textbf{Sequential SGD.}
We run the standard SGD algorithm to train the 
global model using the whole dataset, 
and set the total number of iterations in each round as $\beta N C$.
This ensures the same amount of computation per round
with the distributed algorithms. 
The result of sequential SGD can be regarded as 
the ideal solution for the optimization problem.

\subsubsection{Availability Settings and Other Settings}
We simulate the client availability as follows.

In the MNIST image classification experiments, 
we adopt the typical subcase of our 
intermittent client availability model with the 
diurnal pattern~\cite{Bonawitz2019TowardsFL,Yang2018APPLIEDFL,Eichner2019SemiCyclicSG},
depicted in Figure~\ref{fig:gantt}.
In the white grids, the clients holding digits $0, 1,\cdots, D-1$ 
are available for $E$ rounds, 
and in the black grids, the remaining clients are available 
for the next $E$ rounds. 
This setting captures that  the
client availability correlates with the local data distribution 
in practice, and $D$ controls the degree of
such availability heterogeneity.

\begin{figure*}[t]
	\centering
	\subfigure[Training loss with different $N$]{\label{fig:mnist_N_tl}\includegraphics[width=0.4\textwidth]{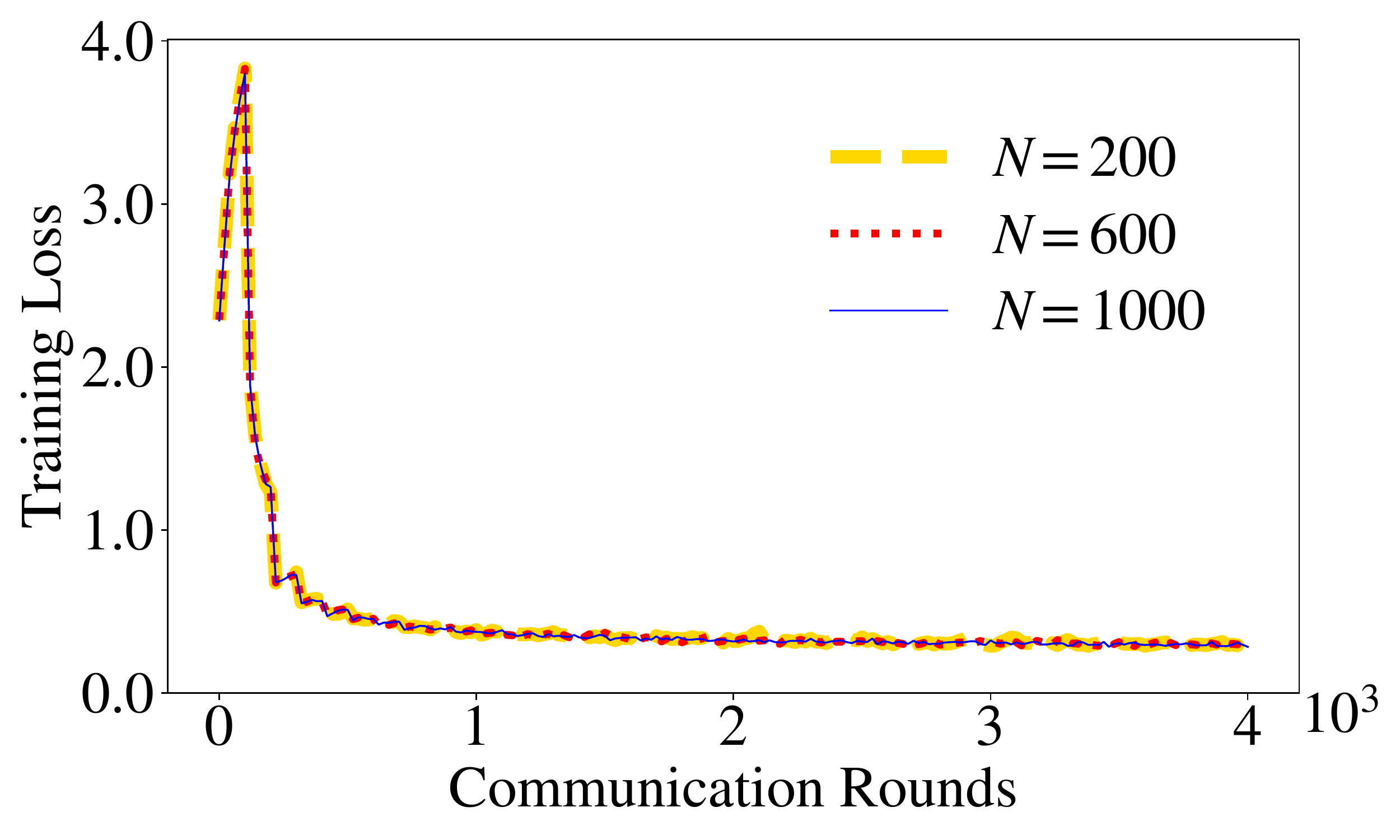}}
	\subfigure[Training loss with different $\beta$]{\label{fig:mnist_beta_tl}\includegraphics[width=0.4\textwidth]{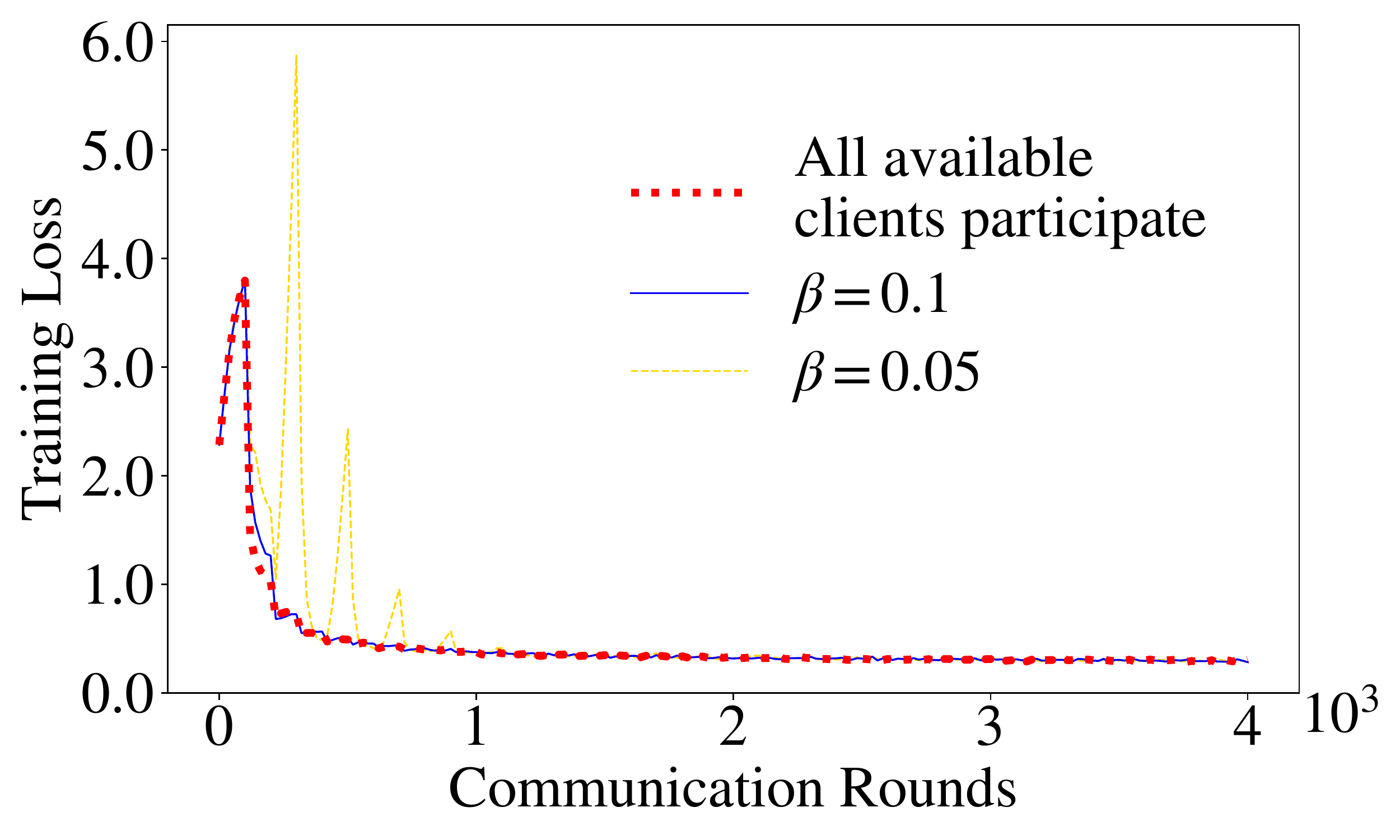}}
	\caption{Training losses of FedLaAvg on MNIST dataset by varying 
	the total number of clients $N$ and the proportion of selected clients $\beta$.}
	\label{fig:mnist_self_test-acc}
\end{figure*}

\begin{figure*}[tb]
	\centering	
	\subfigure[Legends]{\label{sentiment_legends_tl}\includegraphics[width=0.25\textwidth]{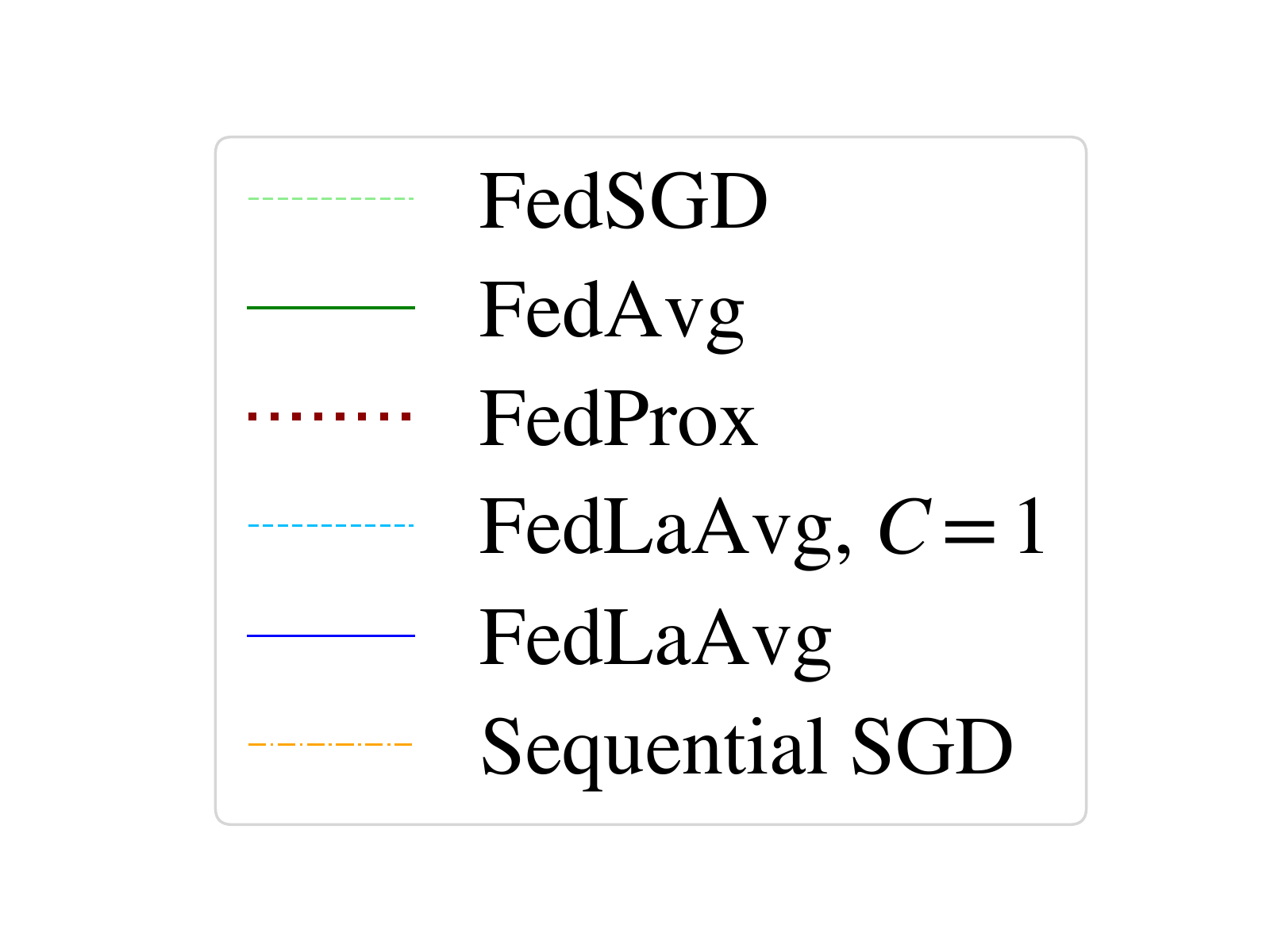}}
	\subfigure[$E=120$, $\alpha=0.25$]{\label{fig:sentiment_E120_a0.25_tl}\includegraphics[width=0.3\textwidth]{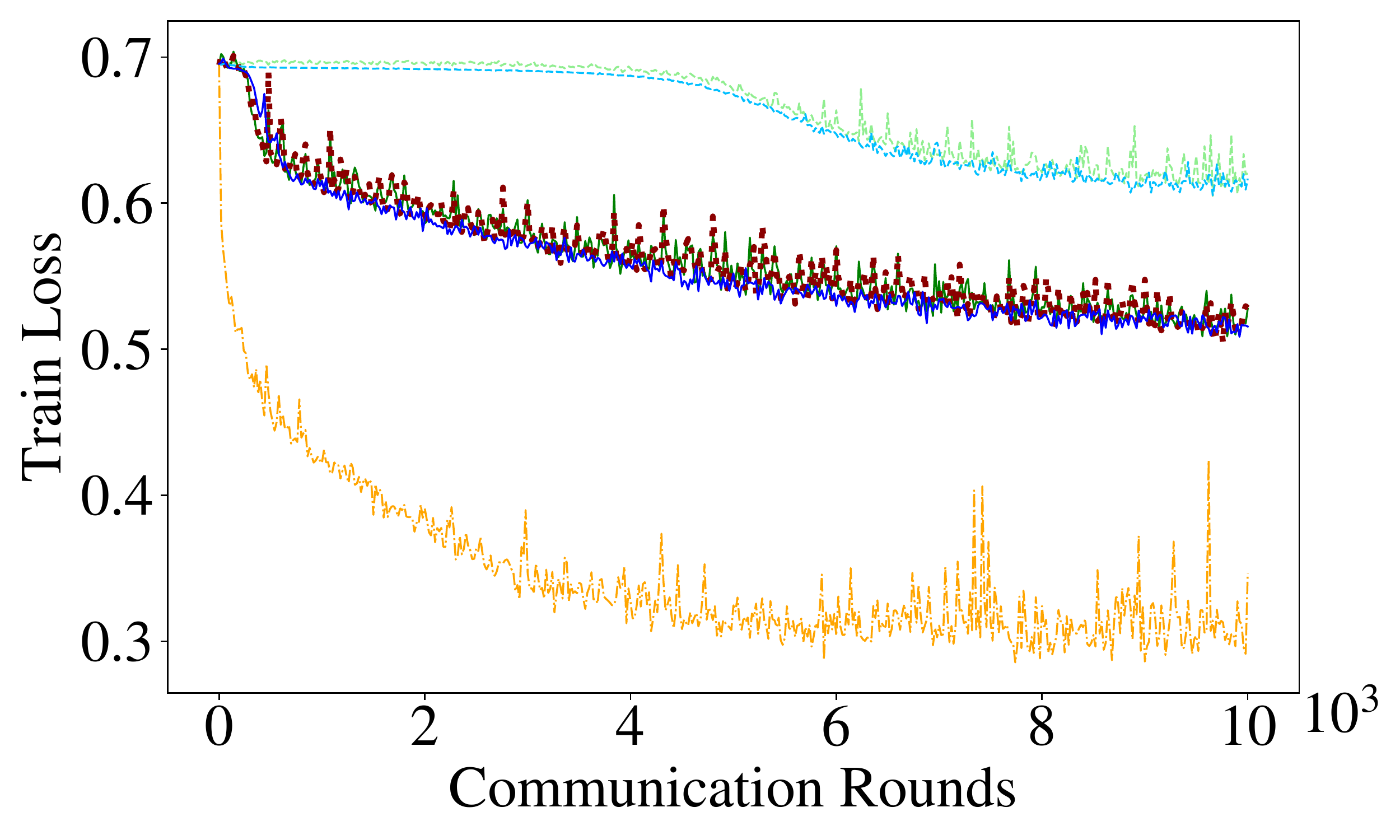}}
	\subfigure[$E=120$, $\alpha=0.5$]{\label{fig:sentiment_E120_a0.5_tl}\includegraphics[width=0.3\textwidth]{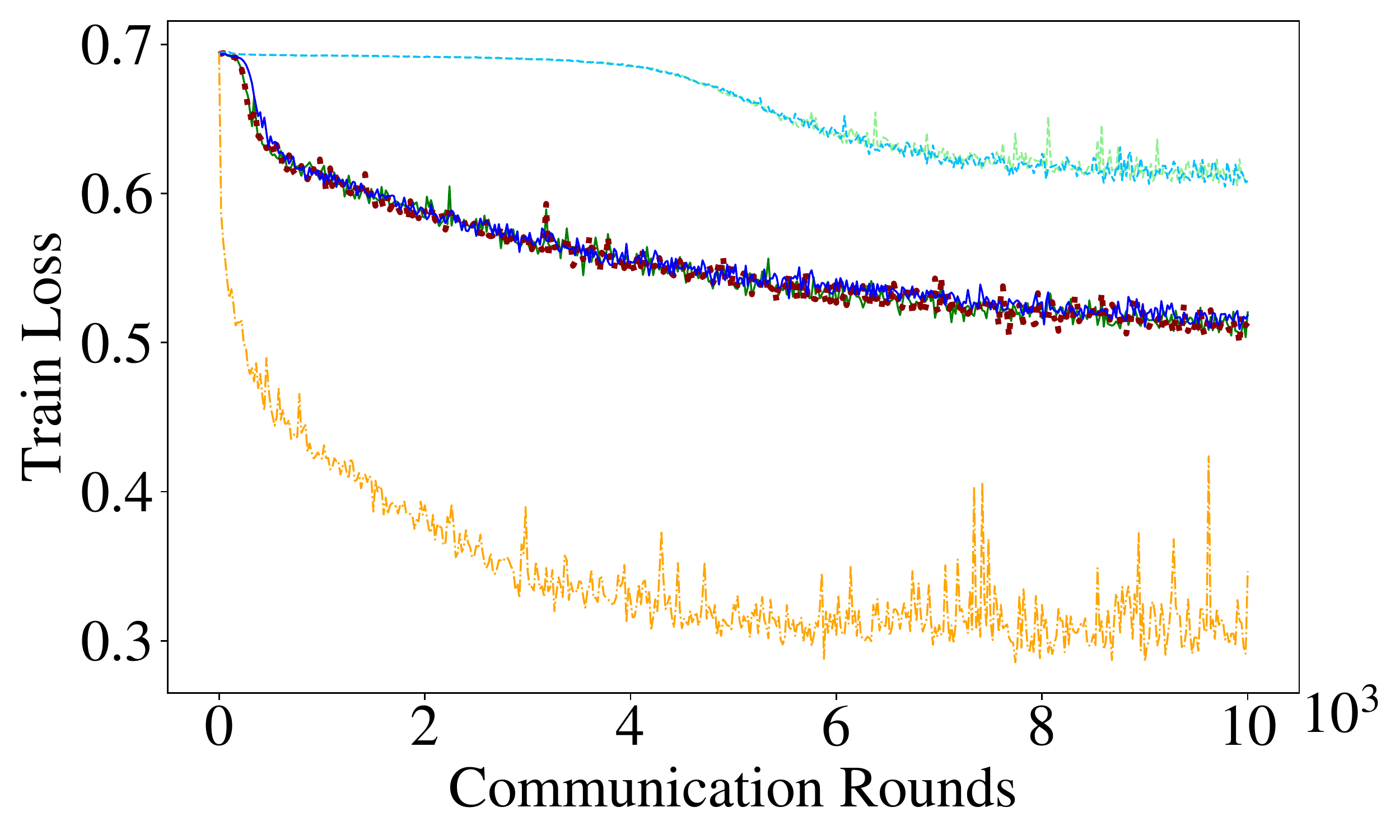}}
	\subfigure[$E=24$, $\alpha=0$]{\label{fig:sentiment_E24_a0_tl}\includegraphics[width=0.3\textwidth]{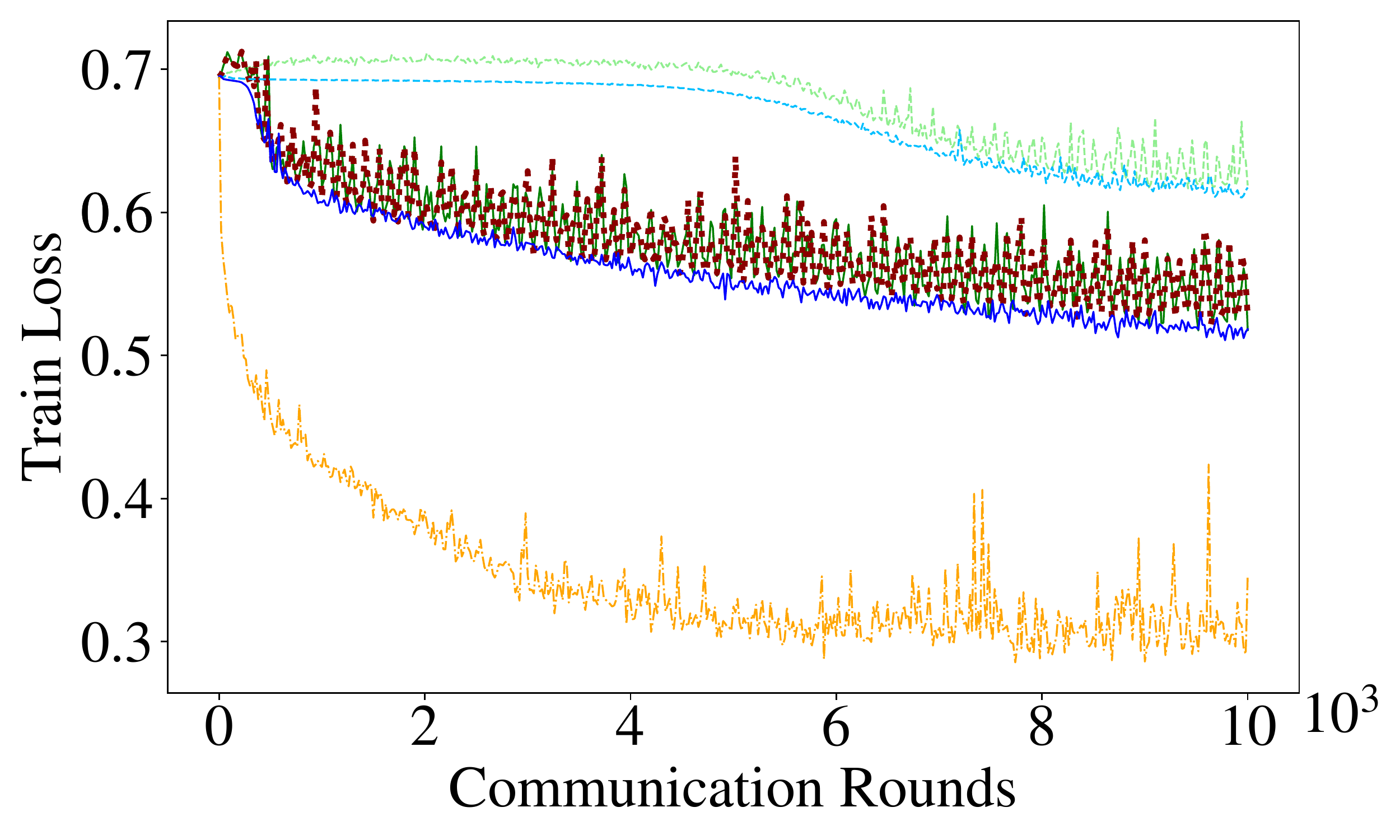}}
	\subfigure[$E=120$, $\alpha=0$]{\label{fig:sentiment_E120_a0_tl}\includegraphics[width=0.3\textwidth]{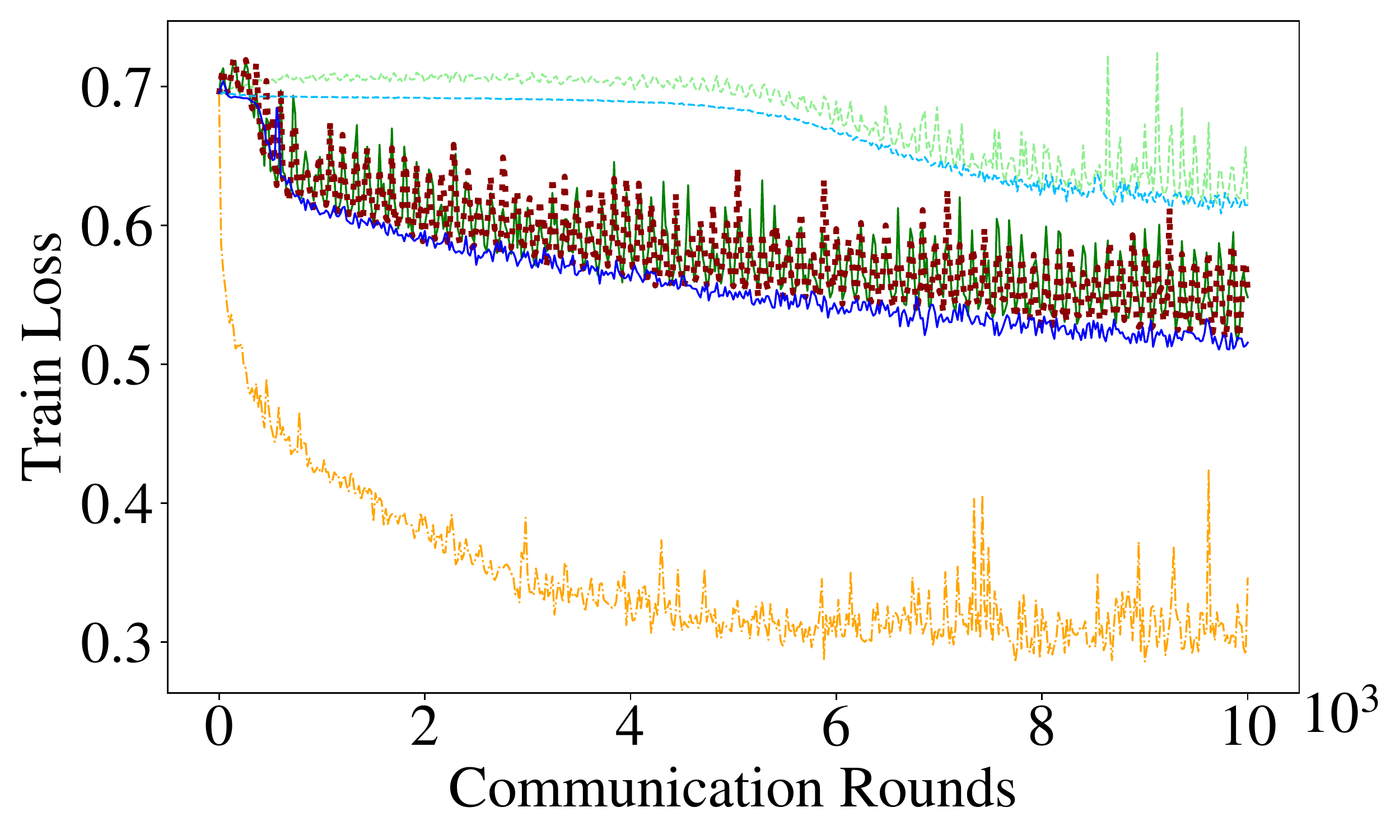}}
	\subfigure[$E=240$, $\alpha=0$]{\label{fig:sentiment_E240_a0_tl}\includegraphics[width=0.3\textwidth]{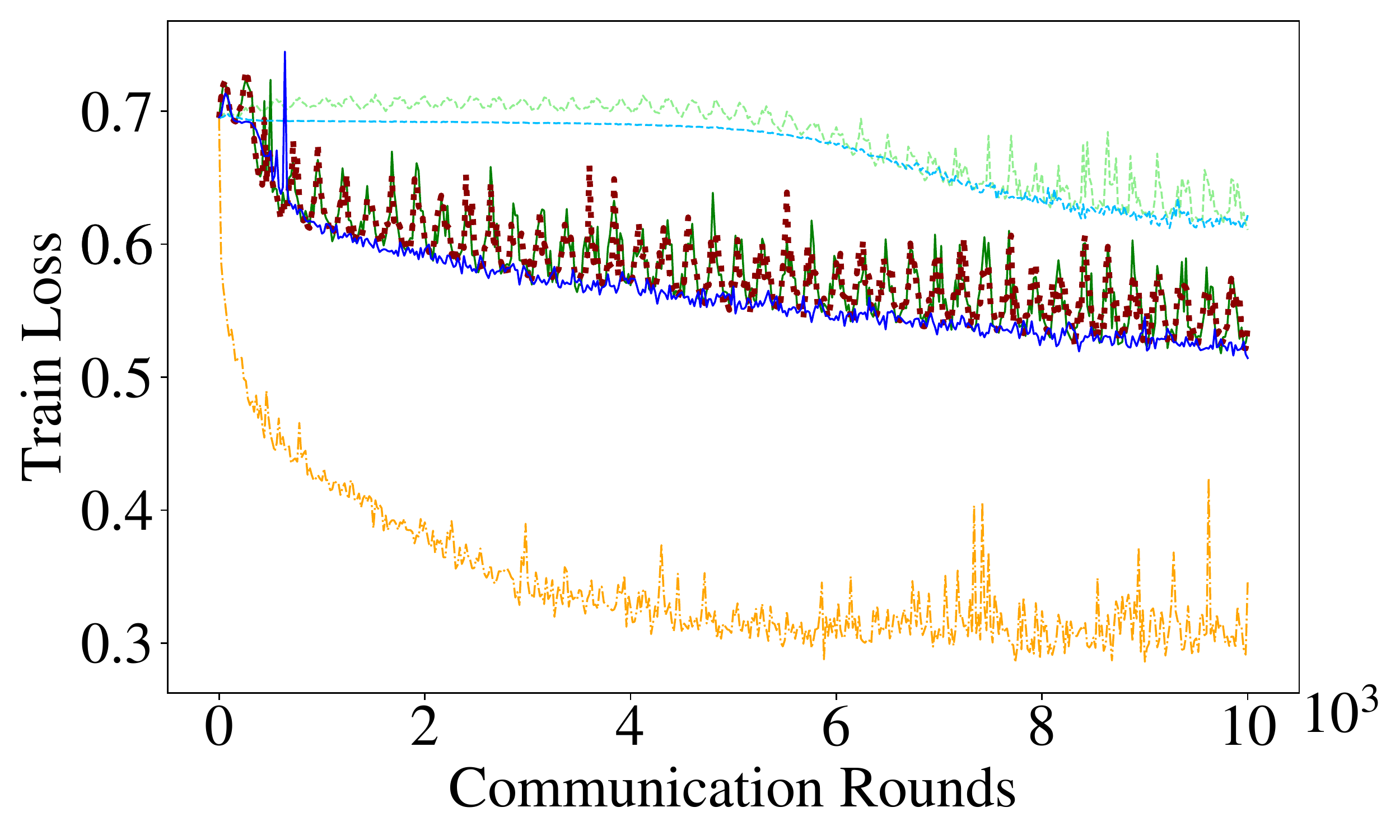}}
	\caption{Training losses of
	FedSGD, FedAvg, FedProx, FedLaAvg, and sequential SGD
	in the Sentiment140 sentiment analysis task 
	with different client availability settings.}
	\label{fig:sentiment_tl}
\end{figure*}

\begin{figure*}[tb]
	\centering
	\subfigure[Training loss with different $N$]{\label{fig:sentiment_N_tl}\includegraphics[width=0.4\textwidth]{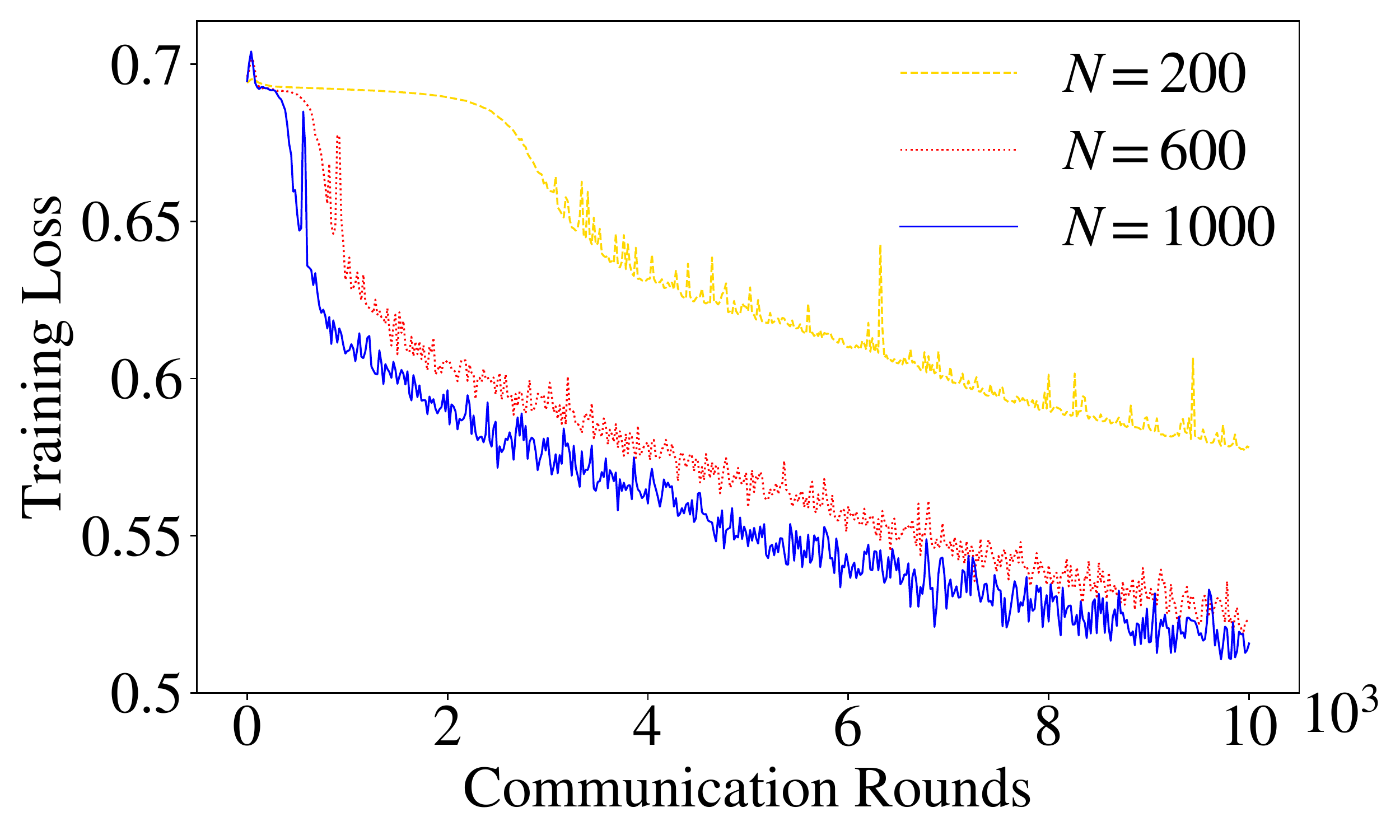}}
	\subfigure[Training loss with different $\beta$]{\label{fig:sentiment_beta_tl}\includegraphics[width=0.4\textwidth]{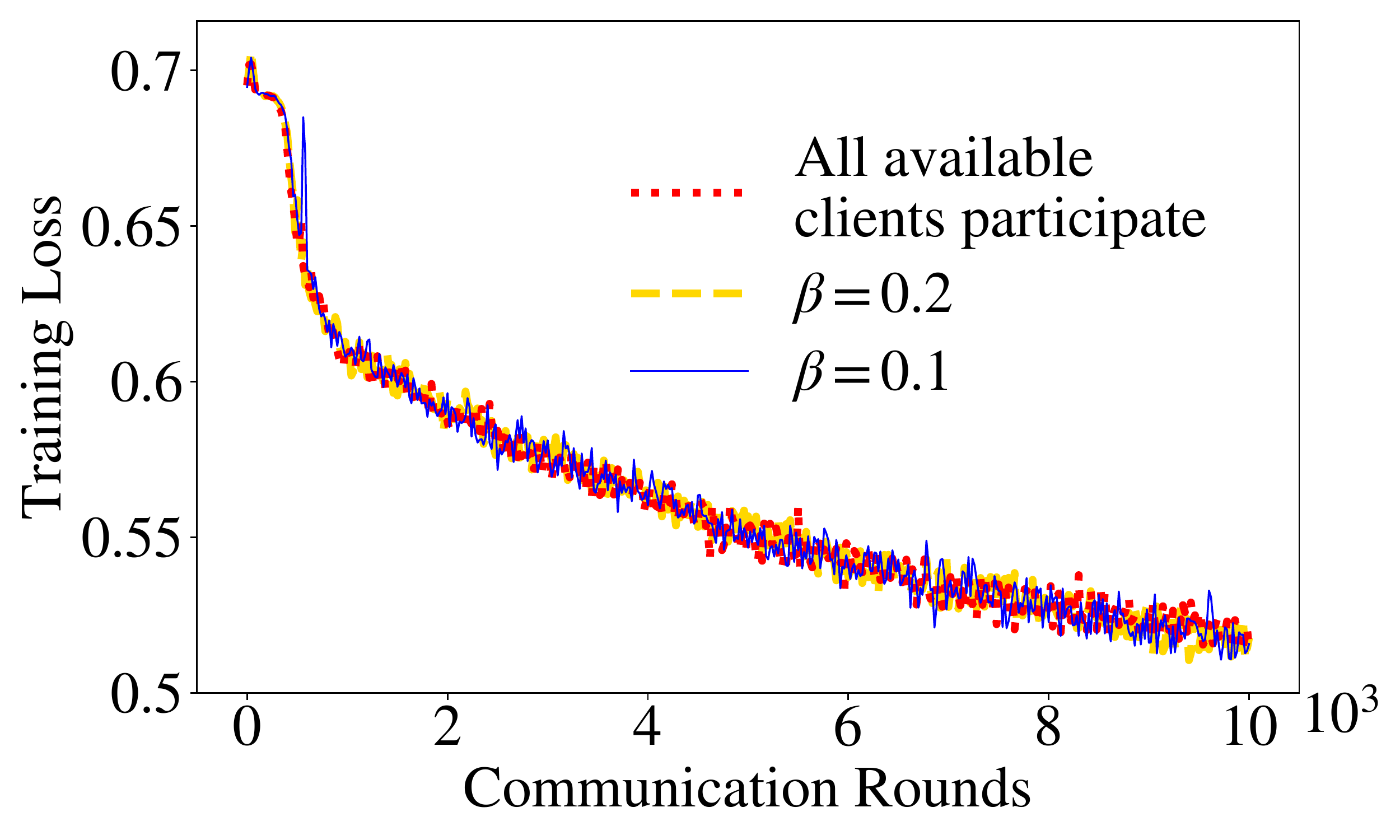}}
	\caption{Training losses of FedLaAvg on Sentiment140 dataset
	by varying the total number of clients $N$ and the proportion of
	selected clients $\beta$.}
	\label{fig:sentiment_self_test-acc}
\end{figure*}

In the sentiment analysis experiments 
over Sentiment140, we assume that 
clients are available only when the Twitter users are sleeping,
such that the devices are idle and eligible for the model training.
Hence, we set each client to be available for eight hours each day,
during which the client sends the least number of tweets,
as illustrated in Figure~\ref{fig:sentiment140_availability}.
To capture the correlation of client availability 
and local data distribution in practical FL,
we vary the label balance as a function of the time of day.
In particular, 
we randomly drop tweets sent in 0:00--1:00
such that $\alpha$ proportion of tweets are positive sentiment;
we randomly drop tweets sent in 12:00-13:00
such that $1-\alpha$ proportion of tweets are positive sentiment; 
and we linearly interpolate for other hours of day.

As the default experiment settings, we set the total number of clients $N=1000$, 
the period length $E=100$ for MNIST and $E=120$ for Sentiment140,
the availability heterogeneity controlling parameter $D=1$ for MNIST
and $\alpha=0.0$ for Sentiment140,
the proportion of selected clients in each round $\beta=0.1$, 
and the number of local iterations $C=10$.
For hyperparameters, we set the learning rate to $0.01$,
and the local batch size to $5$ for MNIST and $2$ for Sentiment140, respectively. 

\subsection{Results in the Convex Case}
We first compare FedLaAvg with the aforementioned baselines
in the MNIST image classification task
under various availability settings.
In particular, we vary $E$ from $50$ to $200$, which controls the 
maximum number of unavailable rounds,
and vary $D$ from 1 to 5,
which controls the differentiation of client availability.
We show the results in Figure~\ref{fig:mnist_tl}.
%FedLaAvg converges	
%   approaches SGD in MNIST
%   reach a plateau in MNIST and Sentiment140
In general, we observe that FedLaAvg converges 
in approximately 2000 rounds under all the five availability settings.
In addition, FedLaAvg approaches sequential SGD in terms of training loss.
These coincide with Theorem~\ref{the:main}.
% FedAvg diverges
% 	FedAvg
%	FedProx
In contrast, FedSGD, FedAvg, and FedProx suffer 
from periodic osillation in any availability setting,
especially when either $E$ or $D$ is large,
which validates Theorem~\ref{the:divergence}.
% FedProx does not work better than FedAvg
By comparing FedProx and FedAvg,
we can see that the training losses
of these two algorithms almost overlap,
which indicates that FedProx also cannot solve the
issue of intermittent client availablity.
% FedLaAvg can benefit from multiple local iterations
By comparing FedLaAvg with 
one and multiple local iterations, 
we find that introducing multiple local iterations 
speeds up the convergence with respect to communication rounds
in practice.

We then evaluate the effect of the total number of clients $N$ 
and the proportion of selected clients $\beta$ on FedLaAvg 
in Figure~\ref{fig:mnist_self_test-acc}.
From Figure~\ref{fig:mnist_N_tl},
we can see that FedLaAvg with different $N$ behaves almost uniformly.
This is because a larger $N$ improves the convergence by reducing 
the variance of the global model update,
while for such a simple convex training scenario,
the variance is already small enough with $N=200$.
From Figure~\ref{fig:mnist_beta_tl},
we observe that selecting more clients in each round
leads to more stable training,
but after reaching a certain threshold, 
selecting further more clients would not help much.

\subsection{Results in the Non-Convex Case}
We next evaluate FedLaAvg and the baselines
over the Sentiment140 dataset.
First, we compare different algorithms 
under various availability settings 
in Figure~\ref{fig:sentiment_tl}.
In particular, we vary $E$ from 24 to 240,
and vary $\alpha$ from 0 to 0.5, which controls the differentiation of 
client availability.
Similar to the convex case,
FedLaAvg converges in all the five availability settings,
while FedSGD, FedAvg, and FedProx suffer from severe oscillation,
especially when either $E$ is large or $\alpha$ is small; 
FedProx still cannot solve the issue of 
intermittent client availability.
We note that $\alpha=0.5$ indicates that
the ratio of positive tweets is fixed at $0.5$ 
throughout the training process.
There is nearly no correlation between client availability
and local data distribution in this setting, 
and thus the performance of FedAvg is similar to that of FedLaAvg.
In addition, we observe that 
introducing multiple local iterations to FedLaAvg
significantly speeds up the convergence,
which further validates the empirical 
communication efficiency improvement of local iterations.

We finally evaluate the impacts of the total number of clients $N$ 
and the proportion of selected clients $\beta$ on FedLaAvg,
and show the results in Figure~\ref{fig:sentiment_self_test-acc}.
From Figure~\ref{fig:sentiment_N_tl},
we observe that FedLaAvg with a larger $N$ converges much faster,
which validates the sublinear speedup with respect $N$.
From Figure~\ref{fig:sentiment_beta_tl},
we can see that $\beta$ has little effect 
on the convergence of FedLaAvg in this task.

For all the experiments above,
we also show the test accuracies with training rounds in 
\myappendix{sec:sup exp},
which behave consistently with the training losses.

\section{Conclusion}
\label{sec:conclusion}
In this work, we investigate intermittent client availability in FL
and its impact on the convergence of the classical FedAvg algorithm.
We use a collection of time-varying sets to represent the available clients in each training iteration, which can accurately model the intermittent client availability. 
Furthermore, we design a simple FedLaAvg algorithm with an $O(E^{1/2}/(N^{1/4} T^{1/2 }))$ convergence guarantee for general distributed non-convex optimization problems.
Empirical studies with the standard MNIST and Sentiment140 datasets 
demonstrate the effectiveness and efficiency of 
FedLaAvg with a remarkable performance improvement 
and a sublinear speedup. 

%%
%% The next two lines define the bibliography style to be used, and
%% the bibliography file.
\bibliographystyle{ACM-Reference-Format}
\bibliography{example_paper}

%%
%% If your work has an appendix, this is the place to put it.
\clearpage
\onecolumn
\appendix
\section{Divergence of the concurrent work under intermittent client availability}
\label{sec:comp flexible}
The concurrent work~\cite{flexible} considered a different availability setting, 
where the number of local iterations performed by client $i$ in round $r$ 
can take an arbitrary value $s_r^i$ from $\{0, 1,\cdots, C\}$, following some time-varying distribution.
For the clients submitting incompleted work, i.e., $1 \leq s_r^i < C$,
they proposed to scale the model update by $w_i^r = \frac{C}{s_i^r} w_i$,
to force equal contribution to the global model among clients.
Under this framework, they proved the convergence, which, however,
is tailed with a bias term $\frac{M_R}{R}$;
$R$ is the total number of communication rounds, 
and $M_R$ represents the accumulated training data bias throughout the training process.
Specifically, $M_R = \sum_{r=1}^{R} \E\left[z_r\right]$, 
where $z_r=0$ if all clients contribute equally to the global model update in round $r$,
i.e., $\frac{\E\left[w_i^r s_i^r\right]}{w_i}$ take the same value for all clients $i$,
and $z_r=1$ else.
When there exists a single client $i$ whose $s_r^i = 0$ with probability $p_i>0$,
i.e., client $i$ is unavailable in round $r$,
the proposed scaling technique cannot work and $z_r=p_i$.
Further, if there exists multiple clients whose $s_r^i = 0$ with probability $p_i > 0$,
$z_r$ is even larger.
The work~\cite{flexible} claims that their method converges if $M_R$ increases sublinearly with $R$.

The proposed method works mainly in the scenario that 
in most rounds of training, there are no unavailable clients.
However, under intermittent client availability considered in this work, there exist unavailable clients in almost every round. 
E.g., in our Example~\ref{exp:mean estimation}, two clients are available alternately,
and there is always one client unavailable throughout the training process,
i.e., $s_i^r=1$ with probability $1$,
indicating that $M_R=R$ and the method diverges.
Further, the work~\cite{flexible} proposed to kick out frequently unavailable clients 
if the evaluated training data bias introduced by keeping the clients, i.e., $\frac{M_R}{R}$,
is larger than that introduced by kicking the clients out.
This, however, still cannot solve the bias problem,
since the bias always exists no matter whether frequently unavailable clients are kicked out or not.

\section{Detailed Proof of the Convergence of FedLaAvg}

\subsection{Proof of Lemma~\ref{le:I}}\label{sec:I}
%===========================================================================
\begin{proof}[Proof of Lemma~\ref{le:I}]
	$\forall t, \forall i$, we focus on the training process from $t$ (not included).
	In iteration $t+I+1$, under Assumption~\ref{as:4}, client $i$ has been available for at least $\lceil N/K\rceil$ times. Note the $\lceil N/K\rceil$ iterations as $\tau_1, \tau_2, \cdots, \tau_{\lceil N/K\rceil}$. We prove the lemma by contradiction. Suppose $i$ is not selected in any of these iterations. 
	Then we have $T_i^{\tau_{\lceil N/K\rceil}}=\Tit$.
	In the $\lceil N/K\rceil$ iterations where client $i$ is available, $\lceil N/K\rceil K$ clients have been selected. 
	All these clients (noted as $j$) are with 
	$T_j^\tau\leq \Tit$
	for all iterations $\tau$ before it participates in the training process and 
	$T_j^\tau> \Tit$
	for all iterations $\tau$ after participation. 
	Hence, the $\lceil N/K\rceil K$ clients are distinct. Including client $i$, the system has at least $\lceil N/K\rceil K + 1$ clients. However, the system has only $N\leq \lceil N/K\rceil K < \lceil N/K\rceil K  + 1$ clients. This forms a contradiction. Therefore, for all $t$, the next iteration $t_{next}$ where $i$ participates in the training process after iteration $t$ satisfies 
	\begin{equation}\label{eq:tnext}
	t_{next} \leq t + I + 1.
	\end{equation}
	For all client $i$, by setting $t$ to iterations where client $i$ is selected in (\ref{eq:tnext}), we can derive 
	\[\forall i,\:\forall t,\:t - \Tit \leq I.\] 
\end{proof}

\subsection{Proof of Theorem~\ref{the:main}}\label{sec:sup main}
Note that local gradient is not calculated in each iteration. In this subsection of the appendix, for mathematical analysis, we extend the definition $\git \triangleq \nabla F(\xts, \xi_i^{t})$. For iterations where client $i$ does not participate, $\xi_i^{t}$ is a random variable which follows $\xi_i^{t} \sim \D_i$.

\begin{lemma}\label{le:2}
	Under Assumption~\ref{as:2} and \ref{as:3}, we have 
	\[
	\E \left[ \left\lVert \git\right\rVert ^2 \right] \leq G^2,\forall i,\forall t
	\]
	and
	\[
	\E \left[ \left\lVert \git - \nabla \fixts \right\rVert ^2 \right] \leq \sigma^2,\forall i,\forall t.
	\]
\end{lemma}

\begin{proof}[Proof of Lemma~\ref{le:2}]
	Our Assumptions~\ref{as:2} and \ref{as:3} take the expectation over the randomness of one training iteration. 
	But we care about the expectation taken over the randomness of the whole training process. This trivial lemma builds the gap. 
	
	For the gradient, we have
	\begin{eqnarray}
	\E \left[ \left\lVert \git \right\rVert ^2 \right] 
	\overset{(a)}{=}
	\E\left[ \E \left[ \left\lVert \git\right\rVert ^2 \mid \xi^{[t-1]} \right] \right] 
	=
	\E \left[  \E \left[  \left\lVert  \nabla F \left( \xts,\xi_i^{  t  } \right)  \right\rVert ^2 \mid \xts  \right]   \right] 
	\overset{(b)}{\leq}
	\E \left[ G^2 \right] 
	=
	G^2,
	\end{eqnarray}
	where (a) follows from \LTE; (b) follows from Assumption~\ref{as:3}. 
	
	For the variance, we have 
	%	\begin{eqnarray}
	%	&&
	%	\E \left[ \left\lVert \git - \nabla \fixts \right\rVert ^2 \right] 
	%	\new
	%	&\overset{(a)}{=}&
	%	\E \left[ \E\left[ \left\lVert \git - \nabla \fixts \right\rVert ^2 \mid \xi^{[t-1]} \right] \right] 
	%	\new
	%	&=&
	%	\E \left[ 
	%	\E \left[
	%	\left\lVert 
	%	\nabla F\left(\xts;\xi_i^{t}\right) 
	%	-
	%	\nabla f_i\left( \xts \right) 
	%	\right\rVert^2 \mid \xts 
	%	\right]
	%	\right] 
	%	\new
	%	&\overset{(b)}{\leq}&
	%	\E \left[\sigma^2\right]
	%	\new
	%	&=&
	%	\sigma^2,
	%	\end{eqnarray}
	\begin{eqnarray}
	&&
	\E \left[ \left\lVert \git - \nabla \fixts \right\rVert ^2 \right] 
	\overset{(a)}{=}
	\E \left[ \E\left[ \left\lVert \git - \nabla \fixts \right\rVert ^2 \mid \xi^{[t-1]} \right] \right] 
	\new
	&=&
	\E \left[ 
	\E \left[
	\left\lVert 
	\nabla F\left(\xts;\xi_i^{t}\right) 
	-
	\nabla f_i\left( \xts \right) 
	\right\rVert^2 \mid \xts 
	\right]
	\right] 
	\overset{(b)}{\leq}
	\E \left[\sigma^2\right]
	=
	\sigma^2,
	\end{eqnarray}
	where (a) follows from \LTE; (b) follows from Assumption~\ref{as:2}. 
\end{proof}

%============================================================================
\begin{lemma}\label{le:eg-ef}
	$\forall i, \forall t$, we have
	\begin{eqnarray}
	\E \left[ \left\lVert\sum_{i=1}^{N} \left( \giTit-\nabla \fixTitS \right) \right\rVert^2 \right] 
	=
	\sum_{i=1}^{N}\E \left[ \left\lVert\giTit-\nabla \fixTitS\right\rVert^2 \right] \nonumber.
	\end{eqnarray}
\end{lemma}

\begin{proof}[Proof of Lemma~\ref{le:eg-ef}.]
	This lemma follows because training data are independent across clients.
	Specifically, note that
	\begin{eqnarray}
	&&
	\E \left[ \left\lVert\sum_{i=1}^{N} \left( \giTit-\nabla \fixTitS \right) \right\rVert^2 \right] 
	\new
	&=&
	\sum_{p=1}^{N}\sum_{q=1}^{N}\E \left[ 
	\left\langle  
	\gpTpt-\nabla \fpxTptS,
	\gqTqt-\nabla \fqxTqtS
	\right\rangle  
	\right]
	\new
	&\overset{(a)}{=}&
	\sum_{p=1}^{N}\sum_{q=1}^{N}\E \left[ 
	\E \left[ 
	\left\langle  
	\gpTpt-\nabla \fpxTptS,
	\gqTqt-\nabla \fqxTqtS
	\right\rangle  
	\mid\xi^{ \left[ \min\{\Tpt,\Tqt\} \right] }
	\right]  
	\right]
	\new
	&\overset{(b)}{=}&
	\sum_{i=1}^{N}\E \left[ \left\lVert\giTit-\nabla \fixTitS\right\rVert^2 \right],
	\end{eqnarray}
	where (a) follows from \LTE. Then we illustrate (b) case by case. Note that
	\[
	\E \left[ 
	\E \left[ 
	\left\langle  
	\gpTpt-\nabla \fpxTptS,
	\gqTqt-\nabla \fqxTqtS
	\right\rangle  
	\mid\xi^{ \left[ \min\{\Tpt,\Tqt\} \right] }
	\right]  
	\right]
	\]
	is equal to $\E\left[\left\lVert \giTit -\nabla \fixTitS\right\rVert^2\right]$ when $p=q=i$. When $p\neq q$, without loss of generality, suppose $\Tpt\leq \Tqt$. Then it is equal to
	\begin{eqnarray}
	&&
	\E \left[ \E \left[ \left\langle  \gpTpt-\nabla \fpxTptS,
	\gqTqt-\nabla \fqxTqtS
	\right\rangle \mid\xi^{ \left[ \Tpt \right] } \right]  \right] 
	\new
	&=&
	\E \left[ \left\langle  \gpTpt-\nabla \fpxTptS,
	\E \left[ \gqTqt-\nabla \fqxTqtS\mid\xi^{ \left[ \Tpt \right] } \right]  \right\rangle  \right] 
	\end{eqnarray}
	because $\gpTpt$ and $\fqxTptS$ are determined by $\xi^{ \left[ \Tpt \right]}$.
	When $\Tpt< \Tqt$, we have $\E[\gqTqt-\nabla \fqxTqtS \mid\xi^{[\Tpt]}
	]=0$. When $\Tpt= \Tqt$, we have 
	\begin{eqnarray}
	&&\E  \left[  \left\langle  \gpTpt-\nabla \fpxTptS,
	\E \left[ \gqTqt-\nabla \fqxTqtS\mid\xi^{ \left[ \Tpt \right] } \right]  \right\rangle   \right] \nonumber
	\\
	&=&
	\E \left[ \left \langle  \gpTpt-\nabla \fpxTptS,
	\gqTpt -\nabla \fqxTptS \right \rangle  \right] 
	\new 
	&\overset{(a)}{=}&
	\E \left[ \E  \left[  \left\langle  \gpTpt-\nabla \fpxTptS,
	\gqTpt-\nabla \fqxTptS \right\rangle \mid\xi^{ \left[ \Tpt-1 \right] } \right]     \right] 
	\new
	&\overset{(b)}{=}&
	0,
	\end{eqnarray}
	where (a) follows from \LTE;
	(b) follows because $\xi_{p}^{\Tpt}$ and $\xi_q^{\Tpt}$ are independent, and thus the covariance of $\gpTpt$ and $\gqTpt$ is 0.
\end{proof}

%============================================================================
\begin{lemma}\label{le:ft-ftit}
	Under Assumptions~\ref{as:1} and \ref{as:3}, $\forall t,\forall t_0\leq t, \forall i$, we have
	\[
	\E \left[  \left\lVert \nabla \fixts-\nabla f_i ( \x^{  t_0 -1 } ) \right\rVert ^2  \right] 
	\leq  \left( t-t_0 \right) ^2 L^2\gamma^2 G^2.
	\]
\end{lemma}

\begin{proof}[Proof of Lemma~\ref{le:ft-ftit}.] 
	This lemma follows the intuition that the difference of $\x$ in two iterations is bounded by the number of iterations between them.
	\begin{eqnarray}\label{eq:f-t-Tit}
	&&
	\E \left[  \left\lVert \nabla \fixts-\nabla f_i ( \x^{  t_0 -1  } ) \right\rVert ^2  \right]  
	\new
	&=&
	\E \left[  \left\lVert \sum_{\tau=t_0}^{t-1} \left( \nabla f_i ( \x^{  \tau  } ) -\nabla f_i ( \x^{  \tau -1  } )  \right) \right\rVert ^2  \right] 
	\new
	&\overset{(a)}{\leq}&
	\left( t-t_0 \right)  \sum_{\tau=t_0}^{t-1}
	\E \left[  \left\lVert \nabla f_i ( \x^{  \tau  } ) -\nabla f_i ( \x^{  \tau -1  } ) \right\rVert ^2  \right]  
	\new
	&\overset{(b)}{\leq}&
	\left( t-t_0 \right) L^2\sum_{\tau=t_0}^{t-1}
	\E \left[  \left\lVert \x^{  \tau  }-\x^{  \tau -1  }\right\rVert ^2  \right]  
	\new
	&\overset{(c)}{=}&
	\left( t-t_0 \right) L^2\gamma^2\sum_{\tau=t_0}^{t-1}
	\E \left[  \left\lVert \frac{1}{N}\sum_{j=1}^{N}\gjTjtau\right\rVert ^2  \right] 
	\new
	&\overset{(d)}{\leq}&
	\left( t-t_0 \right) L^2\gamma^2\frac{1}{N}\sum_{\tau=t_0}^{t-1}
	\sum_{j=1}^{N}
	\E \left[  \left\lVert \gjTjtau \right\rVert ^2  \right] 
	\new
	&\overset{(e)}{\leq}&
	\left( t-t_0 \right) ^2 L^2\gamma^2 G^2,
	\end{eqnarray}
	where (a) and (d) follows from the convexity of $\left\lVert \cdot\right\rVert ^2$; (b) follows from Assumption~\ref{as:1}; (c) follows from (\ref{eq:update}) and (\ref{eq:gt}); (e) follows from Lemma~\ref{le:2}.
\end{proof}

%============================================================================
\begin{corollary}\label{co:ft-ftit-1}
	Corollary of Lemma~\ref{le:ft-ftit}:
	\[
	\E \left[ \left\lVert \nabla \fxts-\nabla f ( \x^{  t_0 -1  } ) \right\rVert ^2   \right] 
	\leq 
	(t-t_0)^2 L^2\gamma^2 G^2.
	\]
\end{corollary}

\begin{proof}[Proof of Corollary~\ref{co:ft-ftit-1}. ]
	\begin{eqnarray}
	&&
	\E \left[ \left\lVert \nabla \fxts-\nabla f ( \x^{  t_0 -1  } ) \right\rVert ^2  \right] 
	\nonumber\\
	&\overset{(a)}{\leq}& 
	\frac{1}{N} \sum_{i=1}^{N}\E \left[ \left\lVert \nabla f_i \left( \xts \right)  -\nabla f_i ( \x^{  t_0 -1  } ) \right\rVert ^2  \right] 
	\nonumber\\
	&\overset{(b)}{\leq}& (t-t_0)^2 L^2\gamma^2 G^2,
	\end{eqnarray}
	where (a) follows from \convex; (b) follows from Lemma~\ref{le:ft-ftit}.
\end{proof}

\begin{proof}[Main Proof of Theorem~\ref{the:main}.]
	From Assumption~\ref{as:1}, local objective functions $f_i$ are all $L-smooth$, and thus the global objective function $f$, which is the mean of them, is also $L-smooth$.
	Hence, fixing $t\geq 1$, we have
	\begin{eqnarray}\label{eq:main-0}
	\E \left[ f ( \x^{  t  } )  \right] 
	\leq
	\E \left[ \fxts \right] 
	+\frac{L}{2}\E \left[ \left\lVert \x^{  t  }-\x^{  t-1  }\right\rVert ^2 \right] 
	+\E \left[ \left\langle\nabla \fxts, \x^{  t  }-\x^{  t-1  }\right\rangle \right].
	\end{eqnarray}
	
	We decompose the terms on the right, during which we refer to Lemma~\ref{le:ft-ftit} and Corollary~\ref{co:ft-ftit-1}.
	Specifically, we first focus on the second term:
	\begin{eqnarray}\label{eq:main-1}
	&&\E \left[ \left\lVert \x^{  t  }-\x^{  t-1  }\right\rVert ^2 \right]
	\new
	&\overset{(a)}{=}&
	\gamma^2\E \left[ \left\lVert \frac{1}{N}\sum_{i=1}^{N}\giTit\right\rVert ^2 \right] 
	\new
	&=&
	\gamma^2\E \left[ \left\lVert \frac{1}{N}\sum_{i=1}^{N}\left(\giTit -  \nabla \fixTitS\right) + \frac{1}{N}\sum_{i=1}^{N}  \nabla \fixTitS \right\rVert ^2 \right] 
	\new
	&\overset{(b)}{\leq}&
	2\gamma^2\E \left[ \left\lVert \frac{1}{N}\sum_{i=1}^{N}\left(\giTit -  \nabla \fixTitS\right) \right\rVert ^2 \right] + 2 \gamma^2 \E\left[\left\lVert \frac{1}{N}\sum_{i=1}^{N} \nabla  \fixTitS \right\rVert ^2 \right] 
	\new
	&\overset{(c)}{=}&
	\frac{2\gamma^2}{N^2} \sum_{i=1}^{N}\E \left[ \left\lVert \giTit -  \nabla \fixTitS \right\rVert ^2 \right] 
	+ 2 \gamma^2 \E\left[\left\lVert \frac{1}{N}\sum_{i=1}^{N}  \nabla \fixTitS \right\rVert ^2 \right] 
	\new
	&\overset{(d)}{=}&
	\frac{2\gamma^2 \sigma^2}{N} 
	+ 2 \gamma^2 \E\left[\left\lVert \frac{1}{N}\sum_{i=1}^{N}  \nabla \fixTitS \right\rVert ^2 \right],
	\end{eqnarray}
	where (a) follows from (\ref{eq:update}) and (\ref{eq:gt}); (b) follows from \convex; (c) follows from Lemma~\ref{le:eg-ef}; (d) follows from Lemma~\ref{le:2}.
	
	Define $\Tt \triangleq min_{i} \left(\Tit\right)$.
	Focus on the third term in (\ref{eq:main-0}),
	\begin{eqnarray}\label{eq:main-4}
	&&
	\E \left[ \left\langle\nabla \fxts, \x^{  t  }-\x^{  t-1  }\right\rangle \right]
	\new
	&\overset{(a)}{=}&
	-\gamma \E \left[ \left\langle\nabla \fxts,  
	\frac{1}{N}\sum_{i=1}^{N}\giTit
	\right\rangle \right]
	\new
	&=&
	-\gamma \E \left[  \left\langle \nabla \fxts -  \nabla \fxTtS , \frac{1}{N}\sum_{i=1}^{N}\giTit \right\rangle \right] - \gamma \E \left[  \left\langle \nabla \fxTtS , \frac{1}{N}\sum_{i=1}^{N}\giTit \right\rangle \right]
	\new
	&=&
	-\gamma \E \left[  \left\langle \nabla \fxts -  \nabla \fxTtS , \frac{1}{N}\sum_{i=1}^{N}\left(\giTit -  \nabla \fixTitS\right)\right\rangle \right] 
	\newl
	-\gamma \E \left[  \left\langle \nabla \fxts - \nabla  \fxTtS , \frac{1}{N}\sum_{i=1}^{N}  \nabla \fixTitS \right\rangle \right] 
	\newl
	- \gamma \E \left[  \left\langle \nabla \fxTtS , \frac{1}{N}\sum_{i=1}^{N}\giTit \right\rangle \right],
	\end{eqnarray}
	where (a) follows from (\ref{eq:gt}) and (\ref{eq:update}).	We further focus on the first term in (\ref{eq:main-4}):
	\begin{eqnarray}\label{eq:main3}
	&&
	-\gamma \E \left[  \left\langle \nabla \fxts - \nabla  \fxTtS , \frac{1}{N}\sum_{i=1}^{N}\left(\giTit -  \nabla \fixTitS\right)\right\rangle \right] 
	\new
	&=&
	-\frac{\gamma^2 I L}{\sqrt{N}} \E \left[  \left\langle \frac{1}{\gamma I L}\left( \nabla \fxts - \nabla  \fxTtS \right), \frac{\sqrt{N}}{N}\sum_{i=1}^{N}\left(\giTit -  \nabla \fixTitS\right)\right\rangle \right] 
	\new
	&\overset{(a)}{\leq}&
	\frac{\gamma^2 I L}{2\sqrt{N}} \E \left[  \left\lVert \frac{1}{\gamma I L}\left( \nabla \fxts - \nabla  \fxTtS \right) \right\rVert^2 \right]
	\newl
	+ 
	\frac{\gamma^2 I L}{2\sqrt{N}} \E\left[ \left\lVert \frac{\sqrt{N}}{N}\sum_{i=1}^{N}\left(\giTit -  \nabla \fixTitS\right)\right\rVert^2 \right] 
	\new
	&\overset{(b)}{\leq}&
	\frac{\gamma^2 I L G^2}{2\sqrt{N}} 
	+ 
	\frac{\gamma^2 I L}{2\sqrt{N}} \E\left[ \left\lVert \frac{\sqrt{N}}{N}\sum_{i=1}^{N}\left(\giTit - \nabla  \fixTitS\right)\right\rVert^2 \right] 
	\new
	&\overset{(c)}{=}&
	\frac{\gamma^2 I L G^2}{2\sqrt{N}} 
	+ 
	\frac{\gamma^2 I L}{2N^{\frac{3}{2}}} \sum_{i=1}^{N} \E\left[ \left\lVert \left(\giTit -  \nabla \fixTitS\right)\right\rVert^2 \right] 
	\new
	&\overset{(d)}{\leq}&
	\frac{\gamma^2 I L G^2}{2\sqrt{N}} 
	+ 
	\frac{\gamma^2 I L \sigma^2}{2\sqrt{N}},
	\end{eqnarray}
	where (a) follows from \neqCSAMGM; (b) follows from Corollary~\ref{co:ft-ftit-1} with $t_0$ assigned as $\Tt$ and Lemma~\ref{le:I}; (c) follows from Lemma~\ref{le:eg-ef}; (d) follows from Lemma~\ref{le:2}. 
	Then we focus on the second term in \ref{eq:main-4} (Note that $\gamma < 1/(2L)$ and thus we can extract the root of $1 - 2\gamma L$):
	\begin{eqnarray}\label{eq:main4}
	&&
	-\gamma \E \left[  \left\langle \nabla \fxts -  \nabla \fxTtS , \frac{1}{N}\sum_{i=1}^{N}  \nabla \fixTitS \right\rangle \right]  
	\new
	&=&
	-\gamma \E \left[  \left\langle \frac{1}{\sqrt{1 - 2\gamma L}}  \left(\nabla \fxts -  \nabla \fxTtS \right), \frac{1}{N}\sqrt{1 - 2\gamma L}\sum_{i=1}^{N}  \nabla \fixTitS \right\rangle \right]  
	\new
	&\overset{(a)}{\leq}&
	\frac{\gamma}{2\left(1 - 2\gamma L\right)} \E \left[   \left\lVert \nabla \fxts -  \nabla \fxTtS \right\rVert^2 \right]
	\newl
	+
	\frac{\gamma \left(1 - 2\gamma L\right)}{2} \E\left[ \left\lVert \frac{1}{N}\sum_{i=1}^{N}  \nabla \fixTitS \right\rVert^2 \right]  
	\new
	&\overset{(b)}{\leq}&
	\frac{\gamma^3 I^2 L^2 G^2}{2\left(1 - 2\gamma L\right)} 
	+
	\frac{\gamma \left(1 - 2\gamma L\right)}{2} \E\left[ \left\lVert \frac{1}{N}\sum_{i=1}^{N} \nabla \fixTitS \right\rVert^2 \right],
	\end{eqnarray}
	where (a) follows from \neqCSAMGM; (b) follows from Corollary~\ref{co:ft-ftit-1} and Lemma~\ref{le:I}. We finally focus on the third term in (\ref{eq:main-4}):
	\begin{eqnarray}\label{eq:main5}
	&&
	\E \left[  \left\langle \nabla \fxTtS , \frac{1}{N}\sum_{i=1}^{N}\giTit \right\rangle \right]
	\new
	&\overset{(a)}{=}&
	\E \left[  \E \left[ \left\langle \nabla \fxTtS , \frac{1}{N}\sum_{i=1}^{N}\giTit \right\rangle \mid \XiTtS \right] \right]
	\new
	&=&
	\E \left[  \frac{1}{N}\sum_{i=1}^{N} \E \left[ \left\langle \nabla \fxTtS , \giTit \right\rangle \mid \XiTtS \right] \right]
	\new
	&\overset{(b)}{=}&
	\E \left[  \frac{1}{N}\sum_{i=1}^{N} \E \left[ \E \left[ \left\langle \nabla \fxTtS , \giTit \right\rangle \mid \XiTitS \right] \mid \XiTtS \right] \right]
	\new
	&\overset{(c)}{=}&
	\frac{1}{N}\sum_{i=1}^{N} \E \left[  \E \left[ \left\langle \nabla \fxTtS , \nabla \fixTitS \right\rangle \mid \XiTtS \right] \right]
	\new
	&\overset{(d)}{=}&
	\frac{1}{N}\sum_{i=1}^{N} \E \left[ \left\langle \nabla \fxTtS , \nabla \fixTitS \right\rangle \right]
	\new
	&=&
	\E \left[ \left\langle \nabla \fxTtS , \frac{1}{N}\sum_{i=1}^{N} \nabla \fixTitS \right\rangle \right]
	\new
	&\overset{(e)}{=}&
	\frac{1}{2}\E \left[ \left\lVert \nabla \fxTtS \right\rVert^2 \right] 
	+ \frac{1}{2}\E\left[ \left\lVert \frac{1}{N}\sum_{i=1}^{N} \nabla \fixTitS \right\rVert^2 \right]
	\newl
	-  \frac{1}{2}\E\left[ \left\lVert \nabla \fxTtS - \frac{1}{N}\sum_{i=1}^{N} \nabla \fixTitS \right\rVert^2 \right],
	\end{eqnarray}
	where (a), (b) and (d) follows from \LTE; (c) follows because $\forall i: \Tt \leq \Tit$, and thus $\fxTtS$ is determined by $\XiTitS$; (e) follows from \equuvv.
	In (\ref{eq:main5}), we further deal with the last term,
	\begin{eqnarray}\label{eq:main6}
	&&
	\E\left[ \left\lVert \nabla \fxTtS - \frac{1}{N}\sum_{i=1}^{N} \nabla \fixTitS \right\rVert^2 \right]
	\new
	&=&
	\E\left[ \left\lVert \frac{1}{N}\sum_{i=1}^{N} \left(\nabla \fixTtS -  \nabla \fixTitS \right) \right\rVert^2 \right]
	\new
	&\overset{(a)}{\leq}&
	\frac{1}{N}\sum_{i=1}^{N} \E\left[ \left\lVert  \left(\nabla \fixTitS -  \nabla \fixTtS \right) \right\rVert^2 \right]
	\new
	&\overset{(b)}{\leq}&
	\frac{1}{N}\sum_{i=1}^{N} (\Tit - \Tt)^2 L^2 \gamma^2 G^2
	\new
	&\overset{(c)}{\leq}&
	I^2 L^2 \gamma^2 G^2,
	\end{eqnarray}
	where (a) follows from \convex; (b) follows from Lemma~\ref{le:ft-ftit} with $t$ assigned as $\Tit$ and $t_0$ assigned as $\Tt$; (c) follows from Lemma~\ref{le:I}, $\Tit \leq t$, and $\Tt = \min_{i} \left(\Tit\right)$.
	Substituting (\ref{eq:main6}) into (\ref{eq:main5}) and (\ref{eq:main3})--(\ref{eq:main5}) into (\ref{eq:main-4}), we have:
	\begin{eqnarray}\label{eq:main-9}
	&&
	\E \left[ \left\langle\nabla \fxts, \x^{  t  }-\x^{  t-1  }\right\rangle \right]
	\new
	&\leq&
	\frac{\gamma^2 I L G^2}{2\sqrt{N}} 
	+ 
	\frac{\gamma^2 I L \sigma^2}{2\sqrt{N}}   
	+
	\frac{\gamma^3 I^2 L^2 G^2}{2\left(1 - 2\gamma L\right)} 
	+  
	\frac{\gamma^3  I^2 L^2 G^2}{2}
		\newl
	-
	\gamma^2L \E\left[ \left\lVert \frac{1}{N}\sum_{i=1}^{N} \nabla  \fixTitS \right\rVert^2 \right]  
	- 
	\frac{\gamma}{2}\E \left[ \left\lVert \nabla \fxTtS \right\rVert^2 \right].
	\newl
	\end{eqnarray}
	Further substituting (\ref{eq:main-1}) and (\ref{eq:main-9}) into (\ref{eq:main-0}), we have
	\begin{eqnarray}\label{eq:main-fin}
		&&
	\E \left[ f ( \x^{  t  } )  \right] -  \E \left[ \fxts \right]
		\new 
		&\leq&
	\frac{\gamma^2 \sigma^2 L}{N} 
	+ 
	\frac{\gamma^2 I L \left(G^2+\sigma^2\right)}{2\sqrt{N}}  
	+
	\frac{\gamma^3 I^2 L^2 G^2}{2\left(1 - 2\gamma L\right)} 
	+  
	\frac{\gamma^3  I^2 L^2 G^2}{2} 
	- 
	\frac{\gamma}{2}\E \left[ \left\lVert \nabla \fxTtS \right\rVert^2 \right].
	\end{eqnarray}
	
	We rearrange (\ref{eq:main-fin}) with summation to obtain the convergence result.
	First, we rearrange (\ref{eq:main-fin}):
	\begin{eqnarray}\label{eq:main-fin-1}
	\E \left[ \left\lVert \nabla \fxTtS \right\rVert^2 \right]
	\leq&&
	\frac{2 \gamma \sigma^2 L}{N} 
	+ 
	\frac{\gamma I L \left(G^2+\sigma^2\right)}{\sqrt{N}} 
	+
	\frac{\gamma^2 I^2 L^2 G^2}{1 - 2\gamma L} 
	+  
	\gamma^2  I^2 L^2 G^2
	\newl
	+  
	\frac{2}{\gamma} \left(
	\E \left[ \fxts \right]
	- 
	\E \left[ f ( \x^{  t  } )  \right]
	\right).
	\end{eqnarray}
	Summing (\ref{eq:main-fin-1}) over iterations from 1 to $T$ and dividing both sides by $T$, we have
	\begin{eqnarray}\label{eq:main-fin-2}
	\frac{1}{T}\sum_{t=1}^{T}\E \left[ \left\lVert \nabla \fxTtS \right\rVert^2 \right]
	\leq
	&&
	\frac{2 \gamma \sigma^2 L}{N} 
	+ 
	\frac{\gamma I L \left(G^2 + \sigma^2\right)}{\sqrt{N}} 
	+
	\frac{\gamma^2 I^2 L^2 G^2}{1 - 2\gamma L} 
	+  
	\gamma^2  I^2 L^2 G^2
	\newl
	+  
	\frac{2}{\gamma T} \left(
	\E \left[ f(\x^{0}) \right]
	- 
	\E \left[ f(\x^*)  \right]
	\right),
	\end{eqnarray}
	where $\x^*$ is the optimal solution for the global objective function $f(\x)$.
	
	Finally, we build the gap between $\nabla \fxts$ and $\nabla \fxTtS$.
	Lemma~\ref{le:I} implies that $t - \Tt \leq I$ since $\Tt = \min_{i} \Tit$. Hence, we have
	\begin{eqnarray}\label{eq:main-fin-3}
	&&
	\E\left[\left\lVert \nabla  \fxts \right\rVert^2\right]
	\new
	&=&
	\E\left[\left\lVert  \nabla \fxts -  \nabla \fxTtS +  \nabla \fxTtS \right\rVert^2\right]
	\new
	&\overset{(a)}{\leq}&
	2 \E\left[\left\lVert  \nabla \fxts - \nabla  \fxTtS \right\rVert^2 \right] 
	+
	2 \E \left[ \left\lVert \nabla  \fxTtS \right\rVert^2\right]
	\new
	&\overset{(b)}{\leq}&
	2 \gamma^2 I^2 L^2 G^2 
	+
	2 \E \left[ \left\lVert \nabla  \fxTtS \right\rVert^2\right],
	\end{eqnarray}
	where (a) follows from \convex; (b) follows from Corollary~\ref{co:ft-ftit-1}.
	Sum (\ref{eq:main-fin-3}) over iterations from 1 to $T$, devide both sides by $T$, and substitute (\ref{eq:main-fin-2}) into it. We then have
	\begin{eqnarray}\label{eq:main-the}
	\frac{1}{T}\sum_{t=1}^{T}\E \left[ \left\lVert \nabla \fxts \right\rVert^2 \right] 
	\leq
	&&
	\frac{4 \gamma \sigma^2 L}{N} 
	+ 
	\frac{2 \gamma I L \left(G^2 + \sigma^2\right)}{\sqrt{N}} 
	+
	\frac{2 \gamma^2 I^2 L^2 G^2}{1 - 2\gamma L} 
	+  
	4 \gamma^2  I^2 L^2 G^2
	\newl
	+  
	\frac{4}{\gamma T} \left(
	\E \left[ f(\x^{0}) \right]
	- 
	\E \left[ f(\x^*)  \right]
	\right).
	\end{eqnarray}
\end{proof}

\subsection{Proof of Corollary~\ref{co:main}}\label{sec:supco}
\begin{proof}[Proof of Corollary~\ref{co:main}]
	We first summarize the $O(\cdot)$ form of Theorem~\ref{the:main}:
	\begin{equation}\label{eq:comain-O}
	\frac{1}{T}\sum_{t=1}^{T}\E \left[ \left\lVert \nabla \fxts \right\rVert^2 \right]
	=
	O\left(
	\frac{\gamma I L \left(G^2 + \sigma^2\right)}{\sqrt{N}} 
	+
	\frac{\gamma^2 I^2 L^2 G^2}{1 - 2\gamma L} 
	+  
	\frac{B}{\gamma T}
	\right).
	\end{equation}
	
	Substituting $\gamma$ with $( \beta^{1/2}N^{1/4} )/( 2L E^{1/2} T^{1/2} ) $, we have
	\begin{eqnarray}
	\frac{1}{T} \sum_{t=1}^{T} \E  \left[  \left\lVert  \nabla \fxts \right\rVert ^2  \right] 
	&\overset{(a)}{=}&
	O\left(
	\frac{\gamma I L \left(G^2+\sigma^2\right)}{\sqrt{N}}   
	+  
	I^2 \gamma^2 L^2 G^2  
	+
	\frac{B}{\gamma T}
	\right)
	\new 
	&=&
	O\left(
	\frac{ I \beta^{\frac{1}{2}} \left(G^2+\sigma^2\right)}{N^{\frac{1}{4}}E^{\frac{1}{2}} T^{\frac{1}{2}} }   
	+  
	\frac{ I^2  \beta G^2  N^{\frac{1}{2}} }{  ET }   
	+
	\frac{B L  E^{\frac{1}{2}}}
	{\beta^{\frac{1}{2}} N^{\frac{1}{4}} T^{\frac{1}{2}} }
	\right)
	\new 
	&\overset{(b)}{=}&
	O\left(
	\frac{ E^{\frac{1}{2}}  \left(G^2+\sigma^2 + B L\right)}
	{ \beta^{\frac{1}{2}} N^{\frac{1}{4}} T^{\frac{1}{2}} }   
	+  
	\frac{ E  G^2  N^{\frac{1}{2}} }{ \beta  T }   
	\right),
	\end{eqnarray}
	where (a) follows because $\gamma \leq 1/(4L)$, and thus $1 - 2\gamma L > 1/2$; (b) follows because from Lemma~\ref{le:I}, $I=\lceil N/K\rceil E = O(E/\beta)$.
	
	When $T \geq  E N^{3/2}/\beta $,
	we have
	\begin{eqnarray}\label{eq:comain3}
	&&
	\frac{1}{T}\sum_{t=1}^{T}\E \left[ \left\lVert \nabla \fxts \right\rVert^2 \right]
	=
	O\left(
	\frac{ E^{\frac{1}{2}}  \left(G^2+\sigma^2 + BL\right)}
	{ \beta^{\frac{1}{2}} N^{\frac{1}{4}} T^{\frac{1}{2}} }   
	\right) 
	=
	O\left(
	\frac{ 1}
	{ N^{\frac{1}{4}} T^{\frac{1}{2}} }   
	\right),
	\end{eqnarray}
	where the final equation follows if we care only about $N$ and $T$, and regard other parameters as constants.
\end{proof}

As shown in Section~\ref{sec:introduction}, FedAvg is proven to achieve $O(1/\sqrt{NT})$ convergence when all clients participate in each training iteration. However, we can prove only the $O(1/(N^{1/4}T^{1/2}))$ convergence for FedLaAvg because of the partial client participation as a result of the intermittent client availability. Specifically, this gap is introduced by (\ref{eq:main-4}). 
The randomness of the stochastic gradient $\giTit$ is an obstacle for the convergence analysis.
With full client participation, we can reduce this randomness by the following equations:
\begin{eqnarray}\label{eq:gap-1}
&&
\E \left[ \left\langle\nabla \fxts,  
\frac{1}{N}\sum_{i=1}^{N}\giTit
\right\rangle \right]
=
\E \left[ \left\langle\nabla \fxts,  
\frac{1}{N}\sum_{i=1}^{N}\git
\right\rangle \right]
\new
&=&
\E \left[ \E\left[\left\langle\nabla \fxts,  
\frac{1}{N}\sum_{i=1}^{N}\git
\right\rangle \mid \xi^{[t-1]}\right] \right]
=
\E \left[ \left\langle\nabla \fxts,  
\frac{1}{N}\sum_{i=1}^{N}\nabla \fixts
\right\rangle \right].
\end{eqnarray}
However, with partial client participation, (\ref{eq:gap-1}) no longer holds. We analyze the gap between 
\[
	\E \left[ \left\langle\nabla \fxts,  
    \frac{1}{N}\sum_{i=1}^{N}\giTit
    \right\rangle \right]
    \text{ and }
    \E \left[  \left\langle \nabla \fxTtS , \frac{1}{N}\sum_{i=1}^{N}\nabla\fixTitS \right\rangle \right]
\]
in (\ref{eq:main-4}). 
Then, we further study the upper bound for the absolute value of the first term of the gap, i.e., 
\\$
\E \left[  \left\langle \nabla \fxts - \nabla  \fxTtS , \frac{1}{N}\sum_{i=1}^{N}\left(\giTit -  \nabla \fixTitS\right)\right\rangle \right] 
$
in (\ref{eq:main3}). This term is the inner product of two vectors.
The norm of the second vector is bounded by $O(1/N)$, but the norm of the fisrt term is not related to $N$. 
Hence, the upper bound that we can obtain for the inner product is $O(1/\sqrt{N})$, while the $O(1/\sqrt{NT})$ convergence needs an upper bound in the order of $O(1/N)$.
Whether the $O(1/(N^{1/4}T^{1/2}))$ convergence is a tight bound requires further studies.

\section{Detailed Convergence Proof for the Communication Round-Based Setting}
\label{sec:multiround proof}
To make the proof more concise, we introduce an 
mathematically equivalent Algorithm~\ref{alg:eqv} 
of Algorithm~\ref{alg:multiround}. 
Note that $\xt$ (when $t$ is not multiple of $C$)
is intermediate variable for mathematical analysis.
In addition, $\g_i^{\tau}$ ($\tau\leq 0$) is extraly 
defined to avoid undefined symbols when $\Rirt=0$ 
in (\ref{eq:eqv-update}). It can be proved by induction 
that all variables defined in 
Algorithm~\ref{alg:multiround} are consistent 
with those in Algorithm~\ref{alg:eqv}.

\begin{algorithm}
	\caption{An equivalent Algorithm of Algorithm~\ref{alg:multiround}}
	\label{alg:eqv}
	\begin{algorithmic}[1]
		\STATE {\bfseries Input:} Initial model $\x^{0}$
		
		\STATE $\g_i^{\tau}\leftarrow \ze,\,\forall i\in\{1,2,\cdots,N\},\,\tau\in\{0, -1, \cdots, 1-C\}$
		
		\STATE $R_i^{0} \leftarrow \ze,\,\forall i\in\{1,2,\cdots,N\}$
		
		\FOR{$t=1$ {\bfseries to} $RC$}
		\STATE $r^t \leftarrow \lfloor(t-1)/C\rfloor + 1$
		\IF{$t-1$ is a multiple of $C$}
		
		\STATE$\Ch^{r^t}\leftarrow$ the set of available clients in round $r^t$
		
		\STATE $\Bh^{r^t}\leftarrow$ $K$ clients from $\Ch^{r^t}$ with the lowest $\Rirts$ values
		
		\STATE Update $\Rirt$ values:
		$\Rirt \leftarrow r^t,\,\forall i\in \Bh^{r^t};\,\Rirt \leftarrow \Rirts,\,\forall i\notin \Bh^{r^t}$.
		
		\STATE $\xits \leftarrow \xts,\,\forall i \in \Bh^{r^t}$
		\ENDIF
		
		\STATE $\git \leftarrow \nabla F\left(\xits;\xi_i^{t }
		\right),\,\forall i\in\Bh^{r^t}$
		
		\STATE Update the global model parameters:
		\begin{equation}\label{eq:eqv-update}
		\x^{t}\leftarrow \x^{t - 1 } - 
		\gamma \sum_{i=1}^{N}\giTittau.
		\end{equation}
		
		\STATE Update the local model parameters:
		\begin{equation}\label{eq:eqv-local update}
		\xit \leftarrow \xits - 
		\gamma \git.
		\end{equation}
		\ENDFOR
	\end{algorithmic}
\end{algorithm}
	
With equavalence between 
Algorithm~\ref{alg:multiround} and \ref{alg:eqv}
established, we indroduce the corresponding equavalent lemmas of Lemmas~\ref{le:I}--\ref{le:ft-ftit}.

\begin{lemma}\label{le:I-2}
	Under Assumption~\ref{as:5}, following Algorithm~\ref{alg:multiround}, with $I=\lceil N / K\rceil E - 1$, $\forall r,\forall i$, we have
	\begin{equation}
	r-\Rir \leq I.
	\end{equation}
\end{lemma}
\begin{proof}[Proof of Lemma~\ref{le:I-2}]
	Replacing $t$ with $r$ and $\Tit$ with $\Rir$, the proof is exactly the same with that of Lemma~\ref{le:I}.
\end{proof}

\begin{lemma}
	\label{le:2-2}
	Corresponding lemma of Lemma~\ref{le:2}:
	\[
	\E \left[ \left\lVert \git \right\rVert ^2 \right] \leq G^2,\forall i,\forall t;
	\]
	\[
	\E \left[ \left\lVert \gittau - \nabla \fixttaus \right\rVert ^2 \right] \leq \sigma^2,\forall i,\forall t.
	\]
\end{lemma}
\begin{proof}[Proof of Lemma~\ref{le:2-2}]
	Replacing $\xts$ with $\xits$, the proof is exactly the same with that of Lemma~\ref{le:2}.
\end{proof}

\begin{lemma}
	\label{le:eg-ef-2}
	Corresponding lemma of Lemma~\ref{le:eg-ef}:
	$\forall i, \forall t$, we have
	\begin{eqnarray}
	&&
	\E \left[ \left\lVert\sum_{i=1}^{N} \left( \giTittau-\nabla \fixTittaus \right) \right\rVert^2 \right]
	\new 
	&=&
	\sum_{i=1}^{N}\E \left[ \left\lVert\giTittau-\nabla \fixTittaus\right\rVert^2 \right] \nonumber.
	\end{eqnarray}
\end{lemma}
\begin{proof}[Proof of Lemma~\ref{le:eg-ef-2}]
	Replacing $\giTit$ with $\giTittau$ and $\nabla \fixTitS$ with $\nabla \fixTittaus$, the proof is exactly the same with that of Lemma~\ref{le:eg-ef}.
\end{proof}
	
Note that Lemma~\ref{le:ft-ftit} and Corollary~\ref{co:ft-ftit-1} still hold. Their proof follows as well if we replace the relation $\x^{  \tau  }-\x^{  \tau -1  } = \sum_{j=1}^{N}\gjTjtau$ with $\x^{  \tau  }-\x^{  \tau -1  } = \sum_{j=1}^{N}\intermediateForUpdate$.

\begin{proof}[Main proof of Theorem~\ref{co:multiround}]
	The proof is similar to that of Theorem~\ref{the:main} and Corollary~\ref{co:main}. We illustrate it in detail as follow.
	
	Fix $t\geq 1$, by Assumption~\ref{as:1}, we have
	\begin{eqnarray}\label{eq:mul-0}
	\E \left[ f ( \x^{  t   } )  \right] 
	\leq
	\E \left[ \fxts \right] 
	+\frac{L}{2}\E \left[ \left\lVert \xt - \xts \right\rVert ^2 \right] 
	+\E \left[ \left\langle\nabla \fxts, 
	\xt - \xts \right\rangle \right].
	\end{eqnarray}
	
	Focus on the second term on the right. Following the procedure of (\ref{eq:main-1}), we omit the intermediate results and show the final bound:
	\begin{eqnarray}\label{eq:mul-1}
	\E \left[ \left\lVert \xt - \xts \right\rVert ^2 \right]
	\leq
	\frac{2\gamma^2 \sigma^2}{N} 
	+ 2 \gamma^2 \E\left[\left\lVert \frac{1}{N}\sum_{i=1}^{N}  \nabla \fixTittaus \right\rVert ^2 \right].
	\end{eqnarray}
	
	For simplicity, we define $\Tht$ as $\min_{i} \left(\Rirt C - r^t C + t \right)$. Focus on the third term in (\ref{eq:mul-0}), we can separate it into 3 parts
	\begin{eqnarray}\label{eq:mul-4}
	&&
	\E \left[ \left\langle\nabla \fxts, 
	\xt - \xts \right\rangle \right]
	\new
	&=&
	-\gamma \E \left[  \left\langle \nabla \fxts -  \nabla \fxThtS , \frac{1}{N}\sum_{i=1}^{N}\left(\giTittau -  \nabla \fixTittaus \right)\right\rangle \right] 
	\newl
	-\gamma \E \left[  \left\langle \nabla \fxts - \nabla  \fxThtS , \frac{1}{N}\sum_{i=1}^{N}  \nabla \fixTittaus \right\rangle \right] 
	\newl
	- \gamma \E \left[  \left\langle \nabla \fxThtS , \frac{1}{N}\sum_{i=1}^{N}\giTittau \right\rangle \right].
	\end{eqnarray}
	
	We further focus on the first term in (\ref{eq:mul-4}). Following the procedure of (\ref{eq:main3}), we have the following bound:
	\begin{eqnarray}\label{eq:mul3}
	&&
	-\gamma \E \left[  \left\langle 
	\nabla \fxts -  \nabla \fxThtS 
	, 
	\frac{1}{N}\sum_{i=1}^{N}\left(\giTittau -  \nabla \fixTittaus \right)
	\right\rangle \right] 
	\new
	&\leq&
	\frac{\gamma^2 I C L \left(G^2 + \sigma^2\right)}{2\sqrt{N}}.
	\end{eqnarray}
	
	Then we focus on the second term in (\ref{eq:mul-4}). Following the procedure of (\ref{eq:main4}), we have
	\begin{eqnarray}\label{eq:mul4}
	&&
	\gamma \E \left[  \left\langle 
	\nabla \fxts - \nabla  \fxThtS 
	, 
	\frac{1}{N}\sum_{i=1}^{N}  \nabla \fixTittaus 
	\right\rangle \right]  
	\new
	&\leq&
	\frac{\gamma^3 I^2 C^2 L^2 G^2}{2\left(1 - 2\gamma L\right)} 
	+
	\frac{\gamma \left(1 - 2\gamma L\right)}{2} \E\left[ \left\lVert 
	\frac{1}{N} \sum_{i=1}^{N}  \nabla \fixTittaus  
	\right\rVert^2 \right].
	\end{eqnarray}
	
	We finally focus on the third term in (\ref{eq:mul-4}). Following the procedure of (\ref{eq:main5}) and (\ref{eq:main6}), we have
	\begin{eqnarray}\label{eq:mul5}
	&&
	\E \left[  \left\langle \nabla \fxThtS , \frac{1}{N}\sum_{i=1}^{N}\giTittau \right\rangle \right]
	\new
	&\overset{(b)}{=}&
	\frac{1}{2}\E \left[ \left\lVert \nabla \fxThtS \right\rVert^2 \right] 
	+ \frac{1}{2}\E\left[ \left\lVert \frac{1}{N}\sum_{i=1}^{N} \nabla \fixTittaus \right\rVert^2 \right]
	-  \frac{1}{2} I^2 L^2 C^2 \gamma^2 G^2.
	\newl
	\end{eqnarray}
	
	Substituting (\ref{eq:mul3})--(\ref{eq:mul5}) into (\ref{eq:mul-4}), we have:
	\begin{eqnarray}\label{eq:mul-9}
	&&
	\E \left[ \left\langle\nabla \fxts, 
	\xt - \xts \right\rangle \right]
	\new
	&\leq&
	\frac{\gamma^2 I C L \left(G^2 + \sigma^2\right)}{2\sqrt{N}}
	+
	\frac{\gamma^3 I^2 C^2 L^2 G^2}{2\left(1 - 2\gamma L\right)} 
	+  
	\frac{\gamma^3  I^2 C^2 L^2 G^2}{2}
	\newl
	-
	\gamma^2L \E\left[ \left\lVert \frac{1}{N}\sum_{i=1}^{N} \nabla  \fixTittaus \right\rVert^2 \right]  
	- 
	\frac{\gamma}{2}\E \left[ \left\lVert \nabla \fxThtS \right\rVert^2 \right].
	\end{eqnarray}
	
	Further substituting (\ref{eq:mul-1}) and (\ref{eq:mul-9}) into (\ref{eq:mul-0}), we have
	\begin{eqnarray}\label{eq:mul-fin}
	&&
	\E \left[ f ( \x^{  t  } )  \right] 
	-
	\E \left[ \fxts \right] 
	\new 
	&\leq&
	\frac{\gamma^2 \sigma^2 L}{N} 
	+ 
	\frac{\gamma^2 IC L \left(G^2+\sigma^2\right)}{2\sqrt{N}} 
	+
	\frac{\gamma^3 I^2 C^2 L^2 G^2}{2\left(1 - 2\gamma L\right)} 
	+  
	\frac{\gamma^3  I^2 C^2 L^2 G^2}{2} 
	- 
	\frac{\gamma}{2}\E \left[ \left\lVert \nabla \fxThtS \right\rVert^2 \right].\newl
	\end{eqnarray}
	
	Rearrange the above equation and we have 
	\begin{eqnarray}\label{eq:mul-fin-1}
	\E \left[ \left\lVert \nabla \fxThtS \right\rVert^2 \right] 
	\leq
	&&
	\frac{2\gamma \sigma^2 L}{N} 
	+ 
	\frac{\gamma IC L \left(G^2 + \sigma^2\right)}{\sqrt{N}} 
	+
	\frac{\gamma^2 I^2 C^2 L^2 G^2}{\left(1 - 2\gamma L\right)} 
	+  
	\gamma^2  I^2 C^2 L^2 G^2 
	\newl
	+
	\frac{2}{\gamma} \left(\E \left[ \fxts \right] 
	- 
	\E \left[ \fxt \right] \right).
	\end{eqnarray}
	
	Summing (\ref{eq:mul-fin-1}) over iterations from 1 to $RC$ and dividing both sides by $RC$, we have
	\begin{eqnarray}\label{eq:mul-fin-2}
	\frac{1}{RC}\sum_{t=1}^{RC}\E \left[ \left\lVert \nabla \fxThtS \right\rVert^2 \right] 
	\leq
	&&
	\frac{2\gamma \sigma^2 L}{N} 
	+ 
	\frac{\gamma IC L \left(G^2+\sigma^2\right)}{\sqrt{N}}   
	+
	\frac{\gamma^2 I^2 C^2 L^2 G^2}{\left(1 - 2\gamma L\right)} 
	\newl
	+  
	\gamma^2  I^2 C^2 L^2 G^2 
	+
	\frac{2}{\gamma RC} \left(\E \left[ f(\x^{0 }) \right] 
	- 
	\E \left[ f ( \x^* )  \right] \right),
	\end{eqnarray}
	where $\x^*$ is the optimal value for the objective function $f(\x)$.
	
	Finally, we build the gap between $\nabla \fxrtCs$ and $\nabla \fxThtS$.
	Lemma~\ref{le:I-2} implies that $t - \Tht \leq IC$, thus $r^tC - \Tht\leq (I+1)C $. Hence, we have
	\begin{eqnarray}\label{eq:mul-fin-3}
	&&
	\E\left[\left\lVert \nabla  \fxrtCs \right\rVert^2\right]
	\new
	&\leq&
	2 \E\left[\left\lVert  \nabla \fxrtCs - \frac{1}{C}\sum_{\tau=(r_t-1)C+1}^{r_t C } \nabla  \fxThtauS \right\rVert^2 \right] 
	\newl
	+
	2 \E \left[ \left\lVert \frac{1}{C} \sum_{\tau=(r_t-1)C+1}^{r_t C } \nabla  \fxThtauS \right\rVert^2\right]
	\new
	&\leq&
	2 \gamma^2 \left(I+1\right)^2 C^2 L^2 G^2 
	+
	2 \frac{1}{C} \sum_{\tau=(r_t-1)C+1}^{r_t C} \E \left[ \left\lVert   \nabla  \fxThtauS \right\rVert^2\right],
	\end{eqnarray}
	which follows from \convex  and Corollary~\ref{co:ft-ftit-1}.
	
	Summing \ref{eq:mul-fin-3} over $t\in \left\{C, 2C, \cdots, RC\right\}$, dividing both sides by $R$ and substituting \ref{eq:mul-fin-2} into it, we have
	\begin{eqnarray}\label{eq:mul-the}
	&&
	\frac{1}{R}\sum_{r=1}^{R}\E \left[ \left\lVert \nabla \fxrCs \right\rVert^2 \right]
	\new 
	&\leq&
	\frac{4\gamma \sigma^2 L}{N} 
	+ 
	\frac{2\gamma IC L \left(G^2+\sigma^2\right)}{\sqrt{N}}   
	+  
	\left( \frac{2I^2}{\left(1 - 2\gamma L\right)} + 4I^2 + 4I + 2\right) \gamma^2 C^2 L^2 G^2 
	\newl
	+
	\frac{4}{\gamma RC} \left(\E \left[ f(\x^{ 0}) \right] 
	- 
	\E \left[ f ( \x^* )  \right] \right).
	\end{eqnarray}
	
	Then, we write the $O(\cdot)$ expression of the above equation:
	\begin{eqnarray}\label{eq:mul-co1}
	\frac{1}{R}\sum_{t=1}^{R}\E \left[ \left\lVert \nabla \fxrCs \right\rVert^2 \right]
	=
	O\left(
	\frac{\gamma IC L \left(G^2+\sigma^2\right)}{\sqrt{N}}   
	+  
	\frac{I^2 \gamma^2 C^2 L^2 G^2}{\left(1 - 2\gamma L\right)}  
	+
	\frac{B}{\gamma RC}
	\right).
	\end{eqnarray}
	
	Substituting $\gamma$ with $( \beta^{1/2} N^{1/4} )/(2L C E^{1/2} R^{1/2} ) $, we have
	\begin{eqnarray}\label{eq:mul-co2}
	\frac{1}{R}\sum_{t=1}^{R}\E \left[ \left\lVert \nabla \fxrCs \right\rVert^2 \right]
	=
	O\left(
	\frac{ E^{\frac{1}{2}}  \left(G^2+\sigma^2 + B L\right)}
	{ \beta^{\frac{1}{2}} N^{\frac{1}{4}} R^{\frac{1}{2}} }   
	+  
	\frac{ E  G^2  N^{\frac{1}{2}} }{ \beta  R }   
	\right).
	\end{eqnarray}
	
	If we further choose $R >  E N^{3/2} /\beta $,
	we have
	\begin{eqnarray}\label{eq:mul-co3}
	&&
	\frac{1}{R}\sum_{t=1}^{R}\E \left[ \left\lVert \nabla \fxrCs \right\rVert^2 \right]
	=
	O\left(
	\frac{ E^{\frac{1}{2}}  \left(G^2+\sigma^2\right) + BL}
	{ \beta^{\frac{1}{2}} N^{\frac{1}{4}} R^{\frac{1}{2}} }   
	\right) 
	=
	O\left(
	\frac{ 1}
	{ N^{\frac{1}{4}} R^{\frac{1}{2}} }   
	\right).
	\end{eqnarray}
	The final equation follows if we care only about $N$ and $R$, and regard other parameters as constants.
\end{proof}

% FedAvg lets each client perform multiple local iterations to achieve communication efficiency. 
% However, the theoretical analysis can not support the substantial improvement in communication efficiency for the communication round-based FedLaAvg.
% The main reason is that with $C$ local iterations, the difference between the latest gradient and the current gradient is roughly $C$ times larger, and we have to choose smaller $\gamma$ to guarantee convergence.
% Since we finally choose $\gamma\propto 1/C$, increasing local iteration number can help reduce the variance term $4\gamma \sigma^2 L/N$ in (\ref{eq:mul-the}), which is not the dominating term, however.
% In contrast, we observe in the empirical studies that 
% introducing multiple local iterations to FedLaAvg 
% significantly speeds up the convergence, 
% implying that the convergence rate 
% in Theorem~\ref{co:multiround}
% may not be tight.
% How introducing local iterations theoretically 
% speeds up the convergence is still an open question.

\section{Complexity Analysis}
\label{sec:complexity}

We analyze the time and space complexity of Algorithms~\ref{alg:LAS} and \ref{alg:multiround} in this appendix.
We use $P$ to denote the time complexity of one backpropagation and $Q$ to denote the number of parameters in the deep learning model. 

In each iteration of Algorithm~\ref{alg:LAS}, each client performs one backpropagation to obtain the local gradient and computes the gradient difference. This requires $O(P+Q)$ time complexity per client per iteration and $O(Q)$ space complexity to locally store the gradient calculated in the previous participating iteration. The cloud server selects $K$ clients from $\CC^t$ in each iteration $t$. Our implementation is sorting an array of $\Tits$ first and picking the $K$ clients from $\CC^t$ with the lowest $\Tits$ according to the sorted array. This requires $O(N\log{N})$ time complexity to sort the array and $O(N)$ space complexity to store the array.
Then, the cloud server aggregates the gradient difference to obtain the average latest gradient $\gt$, and update the global model. This requires $O(KQ)$ time complexity and $O(Q)$ space complexity.
To summarize, the time complexity of each iteration in Algorithm~\ref{alg:LAS} is $O(P+Q)$ on each client and $O(N\log{N}+KQ)$ on the cloud server. The space complexity is $O(Q)$ on each client and $O(N + Q)$ on the cloud server.
By similar analysis, the time complexity of each round in Algorithm~\ref{alg:multiround} is $O(CP+Q)$ on each client and $O(N\log{N}+KQ)$ on the cloud server. The space complexity is $O(Q)$ on each client and $O(N + Q)$ on the cloud server.

Compared with FedAvg, FedLaAvg only needs to additionally store the latest gradients of all clients, incurring $O(Q)$ disk space on each resource-limited client and $O(N+Q)$ memory space  on the resource-rich cloud server, which are acceptable and affordable.

\section{Supplementary Experiment Results}
\label{sec:sup exp}
In this section, we show the test accuracies with training rounds for each experiment.
Figure~\ref{fig:mnist_ta} compares the test accuracies of FedLaAvg and other baselines
in the MNIST image classification task under various availability settings;
Figure~\ref{fig:mnist_self_test-acc(real)} shows the test accuracies 
of FedLaAvg with different total number of clients $N$ 
and proportion of selected clients $\beta$;
Figure~\ref{fig:sentiment_ta} compares the test accuracies of FedLaAvg and other baselines
in the  Sentiment140 dataset under various availability settings;
Figure~\ref{fig:sentiment_self_test-acc(real)} compares 
shows the test accuracies 
of FedLaAvg with different $N$ and $\beta$.
All figures show consistent results with the training losses

\begin{figure*}[tb]
	\centering	
	\subfigure[Legends]{\label{mnist_legends_ta}\includegraphics[width=0.25\textwidth]{figures/mnist_legend}}
	\subfigure[$E=100$, $D=3$]{\label{fig:mnist_E100_D3_ta}\includegraphics[width=0.3\textwidth]{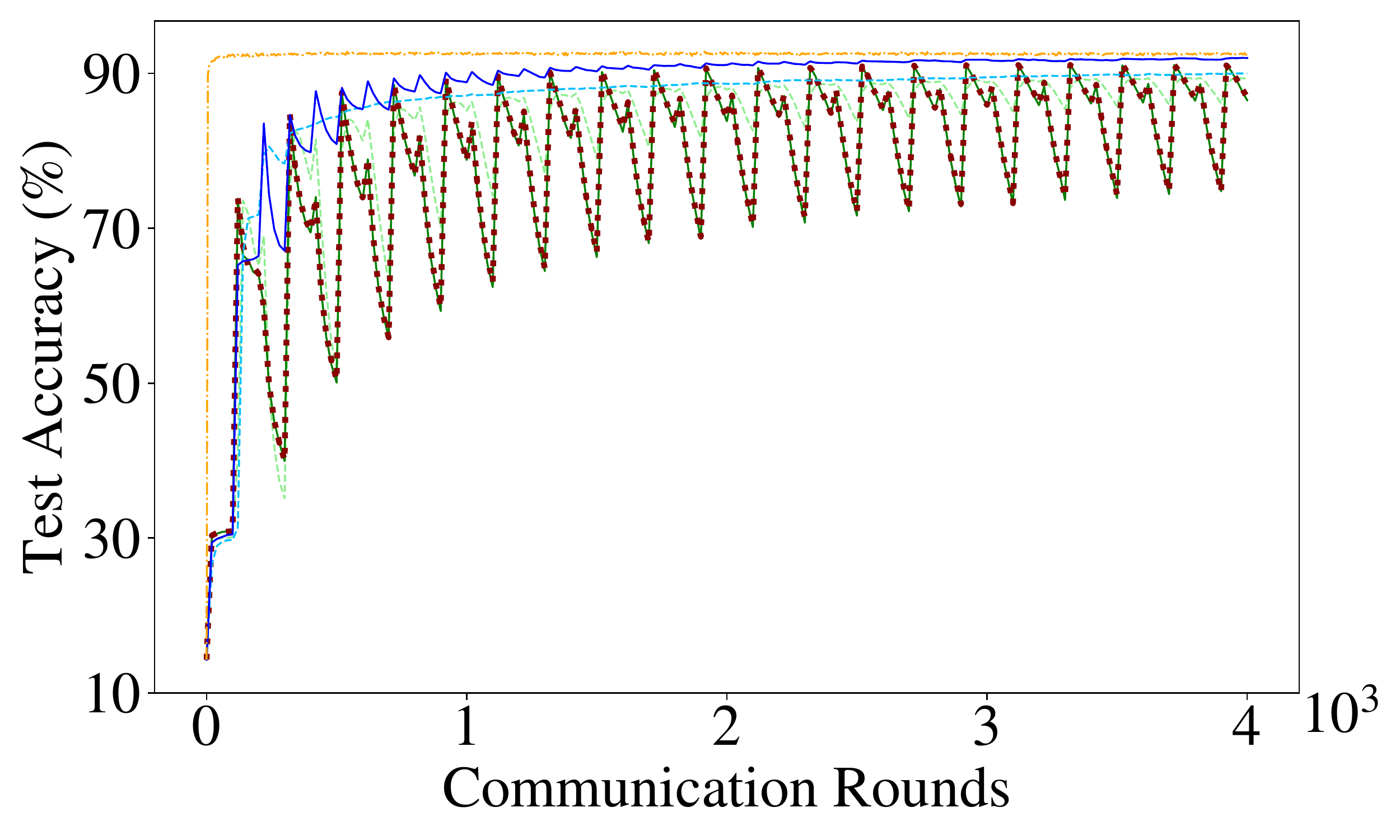}}
	\subfigure[$E=100$, $D=5$]{\label{fig:mnist_E100_D5_ta}\includegraphics[width=0.3\textwidth]{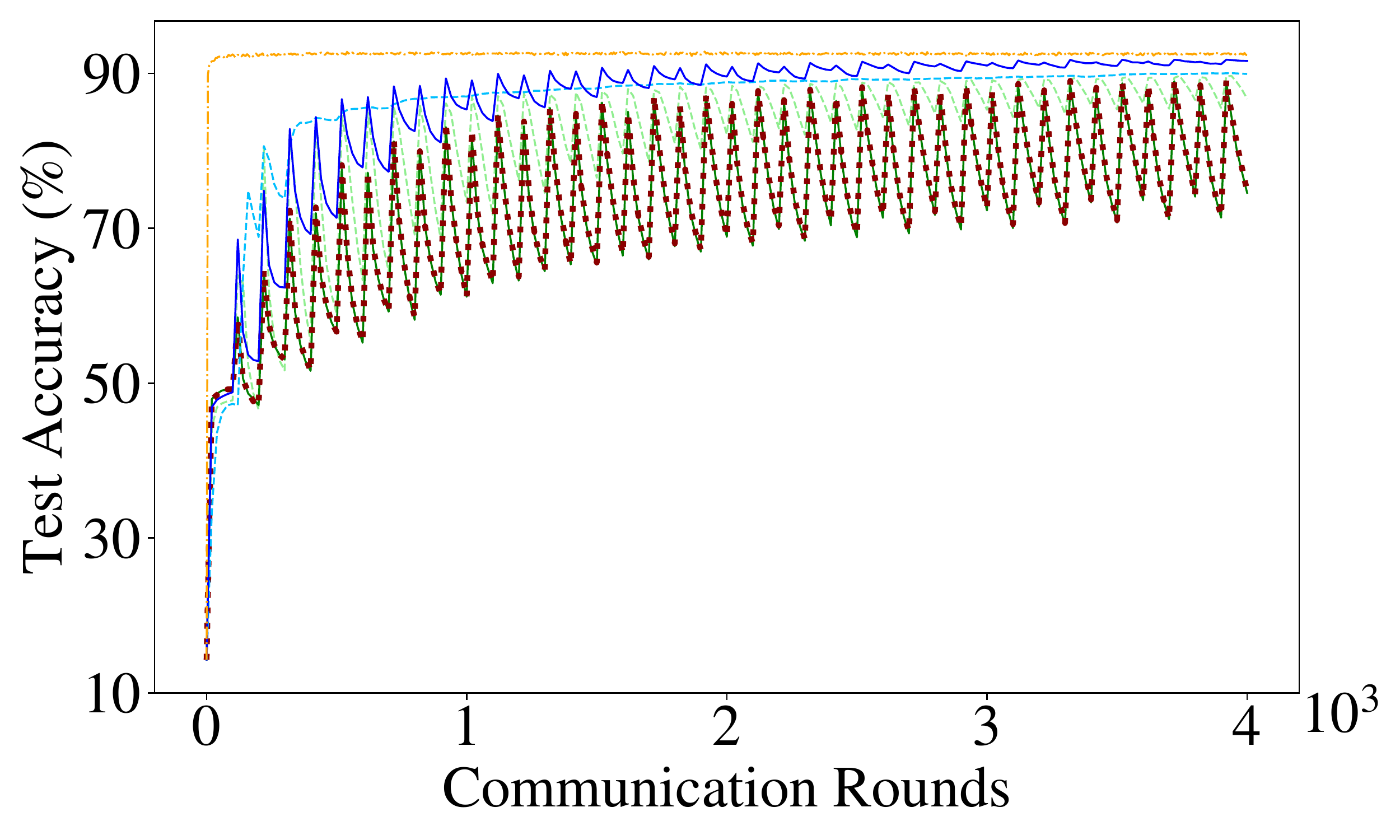}}
	\subfigure[$E=50$, $D=1$]{\label{fig:mnist_E50_D1_ta}\includegraphics[width=0.3\textwidth]{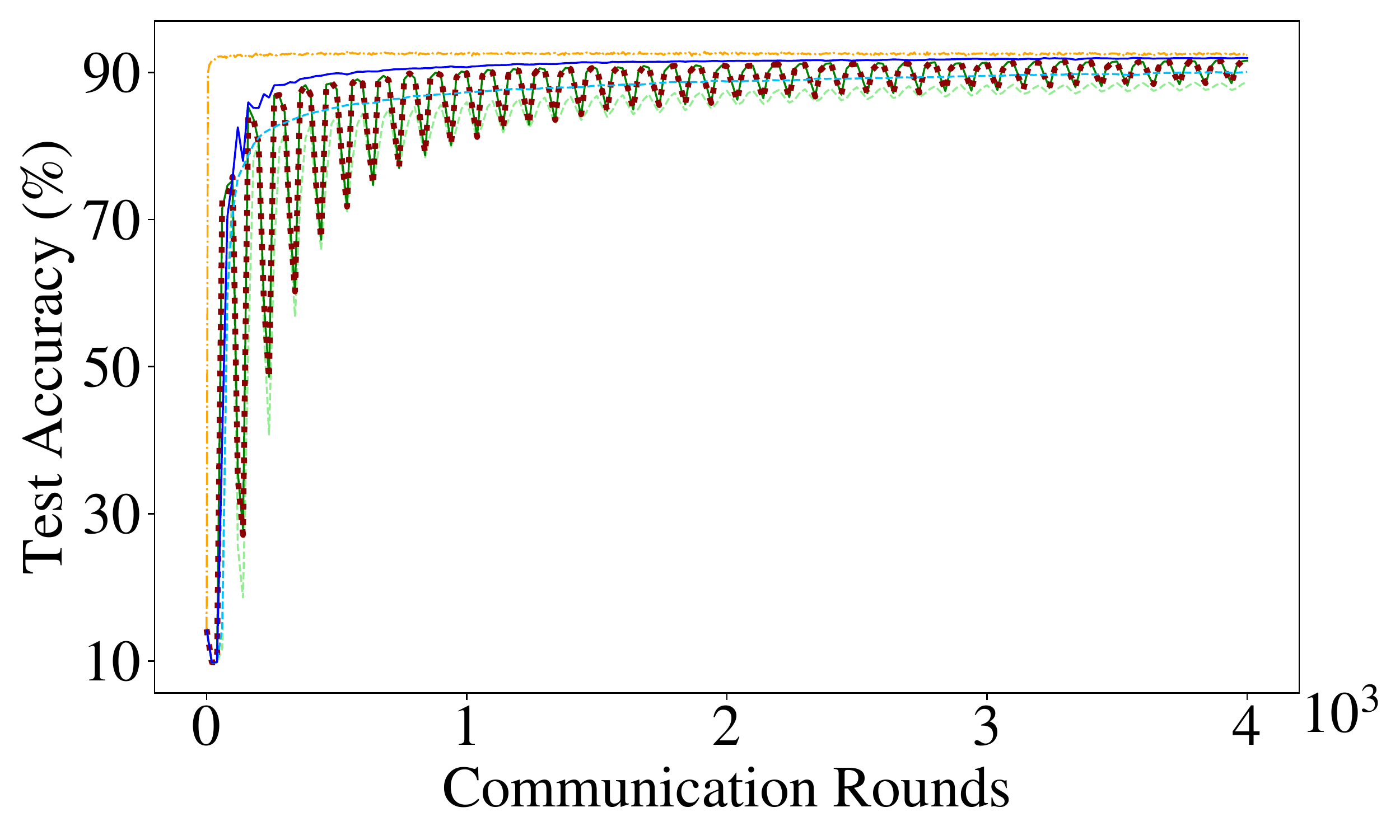}}
	\subfigure[$E=100$, $D=1$]{\label{fig:mnist_E100_D1_ta}\includegraphics[width=0.3\textwidth]{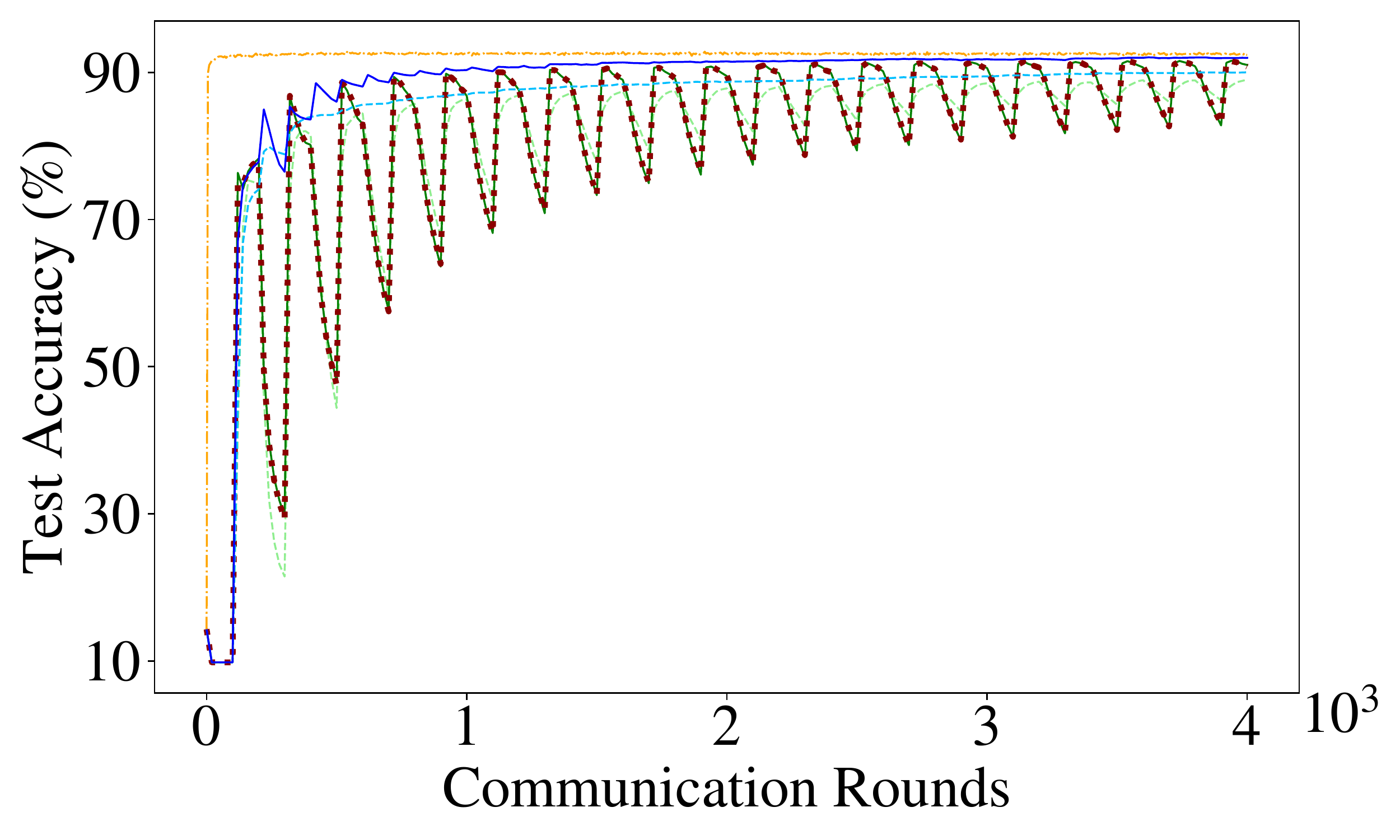}}
	\subfigure[$E=200$, $D=1$]{\label{fig:mnist_E200_D1_ta}\includegraphics[width=0.3\textwidth]{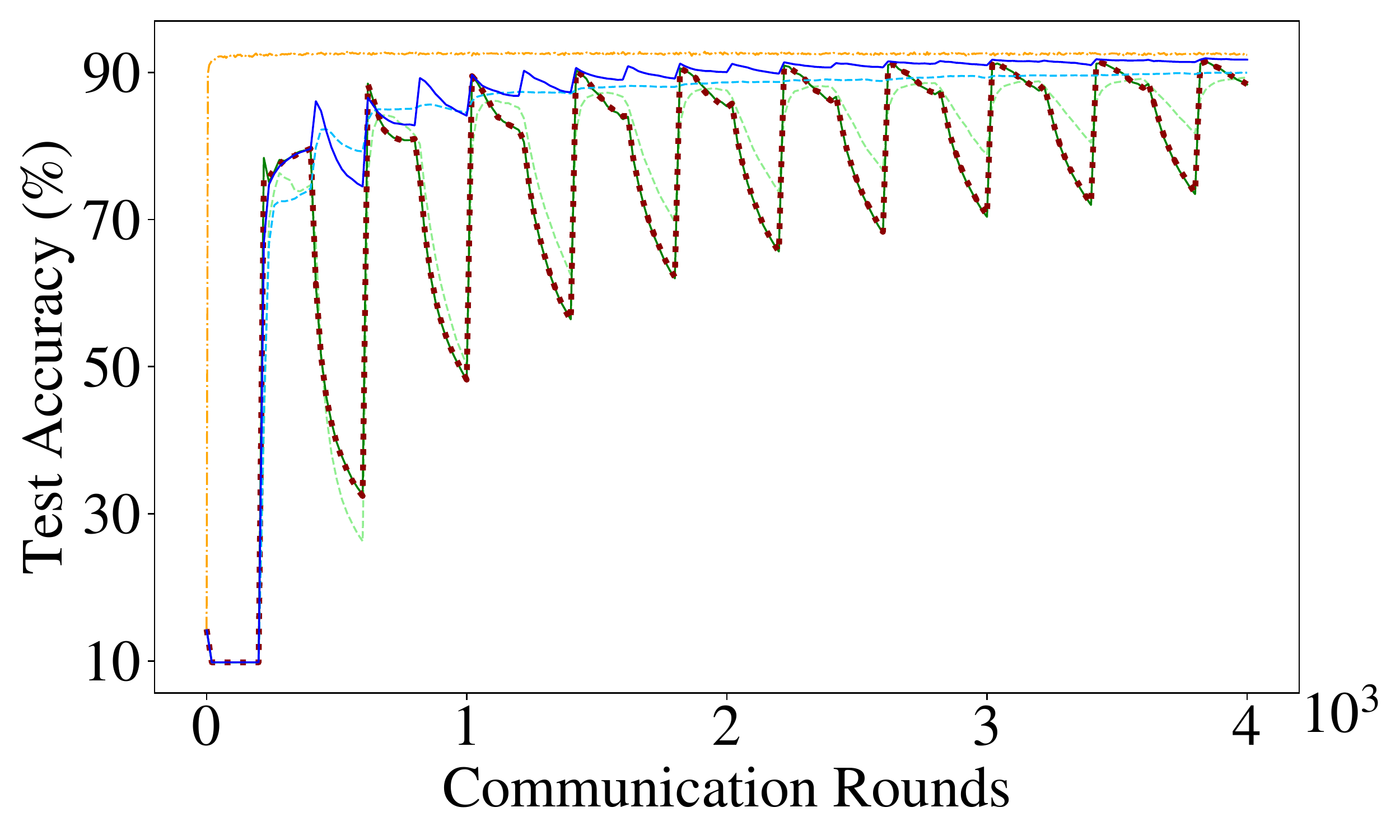}}
	\caption{Test accuracies of
	FedSGD, FedAvg, FedProx, FedLaAvg, and sequential SGD
	in the MNIST image classification task 
	with different client availability settings.}
	\label{fig:mnist_ta}
\end{figure*}

\begin{figure*}[tb]
	\centering
	\subfigure[Test accuracy with different $N$]{\label{fig:mnist_N_ta}\includegraphics[width=0.4\textwidth]{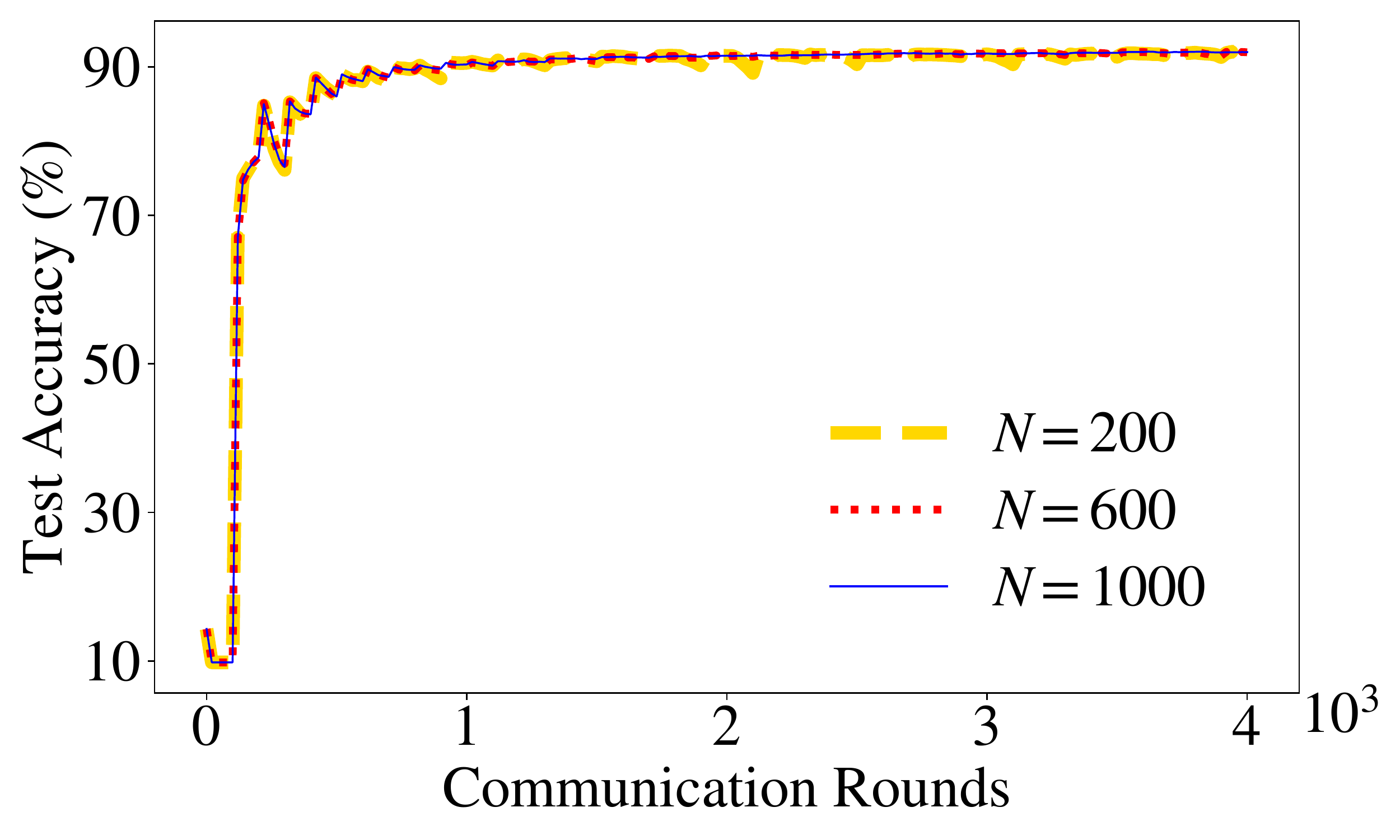}}
	\subfigure[Test accuracy with different $\beta$]{\label{fig:mnist_beta_ta}\includegraphics[width=0.4\textwidth]{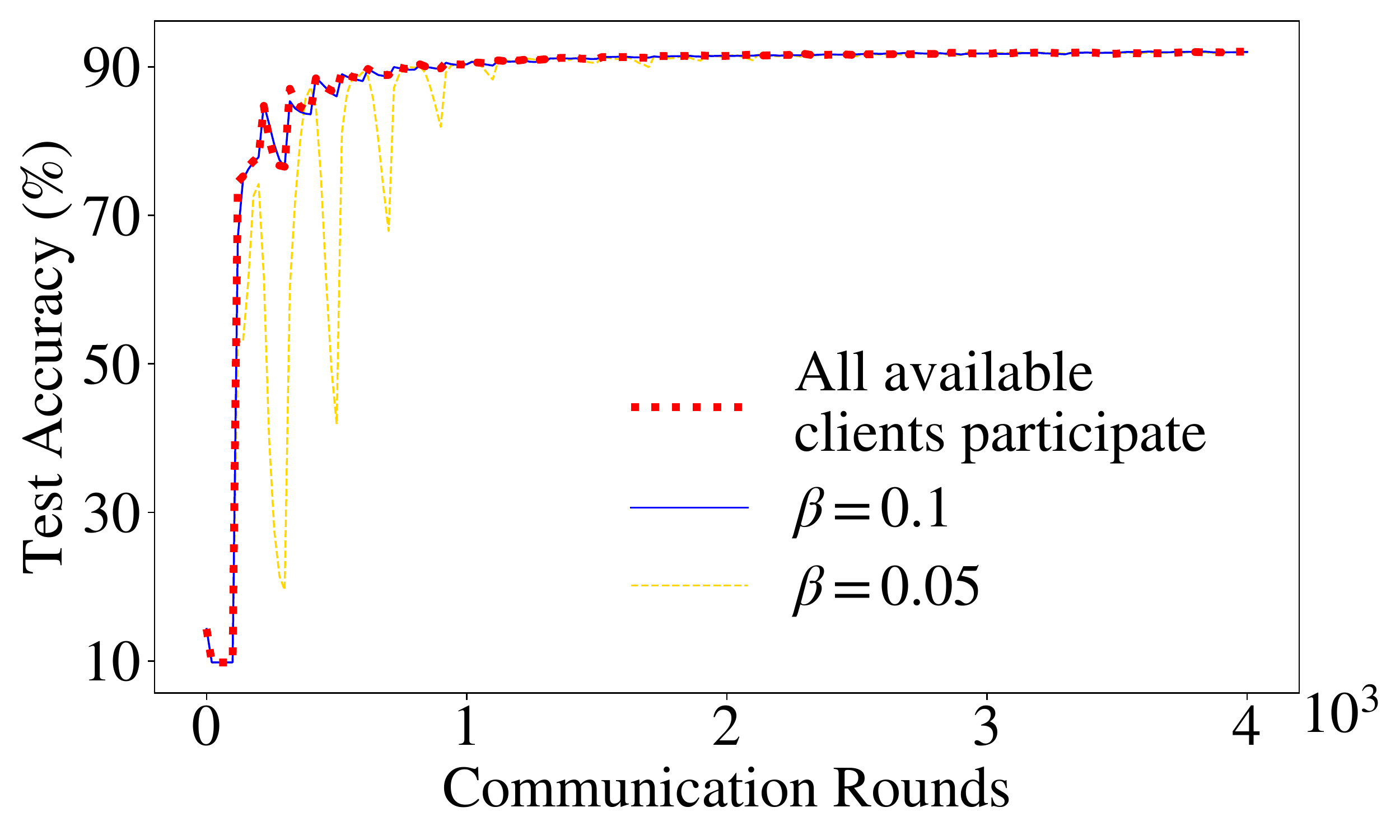}}
	\caption{Test accuracies of FedLaAvg on  MNIST dataset by varying 
	the total number of clients $N$ and the proportion of selected clients $\beta$.}
	\label{fig:mnist_self_test-acc(real)}
\end{figure*}

\begin{figure*}
	\subfigure[Legends]{\label{sentiment_legends_ta}\includegraphics[width=0.25\textwidth]{figures/sentiment_legend}}
	\subfigure[$E=120$, $\alpha=0.25$]{\label{fig:sentiment_E120_a0.25_ta}\includegraphics[width=0.3\textwidth]{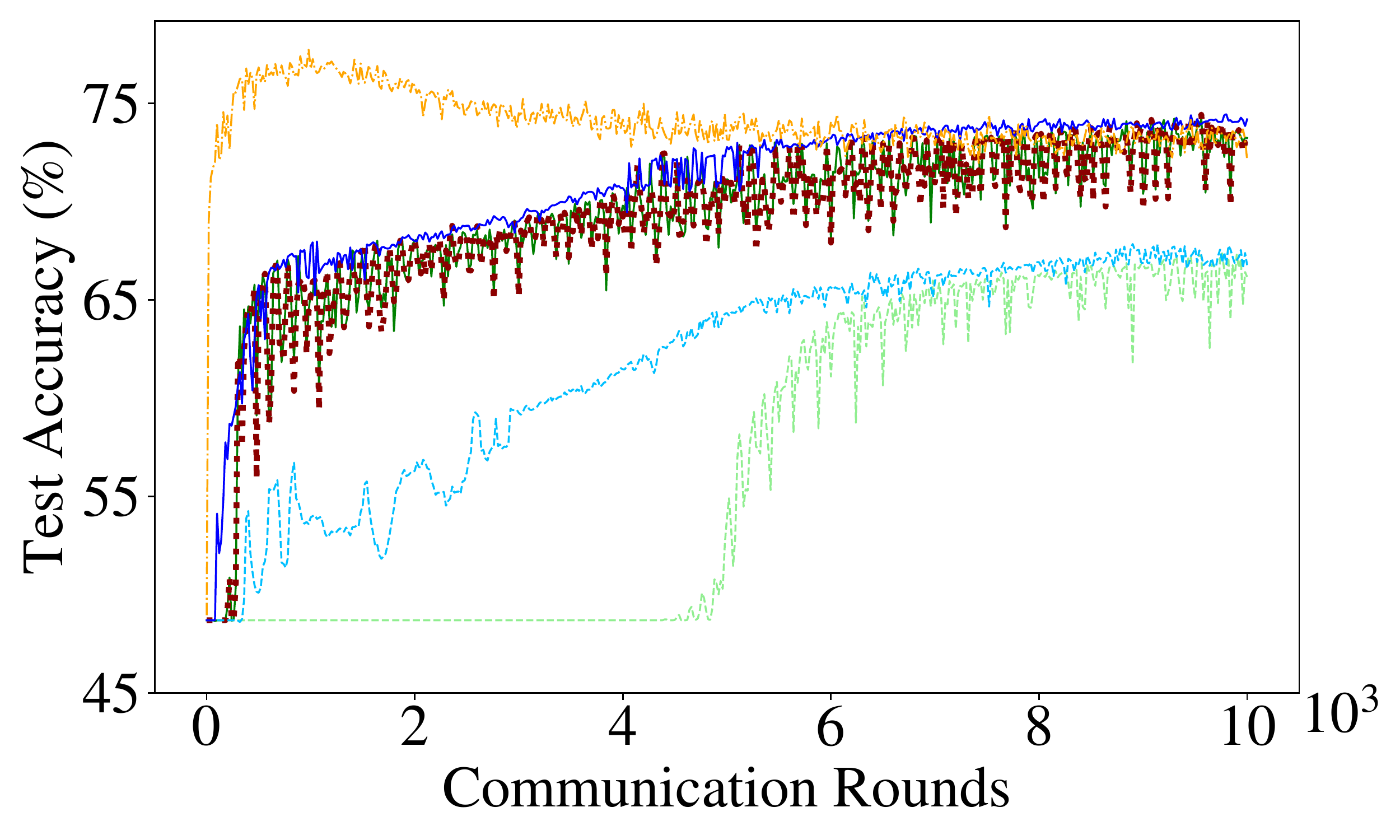}}
	\subfigure[$E=120$, $\alpha=0.5$]{\label{fig:sentiment_E120_a0.5_ta}\includegraphics[width=0.3\textwidth]{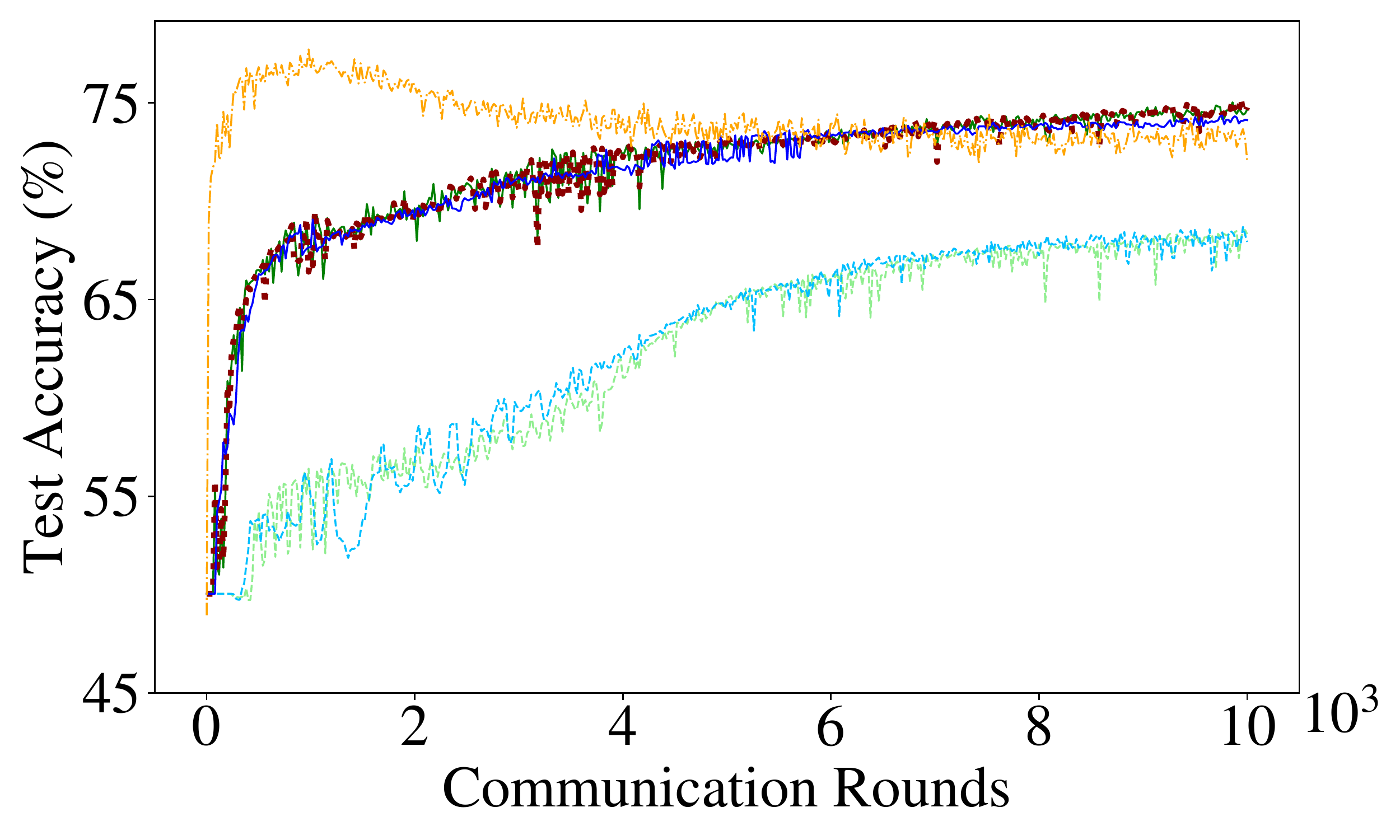}}
	\subfigure[$E=24$, $\alpha=0$]{\label{fig:sentiment_E24_a0_ta}\includegraphics[width=0.3\textwidth]{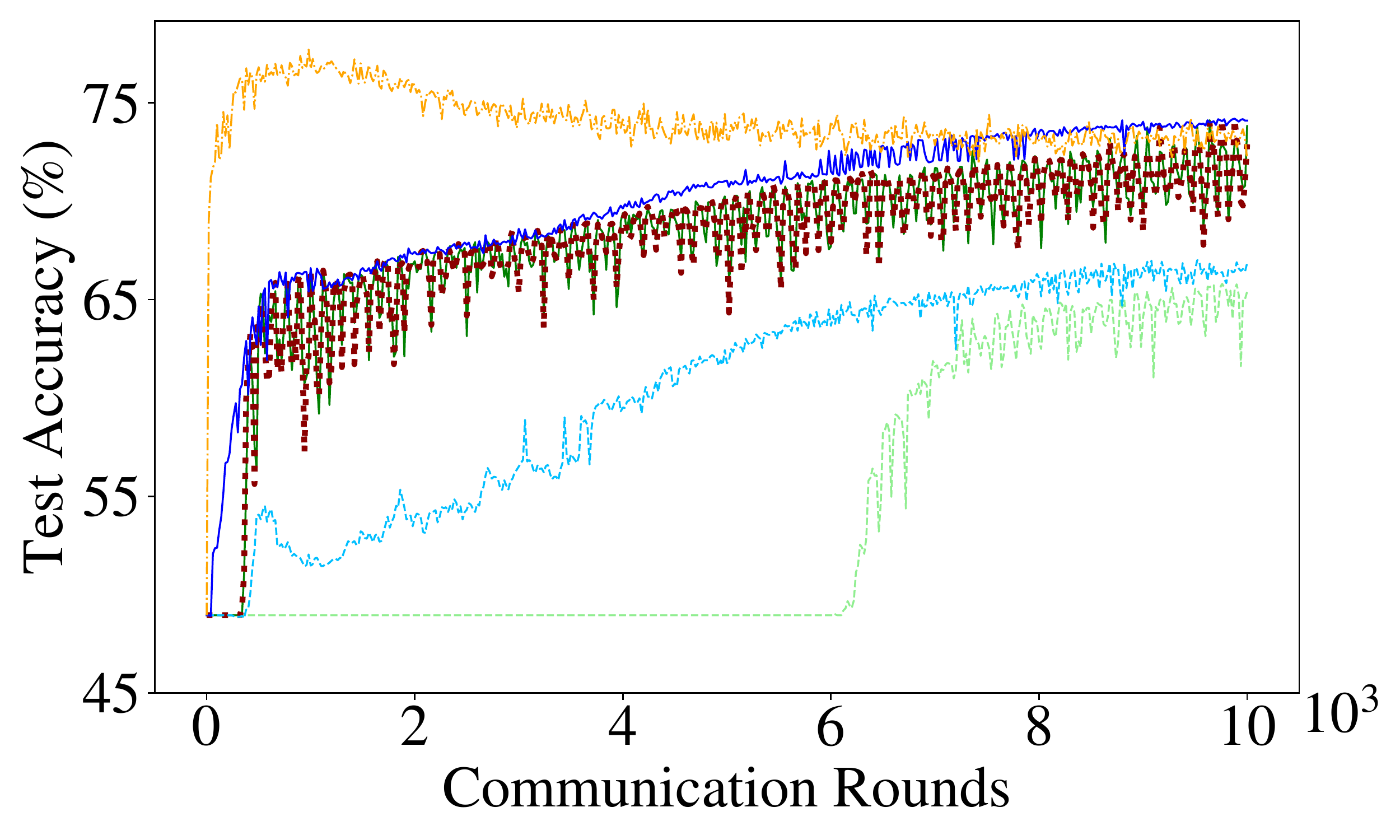}}
	\subfigure[$E=120$, $\alpha=0$]{\label{fig:sentiment_E120_a0_ta}\includegraphics[width=0.3\textwidth]{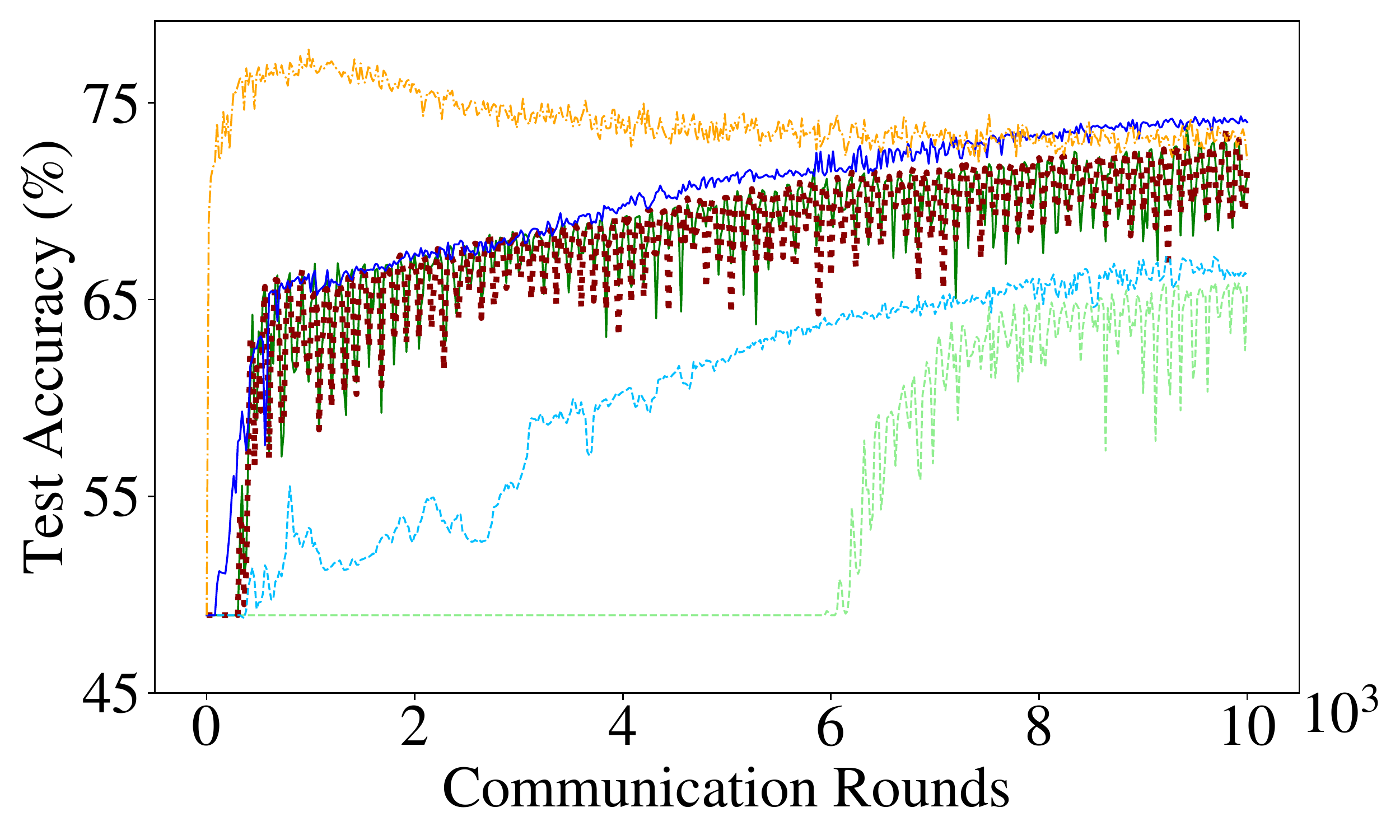}}
	\subfigure[$E=240$, $\alpha=0$]{\label{fig:sentiment_E240_a0_ta}\includegraphics[width=0.3\textwidth]{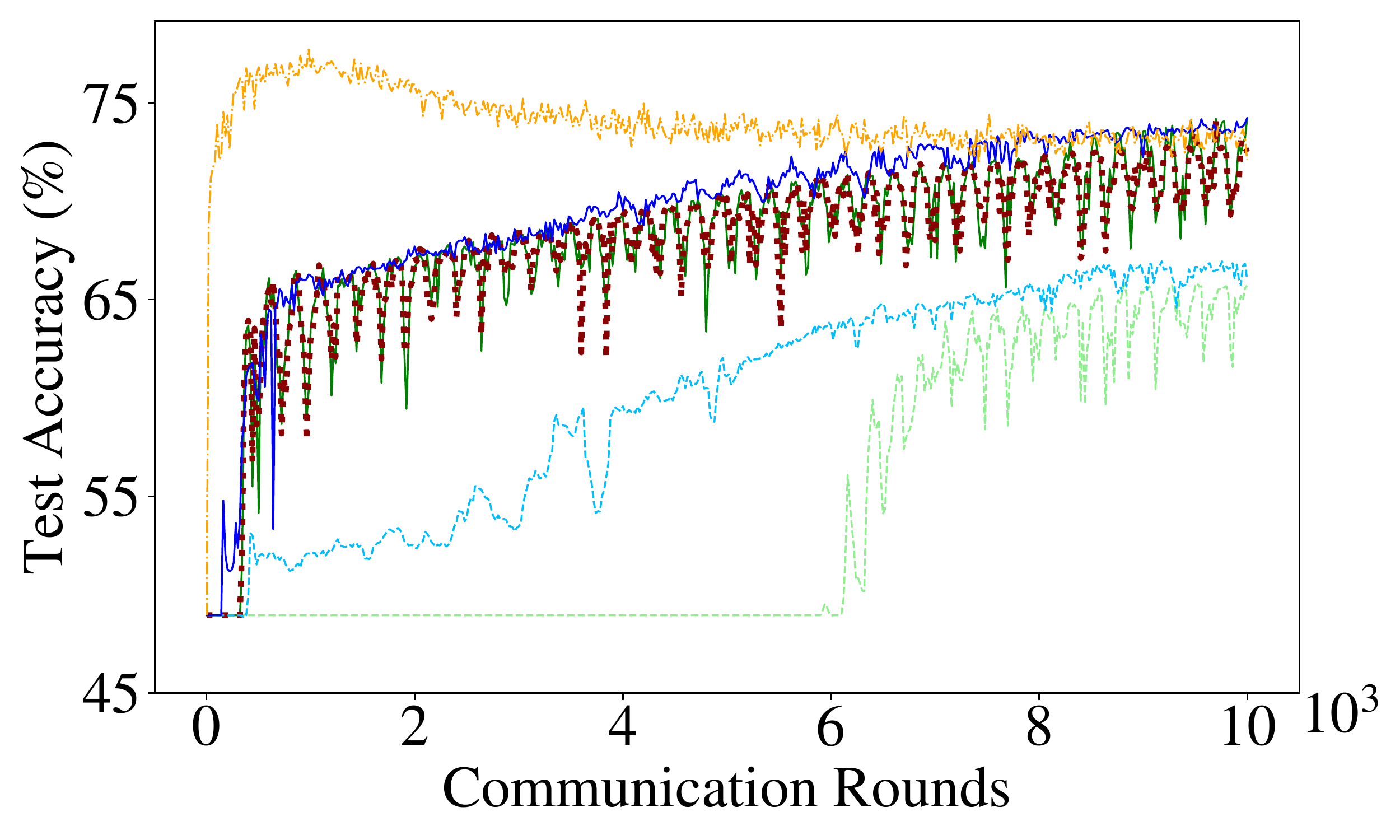}}
	\caption{Test accuracies of
	FedSGD, FedAvg, FedProx, FedLaAvg, and sequential SGD
	in the Sentiment140 sentiment analysis task 
	with different client availability settings.}
	\label{fig:sentiment_ta}
\end{figure*}

\begin{figure*}[tb]
	\centering
	\subfigure[Test accuracy with different $N$]{\label{fig:sentiment_N_ta}\includegraphics[width=0.4\textwidth]{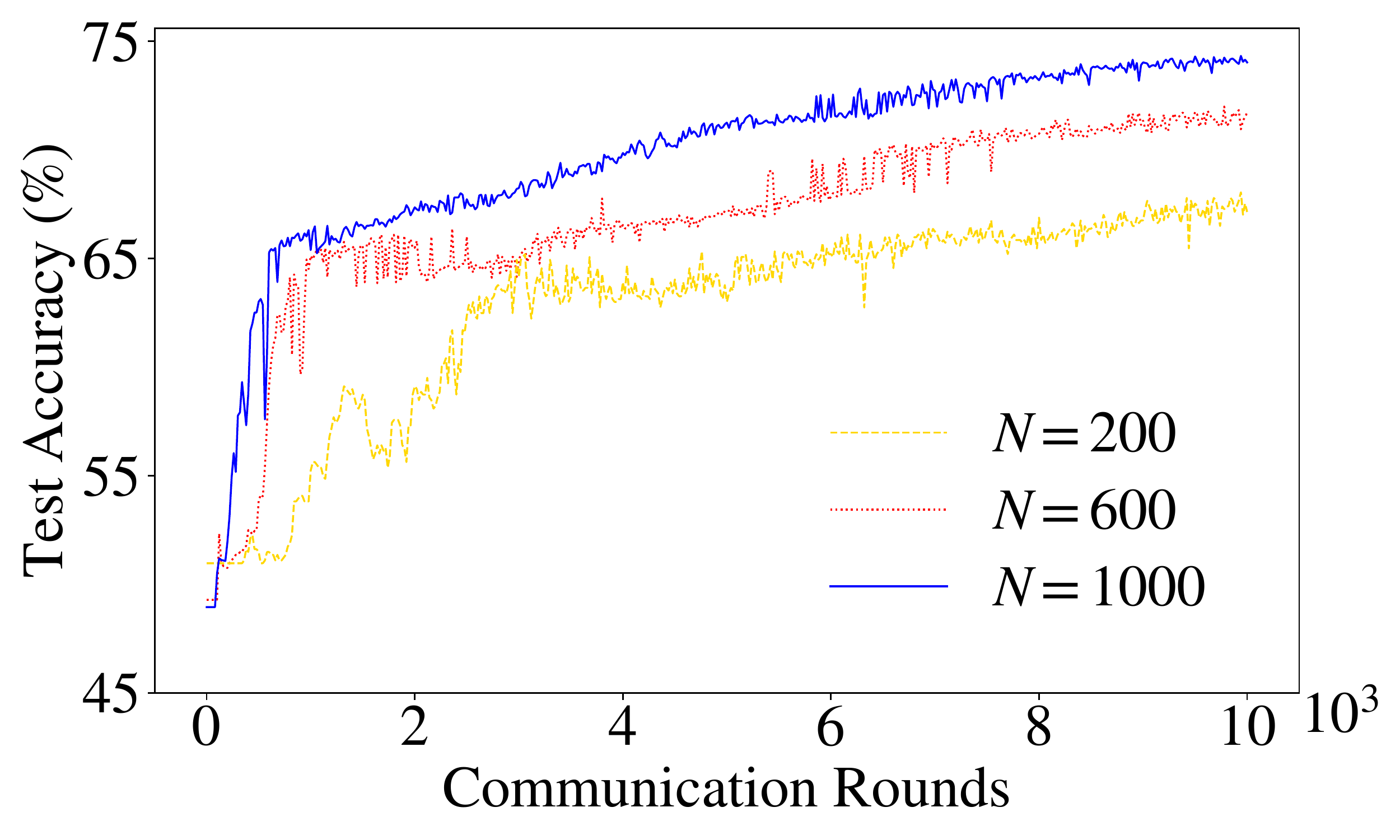}}
	\subfigure[Test accuracy with different $\beta$]{\label{fig:sentiment_beta_ta}\includegraphics[width=0.4\textwidth]{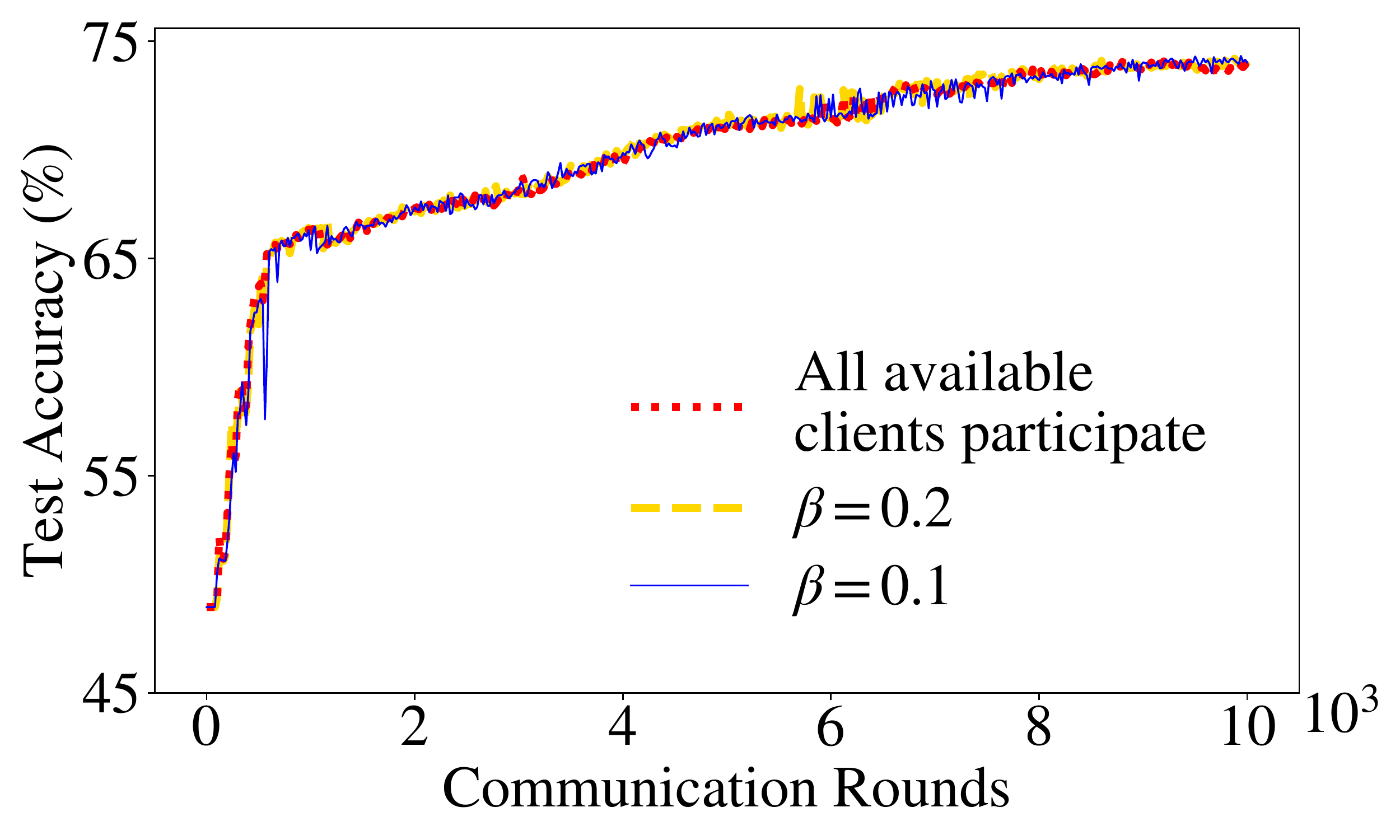}}
	\caption{Test accuracies of FedLaAvg on Sentiment140 dataset
	by varying the total number of clients $N$ and the proportion of
	selected clients $\beta$.}
	\label{fig:sentiment_self_test-acc(real)}
\end{figure*}

\section{Supplementary Notes for the Experiment Environment}
The MNIST dataset is available from \url{http://yann.lecun.com/exdb/mnist/}.
The Sentiment140 dataset is available from \url{http://help.sentiment140.com/for-students}.
The pretrained GloVe embeddings can be downloaded from 
\url{http://nlp.stanford.edu/data/glove.twitter.27B.zip}.
In addition, experiments are conducted on machines with operating system Ubuntu~18.04.3, CUDA version 10.1, and one NVIDIA GeForce RTX 2080Ti GPU. 
The average runtime on our machine is approximately 10 hours per experiment.

\end{document}